\documentclass[letterpaper,11pt]{article}

\usepackage[linesnumbered,algoruled,boxed,lined,noend]{algorithm2e}
\usepackage{fullpage, subfig}
\usepackage{bbm}
\usepackage{graphicx}
\usepackage{amsmath,amssymb,amsthm,amsfonts}

\usepackage{paralist}
\usepackage{bm}
\usepackage{xspace}
\usepackage{url}
\usepackage{prettyref}
\usepackage{boxedminipage}
\usepackage{wrapfig}
\usepackage{ifthen}
\usepackage{color}
\usepackage{xspace}
\usepackage{graphicx}

\usepackage{nicefrac}

\DeclareMathOperator*{\argmax}{arg\,max}
\DeclareMathOperator*{\argmin}{arg\,min}

\usepackage[utf8]{inputenc}

\newcommand{\subexpo}{\mathsf{SubE}}

\usepackage{xcolor}
\definecolor{expert}{HTML}{008000}
\definecolor{error}{HTML}{f96565}

\usepackage{thmtools}
\usepackage{thm-restate}

\usepackage{tikz}
\usetikzlibrary{arrows,calc} 
\newcommand{\tikzAngleOfLine}{\tikz@AngleOfLine}
\def\tikz@AngleOfLine(#1)(#2)#3{%
\pgfmathanglebetweenpoints{%
\pgfpointanchor{#1}{center}}{%
\pgfpointanchor{#2}{center}}
\pgfmathsetmacro{#3}{\pgfmathresult}%
}

\declaretheoremstyle[
    headfont=\normalfont\bfseries, 
    bodyfont = \normalfont\itshape]{mystyle} 

\usepackage[linesnumbered,algoruled,boxed,lined,noend]{algorithm2e} 

\usepackage{listings}
\usepackage{tikz}
\usepackage{caption}
\usepackage{mdwmath}
\usepackage{multirow}
\usepackage{mdwtab}
\usepackage{eqparbox}
\usepackage{multicol}
\usepackage{amsfonts}
\usepackage{tikz}
\usepackage{multirow,bigstrut,threeparttable}
\usepackage{bbm}
\usepackage{epstopdf}
\usepackage{mdwmath}
\usepackage{mdwtab}
\usepackage{eqparbox}
\usetikzlibrary{topaths,calc}
\usepackage{latexsym}
\usepackage{natbib}
\usepackage{amssymb}
\usepackage{bm}
\usepackage{amssymb}
\usepackage{graphicx}
\usepackage{mathrsfs}
\usepackage{epsfig}
\usepackage{psfrag}
\usepackage{setspace}
\usepackage[
            CJKbookmarks=true,
            bookmarksnumbered=true,
            bookmarksopen=true,
						bookmarks=false,
            colorlinks=true,
            citecolor=red,
            linkcolor=blue,
            anchorcolor=red,
            urlcolor=blue,
            hypertexnames=false
            ]{hyperref} 

\usepackage{comment}
\usepackage{mathtools}
\usepackage{blkarray}
\usepackage{multirow,bigdelim,dcolumn,booktabs}

\usepackage{xparse}
\usepackage{tikz}
\usetikzlibrary{calc}
\usetikzlibrary{decorations.pathreplacing,matrix,positioning}

\usepackage[T1]{fontenc}
\usepackage[utf8]{inputenc}
\usepackage{mathtools}
\usepackage{blkarray, bigstrut}
\usepackage{gauss}

\newcommand*{\BraceAmplitude}{0.4em}%
\newcommand*{\VerticalOffset}{0.5ex}%
\newcommand*{\HorizontalOffset}{0.0em}%
\newcommand*{\blocktextwid}{3.0cm}%
\NewDocumentCommand{\InsertLeftBrace}{%
	O{} 
	O{\HorizontalOffset,\VerticalOffset} 
	O{\blocktextwid} 
	m   
	m   
	m   
}{%
	\begin{tikzpicture}[overlay,remember picture]
	\coordinate (Brace Top)    at ($(#4.north) + (#2)$);
	\coordinate (Brace Bottom) at ($(#5.south) + (#2)$);
	\draw [decoration={brace, amplitude=\BraceAmplitude}, decorate, thick, draw=black, #1]
	(Brace Bottom) -- (Brace Top) 
	node [pos=0.5, anchor=east, align=left, text width=#3, color=black, xshift=\BraceAmplitude] {#6};
	\end{tikzpicture}%
}%
\NewDocumentCommand{\InsertRightBrace}{%
	O{} 
	O{\HorizontalOffset,\VerticalOffset} 
	O{\blocktextwid} 
	m   
	m   
	m   
}{%
	\begin{tikzpicture}[overlay,remember picture]
	\coordinate (Brace Top)    at ($(#4.north) + (#2)$);
	\coordinate (Brace Bottom) at ($(#5.south) + (#2)$);
	\draw [decoration={brace, amplitude=\BraceAmplitude}, decorate, thick, draw=black, #1]
	(Brace Top) -- (Brace Bottom) 
	node [pos=0.5, anchor=west, align=left, text width=#3, color=black, xshift=\BraceAmplitude] {#6};
	\end{tikzpicture}%
}%
\NewDocumentCommand{\InsertTopBrace}{%
	O{} 
	O{\HorizontalOffset,\VerticalOffset} 
	O{\blocktextwid} 
	m   
	m   
	m   
}{%
	\begin{tikzpicture}[overlay,remember picture]
	\coordinate (Brace Top)    at ($(#4.west) + (#2)$);
	\coordinate (Brace Bottom) at ($(#5.east) + (#2)$);
	\draw [decoration={brace, amplitude=\BraceAmplitude}, decorate, thick, draw=black, #1]
	(Brace Top) -- (Brace Bottom) 
	node [pos=0.5, anchor=south, align=left, text width=#3, color=black, xshift=\BraceAmplitude] {#6};
	\end{tikzpicture}%
}%

\usetikzlibrary{patterns}

\definecolor{cof}{RGB}{219,144,71}
\definecolor{pur}{RGB}{186,146,162}
\definecolor{greeo}{RGB}{91,173,69}
\definecolor{greet}{RGB}{52,111,72}

\theoremstyle{plain}
\newtheorem{thm}{Theorem}[section]
\newtheorem{theorem}{Theorem}[section]

\newtheorem{lmm}{Lemma}[section]

\newtheorem{definition}{Definition}
\newtheorem{prop}[theorem]{Proposition}

\def \bP {\mathbb{P}}

\def \bE {\mathbb{E}}
\def \bR {\mathbb{R}}

\def\1{\mathbbm{1}}

\usepackage{xspace}

\newcommand{\stepa}[1]{\overset{\rm (a)}{#1}}
\newcommand{\stepb}[1]{\overset{\rm (b)}{#1}}
\newcommand{\stepc}[1]{\overset{\rm (c)}{#1}}
\newcommand{\stepd}[1]{\overset{\rm (d)}{#1}}

\newcommand{\naturals}{\mathbb{N}}

\newcommand{\TV}{{\sf TV}}

\newcommand{\KL}{{\sf KL}}

\newcommand{\pth}[1]{\left( #1 \right)}
\newcommand{\qth}[1]{\left[ #1 \right]}
\newcommand{\sth}[1]{\left\{ #1 \right\}}
\newcommand{\bpth}[1]{\Big( #1 \Big)}
\newcommand{\bqth}[1]{\Big[ #1 \Big]}
\newcommand{\bsth}[1]{\Big\{ #1 \Big\}}

\newcommand{\Poi}{\text{\rm Poi}}
\newcommand{\Unif}{\text{\rm Unif}}

\newcommand{\indc}[1]{{\mathbf{1}_{\left\{{#1}\right\}}}}

\definecolor{myblue}{rgb}{.8, .8, 1}
\definecolor{mathblue}{rgb}{0.2472, 0.24, 0.6} 
\definecolor{mathred}{rgb}{0.6, 0.24, 0.442893}
\definecolor{mathyellow}{rgb}{0.6, 0.547014, 0.24}

\newcommand{\calE}{{\mathcal{E}}}
\newcommand{\calF}{{\mathcal{F}}}
\newcommand{\calG}{{\mathcal{G}}}

\newcommand{\calN}{{\mathcal{N}}}

\newcommand{\calP}{{\mathcal{P}}}

\newcommand{\calS}{{\mathcal{S}}}

\newcommand{\calX}{{\mathcal{X}}}
\newcommand{\calY}{{\mathcal{Y}}}

\newcommand{\rmd}{\mathrm{d}}

\usepackage{cleveref}
\crefname{lemma}{Lemma}{Lemmas}
\Crefname{lemma}{Lemma}{Lemmas}
\crefname{thm}{Theorem}{Theorems}
\Crefname{thm}{Theorem}{Theorems}
\Crefname{assumption}{Assumption}{Assumptions}
\Crefname{definition}{Definition}{Definitions}
\Crefname{prop}{Proposition}{Propositions}
\crefname{prop}{Proposition}{Propositions}
\crefformat{equation}{(#2#1#3)}

\newcommand{\reg}{\mathsf{Regret}}
\newcommand{\mmse}{\mathsf{mmse}}
\newcommand{\ntest}{{n_{\mathsf{test}}}}
\newcommand{\Nloc}{N_{\mathrm{loc}}}

\newrefformat{eq}{(\ref{#1})}
\newrefformat{thm}{Theorem~\ref{#1}}
\newrefformat{th}{Theorem~\ref{#1}}
\newrefformat{chap}{Chapter~\ref{#1}}
\newrefformat{defn}{Definition~\ref{#1}}
\newrefformat{definition}{Definition~\ref{#1}}
\newrefformat{sec}{Section~\ref{#1}}
\newrefformat{seca}{Section~\ref{#1}}
\newrefformat{algo}{Algorithm~\ref{#1}}
\newrefformat{fig}{Fig.~\ref{#1}}
\newrefformat{tab}{Table~\ref{#1}}
\newrefformat{rmk}{Remark~\ref{#1}}
\newrefformat{clm}{Claim~\ref{#1}}
\newrefformat{def}{Definition~\ref{#1}}
\newrefformat{cor}{Corollary~\ref{#1}}
\newrefformat{lemma}{Lemma~\ref{#1}}
\newrefformat{lmm}{Lemma~\ref{#1}}
\newrefformat{prop}{Proposition~\ref{#1}}
\newrefformat{pr}{Proposition~\ref{#1}}
\newrefformat{property}{Property~\ref{#1}}
\newrefformat{app}{Appendix~\ref{#1}}

\newcommand{\UnifiedFigure}[4]{
    \begin{figure}[#1]
      \centering
      #4
      \caption{#3}\label{#2}
    \end{figure}
}

\RestyleAlgo{boxruled}  

\begin{document}

\title
{Universal priors: solving empirical Bayes \\
via Bayesian inference and pretraining}

\author{Nick Cannella, Anzo Teh, Yanjun Han, Yury Polyanskiy\thanks{N.C. is with the Courant Institute of Mathematical Sciences, NYU, New York City, NY, email: \url{nvc9912@nyu.edu}; Y.H. is with the Courant Institute of Mathematical Sciences and Center for Data Science, NYU, New York City, NY, email: \url{yanjunhan@nyu.edu}; A.T. and Y.P. are with Department of EECS, MIT, Cambridge,
		MA, email: \url{anzoteh@mit.edu} and \url{yp@mit.edu}.}}
\maketitle

\begin{abstract} 
We theoretically justify the recent empirical finding of \citep{teh2025solving} that a transformer pretrained on synthetically generated data achieves strong performance on empirical Bayes (EB) problems. We take an indirect approach to this question: rather than analyzing the model architecture or training dynamics, we ask why a pretrained Bayes estimator, trained under a \emph{prespecified training distribution}, can adapt to \emph{arbitrary test distributions}. Focusing on Poisson EB problems, we identify the existence of \emph{universal priors} such that training under these priors yields a near-optimal regret bound of $\widetilde{O}(\frac{1}{n})$ uniformly over all test distributions. Our analysis leverages the classical phenomenon of \emph{posterior contraction} in Bayesian statistics, showing that the pretrained Bayes estimator adapts to unknown test distributions precisely through posterior contraction. This perspective also explains the phenomenon of \emph{length generalization}, in which the test sequence length exceeds the training length, as the model performs Bayesian inference using a fractional posterior.
\end{abstract}

\tableofcontents

\section{Introduction}\label{sec:intro}

Consider the following empirical Bayes (EB) task in the Poisson model: let $\theta_1,\dots,\theta_n$ be i.i.d. drawn from some unknown prior $G_0$ supported on $[0, A]$, and the observations $X^n$ be conditionally independent with $X_i\sim \Poi(\theta_i)$ given $\theta^n$. Here we assume knowledge of $A$, but do not impose any condition on the prior $G_0$ (not even continuity or smoothness). The target of empirical Bayes is to propose an estimator $\widehat{\theta}^n = \widehat{\theta}^n(X^n)$ that is nearly optimal \emph{on every problem instance}, i.e., achieves competitive performance compared to the Bayes estimator with the oracle knowledge of $G_0$. The standard notion in empirical Bayes to quantify the estimator performance is the \emph{regret}, defined as the excess MSE over the Bayes risk: 
\begin{align}\label{eq:regret}
	\reg(\widehat{\theta}^n; G_0) &= \bE_{G_0}\bqth{\frac{1}{n}\|\widehat{\theta}^n(X^n) - \theta^n \|_2^2 - \min_{\theta^\star(\cdot)} \frac{1}{n}\|\theta^\star(X^n) - \theta^n \|_2^2 }\nonumber
	\\
	&= \bE_{G_0}\bqth{\frac{1}{n} \|\widehat{\theta}^n - \theta_{G_0}(X^n)\|_2^2 }, 
\end{align}
where $\theta_{G_0}(X_i) = \bE_{G_0}[\theta_i | X_i]$ is the Bayes estimator (posterior mean) with the knowledge of $G_0$, and $\theta_{G_0}(X^n) = (\theta_{G_0}(X_1),\dots,\theta_{G_0}(X_n))$. A small regret $\reg(\widehat{\theta}^n) = o(1)$ means that the Bayes risk can be asymptotically attained by the legal estimator $\widehat{\theta}^n$. Compared with classical statistical estimators (like the MLE), empirical Bayes estimators usually enjoy a much better empirical performance due to instance-wise guarantees and implicit adaptations to the prior structure \citep{Rob51, Rob56, jiang2009general, han2025besting}.

There exist several ways to solve empirical Bayes problems in the literature. Specializing to the Poisson empirical Bayes model, the earliest example is Robbins' estimator based on $f$-modeling (i.e., mimicking the form of the Bayes estimator $\theta_{G_0}(X)$). A numerically more stable approach is the $g$-modeling, which learns a prior from data and uses the Bayes estimator under the learned prior. A notable example for learning the prior is the nonparametric MLE (NPMLE), as well as a broader class of minimum distance estimators \citep{vandegar2021neural,jana2025optimal}. Finally, a modern approach is to use empirical risk minimization (ERM), which minimizes a properly constructed loss function evaluated on $X^n$ over a suitably chosen function class \citep{barbehenn2022nonparametric, jana2023empirical}. With the exception of Robbins' estimator, all these estimators solve an optimization program at test time, meaning that these programs depend on $X^n$. 

The recent work \citep{teh2025solving} proposes to solve empirical Bayes problems via a different strategy of using a \emph{pretrained estimator}. Unlike the previous ERM approach which trains a separate model for each sample $X^n$, a pretrained model learns a function $\widehat{\theta}^n(\cdot)$ from a large pool of properly generated training data and enables extremely fast computation at test time (namely, applying $\widehat{\theta}^n(\cdot)$ directly to $X^n$). This strategy is inspired by recent successes of TabPFN \citep{hollmann2022tabpfn, hollmann2025accurate}, which similarly applies a single pretrained model across diverse categorical datasets. 
This therefore achieves \textit{cost amortization} in a similar spirit with amortized inference (see, e.g. \cite{zammit2025neural}), where after pretraining with synthetic data, a transformer can perform quick inference on millions of batches of new observations. As shown via experimental results in \cite{teh2025solving} for Poisson EB, a well-trained transformer indeed achieves a better performance than the state-of-the-art NPMLE-based estimator at only a tiny fraction of inference time. 

In this paper, we study the statistical aspect of the approach in \cite{teh2025solving}. Specifically, we ask the following question: 
\begin{center}
	\fbox{%
		\parbox{0.9\linewidth}{%
			\centering
			Why can a pretrained Bayes estimator trained under a \emph{prespecified training distribution} adapt to \emph{all possible} test distributions?
	}}
\end{center}
For Poisson EB, a pretrained estimator is constructed as follows: Given a large number of training batches $\{(\theta^{n,(m)}, X^{n,(m)})\}_{m=1}^M$ generated from a training prior $\theta^{n,(m)} \sim G_{\Pi}$ and Poisson model $X^{n,(m)} | \theta^{n,(m)} \sim \otimes_{i=1}^n \Poi(\theta_i^{(m)})$, the pretrained estimator is the following empirical risk minimizer: 
\begin{align}\label{eq:ERM}
	\widehat{\theta}^n = \argmin_{T: \naturals^n \to [0,A]^n } \frac{1}{M}\sum_{m=1}^M \| \theta^{n,(m)} - T(X^{n,(m)}) \|_2^2.\tag{\textsf{ERM}}
\end{align}
At the population level $M\to\infty$, and assuming that the global minimizer of \Cref{eq:ERM} is attained, this pretrained estimator $\widehat{\theta}^n$ is the posterior mean vector $\bE_{G_{\Pi}}[\theta^n | X^n] = (\bE_{G_{\Pi}}[\theta_1 | X^n], \dots, \bE_{G_{\Pi}}[\theta_n | X^n])$ under an $n$-dimensional prior $G_{\Pi}$ for $\theta^n$. We call such a training prior $G_\Pi$ \emph{universal} if this Bayes estimator achieves a vanishing regret over \emph{all} test priors $G_0$, i.e.
\begin{align}\label{eq:universal}
	\sup_{G_0\in \calP([0,A])} \reg(\bE_{G_{\Pi}}[\theta^n | X^n] ; G_0) = o(1). 
\end{align}

\begin{algorithm}[!t]
	\caption{A pretrained empirical Bayes estimator via transformers}\label{alg:universal_prior}
	\SetAlgoLined
	\KwIn{Dimension parameter $n$, support parameter $A>0$, number of training batches $M$.}
	\KwOut{A pretrained estimator $\widehat{\theta}^n: \naturals^n \to [0, A]^n$. }
	
	Let $k\gets \lceil c_0\frac{\log n}{\log\log n}\rceil$; \tcp*[f]{$c_0>0$ is a properly chosen hyperparameter}
	
	\For(\tcp*[f]{Generate training data}){batch $m = 1,\cdots, M$}{
		Sample prior locations $\lambda_1, \dots, \lambda_k \sim \Unif([0,A])$\;
		Sample prior weights $(w_1, \dots, w_k) \sim \mathsf{Dir}(1,\dots,1)$\; 
		Sample training outputs $\theta^{n,(m)} \sim (\sum_{j=1}^k w_j \delta_{\lambda_j})^{\otimes n}$\; 
		Sample training inputs $X^{n, (m)}\sim \otimes_{i=1}^n \Poi(\theta_i^{(m)})$\;
	}
	Train a transformer $\mathsf{T}_{\widehat{\zeta}}$ to minimize the training error, with parameter \tcp*[f]{Training} 
	\begin{align*}
		\widehat{\zeta} = \argmin_{\zeta} \frac{1}{M}\sum_{m=1}^M \| \theta^{n,(m)} - \mathsf{T}_\zeta(X^{n,(m)}) \|_2^2.
	\end{align*}\\
	\Return{The trained transformer as the final estimator: $\widehat{\theta}^n = \mathsf{T}_{\widehat{\zeta}}(X^n)$ for any test data $X^n$. }
\end{algorithm}

One might expect that universal priors must be carefully engineered to achieve this ambitious goal; interestingly, this turns out not to be the case. A natural way to construct $G_{\Pi}$ is through a hierarchical prior or a \emph{prior-on-prior (PoP)}: Let $\Pi \in \calP(\calP([0,A]))$\footnote{Here and throughout, for a Polish space (i.e., a complete and separable metric space) $(X,d)$, we write $\calP(X)$ for the set of Borel probability measures on $(X,d)$ equipped with the topology of weak convergence; under this topology, $\calP(X)$ is itself a Polish space by Prokhorov's theorem.} be a PoP (i.e. prior over probability distributions), we can generate $G\sim \Pi$ and $\theta^n | G \sim G^{\otimes n}$. In other words, the high-dimensional prior $G_\Pi = \bE_\Pi[G^{\otimes n}]$ is an i.i.d. mixture induced by the PoP $\Pi$. The resulting Bayes estimator $\theta_\Pi^n(X^n) := \bE_{G_{\Pi}}[\theta^n | X^n]$ will be called a \emph{hierarchical Bayes estimator}, or a \emph{nonparametric Bayes estimator} as the prior $G_\Pi$ is usually nonparametric. We will also call a PoP $\Pi$ \emph{universal} if the corresponding prior $G_\Pi$ is universal in the sense of \Cref{eq:universal}. Under this hierarchical structure, it turns out that there exists a remarkably simple choice of a universal PoP $\Pi$. 

\begin{thm}\label{thm:simple_prior}
	Let the PoP $\Pi$ be the distribution of a random prior $G = \sum_{j=1}^k w_j \delta_{\lambda_j}$ with $\lambda_1, \dots, \lambda_k \sim \Unif([0,A])$ and $(w_1, \dots, w_k) \sim \mathsf{Dir}(1,\dots,1)$. For $k = \lceil c_0\frac{\log n}{\log\log n}\rceil$ with a large enough absolute constant $c_0>0$, it holds that
	\begin{align*}
		\sup_{G_0\in \calP([0,A])} \reg(\theta_\Pi^n; G_0) \le C \frac{\log^3 n}{n(\log\log n)^2},  
	\end{align*}
	where $C=C(A,c_0)$ is an absolute constant depending only on $A$ and $c_0$. 
\end{thm}

The practical end-to-end implementation of $\theta_\Pi^n$ is displayed in \Cref{alg:universal_prior}, where the class of transformers is used to find the sequence-to-sequence map $T$ in \Cref{eq:ERM}. \Cref{thm:simple_prior} shows that, if the approximation error of the pretrained transformer in \Cref{alg:universal_prior} to the Bayes estimator $\theta_\Pi^n$ turns out to be small, then this transformer achieves a vanishing regret for all test priors $G_0$. 

We make some remarks on \Cref{thm:simple_prior}. First, the regret bound in \Cref{thm:simple_prior} is near-optimal: For $A=\Theta(1)$, it was shown in Theorem 2 of \cite{polyanskiy2021sharp} that the minimax regret is $\Theta(\frac{1}{n}(\frac{\log n}{\log\log n})^2)$. Therefore, although the estimator $\theta_\Pi^n$ is constructed in a Bayesian framework, it achieves near-optimal frequentist guarantees (off by a $\log n$ factor) even in the worst case.

Second, the universal prior $G_\Pi$, albeit remarkably simple, is a  \textit{high-dimensional} mixture of i.i.d. priors. Consequently, the posterior mean $\bE_{\Pi}[\theta_1 | X^n]$ depends on the entire $X^n$, not only $X_1$. In other words, the pretrained Bayes estimator is not separable, and this dependence turns out to be essential for \Cref{thm:simple_prior} to hold. 
Since such expectations are difficult to compute classically, we employ transformers to approximate this sequence-to-sequence map.

Third, rather than analyzing the training dynamics or architecture-specific details of transformers, we adopt the key assumption that the transformers can approximate the Bayes estimator $\theta_\Pi^n$, 
which we numerically verify in \Cref{sec:numerics}. 
We choose transformers (without positional encoding or masking) for their expressive power for sequence-to-sequence maps \citep[Theorem 4.1, 4.2]{teh2025solving}; \citep{furuya2025transformers}, ability to incorporate any sequence length, and permutation equivariance (see \Cref{subsec:batch} for definition; $\theta_\Pi^n$ is also permutation-equivariant).  

Finally, for approximate Bayes estimators close to $\theta_\Pi^n$, \Cref{subsec:batch} develops regret bounds that depend on the approximation error. In particular, such analysis shows that the same regret bound of \Cref{thm:simple_prior} can be attained by the ERM with a finite number of batches $M$, though our current sufficient condition on $M$ grows super-polynomially in $n$. This is consistent with the large data requirements observed in practical pretraining.

In the sequel, we present an in-depth analysis of the hierarchical Bayes estimator $\theta_\Pi^n$ and universal PoPs $\Pi$. First, we establish basic statistical results such as the existence of the least favorable PoP via the minimax theorem, and admissibility of hierarchical Bayes estimators while classical estimators (such as the NPMLE-based estimator) could be inadmissible. Second, based on the classical phenomenon of \emph{posterior contraction} in Bayesian statistics, we will show the central result in our paper that \emph{a broad class of prior-on-priors turns out to be universal}. Finally, drawing on the transformer structure, we explain why the pretrained transformer enjoys \emph{length generalization}. 

\subsection{Least favorable PoP and admissibility}
We establish basic properties of $\theta_\Pi^n$ in this section. Recall that $\theta_\Pi^n$ is the Bayes estimator (under the quadratic loss) in the following hierarchical Bayes model: 
\begin{equation}\label{eq:hierarchical}
	\begin{split}
		G & \sim \Pi, \\
		\theta_i \mid G & \overset{\mathrm{i.i.d.}}{\sim} G, \quad i=1,\dots, n\\
		X_i \mid \theta^n, G & \overset{\mathrm{ind.}}{\sim} \Poi(\theta_i), \quad i=1,\dots,n. 
	\end{split}
\end{equation}
Although the induced high-dimensional prior $G_\Pi = \bE_\Pi[G^{\otimes n}]$ on $\theta^n$ is constrained to be an i.i.d. mixture, the following result shows that there exists a \emph{least favorable PoP} $\Pi^\star$ such that the hierarchical Bayes estimator $\theta_{\Pi^\star}^n$ attains the minimax optimal regret. 

\begin{prop}\label{prop:worst-pop}
	For any $A > 0$ and $n\ge 1$, we have the following identity: 
	\[\inf_{\widehat{\theta}^n}\sup_{G\in\mathcal{P}([0, A])}\reg(\widehat{\theta}^n; G) = \sup_{\Pi\in\mathcal{P}(\mathcal{P}([0, A]))}\inf_{\widehat{\theta}^n}\mathbb{E}_{G\sim\Pi}[\reg(\widehat{\theta}^n; G)].\]
	Consequently, the saddle point gives a least favorable PoP $\Pi^\star$ such that the corresponding Bayes estimator $\theta_{\Pi^\star}^n$ achieves the minimax regret. 
\end{prop}

\prettyref{prop:worst-pop} is a minimax theorem; its proof verifies the compactness and continuity conditions, and is deferred to \Cref{app:worst-pop-proof}. By the characterization of the minimax regret $\Theta(\frac{1}{n}(\frac{\log n}{\log\log n})^2)$ for Poisson EB \citep{polyanskiy2021sharp}, Proposition \ref{prop:worst-pop} immediately implies the existence of a universal PoP $\Pi$ that attains \Cref{eq:universal}. In addition, it justifies the i.i.d. mixture form of the high-dimensional prior $G_\Pi$ used in pretraining, so the only task is to construct the PoP $\Pi$. 

The next result concerns the admissibility of the hierarchical Bayes estimator $\theta_\Pi^n$, with the classical notion of admissibility defined below for Poisson EB. 

\begin{definition}[Admissibility]
	An estimator $\widehat{\theta}^n: \naturals^n\to [0,A]^n$ is called \emph{inadmissible} if there exists another estimator $\widetilde{\theta}^n$ such that $\reg(\widetilde{\theta}^n; G) \le \reg(\widehat{\theta}^n; G)$ for all $G\in \calP([0,A])$, and $\reg(\widetilde{\theta}^n; G_0) < \reg(\widehat{\theta}^n; G_0)$ for some $G_0\in \calP([0,A])$. If an estimator is not inadmissible, it is called \emph{admissible}. 
\end{definition}
We note that this definition remains equivalent if we replace the regret by the MSE $\bE_{G_0}[ \frac{1}{n}\|\widehat{\theta}^n - \theta^n\|_2^2 ]$, since the Bayes risk is independent of $\widehat{\theta}^n$. In the next result, we will compare the hierarchical Bayes estimator $\theta_\Pi^n$ with the NPMLE-based estimator
\begin{align*}
	\widehat{\theta}^{\mathrm{NPMLE}}(X^n) = (\theta_{\widehat{G}}(X_1), \dots, \theta_{\widehat{G}}(X_n)), \quad \widehat{G} = \argmax_{G\in \calP([0,A])} \prod_{i=1}^n f_G(X_i), 
\end{align*}
where we recall that $\theta_G(x) = \bE_{G}[\theta|X=x]$ is the 1-D Bayes estimator, and $
f_G(x) = \bE_{\theta\sim G}[\Poi(x; \theta)]$
is a shorthand for the marginal pmf of $X\sim \Poi(\theta)$ when $\theta\sim G$. Note that the support of the NPMLE $\widehat{G}$ is restricted to be $[0,A]$, since otherwise it is clearly inadmissible by truncating the output to $[0,A]^n$. The next result shows that while the hierarchical Bayes estimator $\theta_\Pi^n$ is admissible for most PoPs, the NPMLE-based estimator is not. 

\begin{prop}\label{prop:admissible}
	For all $\Pi\in \calP(\calP([0,A]))\backslash \{\delta_{\delta_0}\}$, the hierarchical Bayes estimator $\theta_\Pi^n$ is admissible. However, if $A\ge 3$ is an integer, then the NPMLE-based estimator $\widehat{\theta}^{\mathrm{NPMLE}}$ is inadmissible. 
\end{prop}

\subsection{Universal PoPs}

While \prettyref{prop:worst-pop} justifies the pretrained approach in \Cref{alg:universal_prior}, it still remains to identify PoPs that are approximately least favorable. This motivates the following definition of \emph{thick} PoPs, which provide a sufficient condition for attaining a small worst-case regret.

\begin{definition}[Thick PoP]\label{definition:PoP}
	We call a PoP $\Pi$ \emph{thick} with rate function $E(\delta,r)$ if for every prior $G_0$ supported on $[0, A]$, there exists some $G$ such that $\TV(f_{G_0}, f_G)\le \delta$,\footnote{Here $\TV(P,Q) = \frac{1}{2}\int |\rmd P-\rmd Q|$ and $\chi^2(P\|Q) = \int \frac{(\rmd P)^2}{\rmd Q}-1$ denote the total variation distance and chi-squared divergence, respectively.} and
	\begin{align}\label{eq:prior_mass}
		\Pi\pth{\sth{G': \chi^2(f_G \| f_{G'})\le r^2}} \ge e^{-E(\delta, r)}. 
	\end{align}
\end{definition}

The term ``thick'' is motivated by the use of ``thick priors'' in \cite{bickel2012semiparametric}, referring to priors with continuous and strictly positive Lebesgue densities. Taking $G=G_0$, the main condition \eqref{eq:prior_mass} states that the PoP $\Pi$ puts a sufficient amount of mass near the true prior $G_0$. For technical reasons, ``near'' is measured through the $\chi^2$ divergence between the marginal pmfs, and we include an additional prior $G$ for an intermediate coupling. We note that the intuition behind \eqref{eq:prior_mass} is standard in the literature of posterior contraction \citep{ghosal2000convergence}, commonly with $\chi^2$ replaced by KL. 

Our next result shows that thick PoPs (with proper rate function $E$) are universal, and establishes a regret bound on the hierarchical Bayes estimator under a thick PoP. 
\begin{thm}\label{thm:general}
	Let the PoP $\Pi$ be thick with rate $E$, and $r_n>0$ satisfy $E(n^{-2}, r_n)\le nr_n^2$. Then for an absolute constant $C=C(A)>0$, the hierarchical Bayes estimator $\theta_\Pi^n$ satisfies
	\begin{align*}
		\sup_{G_0\in \calP([0,A])}  \reg(\theta_\Pi^n; G_0) \le \frac{C\log n}{n\log\log n}\bpth{\frac{\log n}{\log\log n} + nr_n^2}.  
	\end{align*}
\end{thm}

\Cref{thm:simple_prior} is then a special case of \Cref{thm:general} in conjunction with the following lemma.
\begin{lmm}\label{lemma:simple_prior}
	For a large enough hyperparameter $c_0>0$, the PoP used in \Cref{alg:universal_prior} is thick with $E(n^{-2}, r)= O(\frac{\log^2 n}{\log\log n})$ for $r^2 \ge \frac{1}{n}$. In particular, we can choose $r_n^2 = O(\frac{\log^2 n}{n\log\log n})$.  
\end{lmm}

We sketch the proof idea of \Cref{thm:general} via \emph{posterior contraction} in Bayesian statistics \citep{ghosal2000convergence,shen2001rates}. By \prettyref{lemma:posterior_mean_training}, the $i$-th coordinate of $\theta_\Pi^n$ equals $\theta_{G_i}(X_i)$, the one-dimensional Bayes estimator under the posterior mean $G_i = \bE_{G\sim \Pi_{G\mid X_{\backslash i}}}[G]$. For the true data $X^n \sim f_{G_0}^{\otimes n}$, posterior contraction implies that $G_i$ concentrates around the test distribution $G_0$, so that the pretrained estimator $\theta_{G_i}(X_i)$ adapts to the true Bayes estimator $\theta_{G_0}(X_i)$ under the unknown prior $G_0$. These intuitions will be made precise in \Cref{sec:posterior_contraction}.

\subsection{Length generalization}\label{subsec:length_generalization}
The hierarchical Bayes estimator $\theta_\Pi^n$ is a map from $\naturals^n$ to $\bR^n$, but a transformer can take variable-length inputs and gives a map from $\naturals^m$ to $\bR^m$ for every $m\in \naturals$. Therefore, if we use a transformer to represent the hierarchical Bayes estimator $\theta_\Pi^n$ on length $n$, it can also take input sequence $X^m$ of a longer length $m>n$. This is \emph{length generalization} of transformers: for Poisson EB, the empirical evidence in \cite{teh2025solving} shows that even if the transformer is solely trained on training length $n$, it achieves small regret on larger test lengths $\ntest > n$. Our perspective of Bayesian inference and posterior contraction framework provides a theoretical explanation for this phenomenon.

To this end, we need to generalize our definition of the hierarchical Bayes estimator $\theta_\Pi: \naturals^n\to \bR^n$ into a map that can incorporate any input sequence length. We use the following observations on length generalization for transformers \cite[Section 2.2]{furuya2025transformers}. If a transformer architecture only consists of softmax attention and token-wise operations\footnote{Crucially, this transformer does not contain positional encoding, causal masking, or the final softmax layer to convert logits into a probability distribution.} (e.g. input embedding, MLP, residual connection, and layer normalization), then it is \emph{permutation equivariant}, i.e. $\pi\circ \mathsf{T}(X^n) = \mathsf{T}(\pi\circ X^n)$ for all $\pi\in S_n$. As a result of permutation equivariance, its $i$-th output takes the form $\mathsf{T}_i(X^n) = f_n(X_i,\mu_n)$ for some $f_n: \calX \times \calP(\calX) \to \calY$ and the \emph{empirical distribution} $\mu_n$ of all inputs $X^n$. Furthermore, under this length generalization model for transformers with softmax attention and token-wise operations, the induced map can be represented by a single function $f_n\equiv f$ independent of $n$. Therefore, a single function $f: \calX \times \calP(\calX) \to \calY$ encodes the transformer map for general input length: given input $X^\ntest$, the $i$-th output is $\mathsf{T}_i(X^\ntest) = f(X_i, \mu_\ntest)$. For example, this implies that if $\mathsf{T}(X^n)=Y^n$, then $\mathsf{T}(X^n,\dots,X^n) = (Y^n,\dots,Y^n)$.

Applying the above length generalization model to $\theta_\Pi^n$, we have the following theorem. 

\begin{thm}\label{thm:length-generalization}
	For any PoP $\Pi$ and training length $n$, the hierarchical Bayes estimator $\theta_{\Pi}^n$ is permutation equivariant, so that there exists a map $f_{\Pi,n}: \naturals \times \calP(\naturals)\to \bR_+$, determined solely by $\Pi$ and $n$, such that $\theta_{\Pi,i}(X^n) = f_{\Pi,n}(X_i,\mu_n)$ for every input $X^n\in \naturals^n$, with $\mu_n = \frac{1}{n}\sum_{i=1}^n \delta_{X_i}$. 
	
	Next, suppose that $\Pi$ is thick with rate $E$. For any test sequence length $\ntest \ge n$, let $r>0$ satisfy $E(n_{\mathsf{test}}^{-2}, r) \vee 1\le nr^2$. The following regret bound holds for the estimator $f_{\Pi,n}^{\ntest}(X^{\ntest}):=(f_{\Pi,n}(X_1,\mu_\ntest),\dots,f_{\Pi,n}(X_\ntest,\mu_\ntest))$ with an absolute constant $C = C(A)$: 
	\begin{align*}
		\sup_{G_0\in \calP([0,A])} \reg(f_{\Pi,n}^{\ntest}; G_0) \le \frac{C\log\ntest}{\log\log\ntest}\pth{\frac{\log\ntest}{\ntest\log\log \ntest} + r^2}. 
	\end{align*}
\end{thm}

The first part of \Cref{thm:length-generalization} shows that the hierarchical Bayes estimator $\theta_\Pi^n$ under every PoP can be represented by some function $f_{\Pi,n}$. The second part illustrates both strengths and weaknesses of length generalization for a pretrained transformer that learns the map $f_{\Pi,n}$: it still achieves a vanishing regret on longer lengths as the training length $n\to\infty$, but the regret remains $\widetilde{O}(r^2)=\widetilde{O}(\frac{1}{n})$ rather than $\widetilde{O}(\frac{1}{\ntest})$. This result conforms to the empirical finding in \cite{teh2025solving} that the pretrained transformer enjoys a low regret for all $\ntest \ge n$, but the regret stops decreasing when $\ntest \gg n$; see \Cref{fig:performance} for an illustration. Intuitively, this is because the map $f_{\Pi, n}$ depends on the training length $n$, and is thus not fully adaptive to a longer sequence length $\ntest \ge n$ at test time. 

The proof of \Cref{thm:length-generalization} again relies on posterior contraction, but now for \emph{fractional/generalized posteriors}. 
For $\ntest\ge n$, the new estimator $f_{\Pi,n}^{\ntest}$ is no longer a hierarchical Bayes estimator, but it performs Bayesian inference using an \emph{$\alpha$-posterior}, i.e. with posterior update
\begin{align}\label{eq:alpha-posterior}
	\Pi^{\alpha}(\rmd G|X^{\ntest}) \propto \Pi(\rmd G) \bpth{\prod_{i=1}^{\ntest} f_{G}(X_i) }^\alpha, \quad \text{with } \alpha = \frac{n}{\ntest} \le 1. 
\end{align}
In \Cref{sec:posterior_contraction}, we develop posterior contraction results such that the $\alpha$-posterior $\Pi^\alpha(\rmd G|X^{\ntest})$ still concentrates around $G_0$ for $X^{\ntest}\sim f_{G_0}^{\otimes \ntest}$, with a rate depending now on $\alpha$. We note that such fractional posteriors have appeared previously in the Bayesian literature on model misspecification \citep{bhattacharya2019bayesian, medina2022robustness}. Moreover, although \eqref{eq:alpha-posterior} only holds for the ideal estimator $f_{\Pi,n}^{\ntest}$, for a pretrained transformer with training length $n$ and test length $\ntest$, our numerical experiments in \Cref{sec:alpha-posterior} demonstrate that its output is indeed close to a hierarchical Bayes model performing $\frac{n}{\ntest}$-posterior inference. 

\subsection{Related work}
\paragraph{Empirical Bayes.} 
EB was introduced alongside compound decision theory \citep{Rob51, Rob56}, motivated by the idea that estimating a sequence of parameters can achieve lower risk when each component is allowed to depend on the entire sequence, a phenomenon classically illustrated by the James--Stein estimator \citep{stein1956inadmissibility,james1961estimation}.
For the Poisson model, known estimators are based either on the Tweedie's formula \citep{Rob56}, 
posterior density estimation \citep{kiefer1956consistency, lindsay1983geometry, shen2022empirical, jana2025optimal, han2025besting}, or ERM \citep{jana2023empirical}. 
The former two approaches are also dubbed $f$-modeling and $g$-modeling, respectively \citep{efron2019bayes}. 
Neural approaches have also been explored to approximate the maximum likelihood via generative modeling \citep{wang2019comment,vandegar2021neural}; see also \citep{ghosh2025stein,chen2025score} for the normal means model. 
For the regret in the Poisson-EB problem, tight bounds $\Theta(\frac{1}{n}(\frac{\log n}{\log \log n})^2)$ and $\Theta(\frac{\log^3 n}{n})$ are established for compactly supported and subexponential priors \citep{brown2013poisson, polyanskiy2021sharp, jana2023empirical, jana2025optimal}, with non-trivial extensions to unbounded supports \citep{shen2022empirical}. There is also a rich line of work on the optimal regret in the normal means model \citep{jiang2009general,saha2020nonparametric,polyanskiy2021sharp,soloff2025multivariate,chen2026sharp}. A recent work \citep{kang2026function} also explores function estimation in the empirical Bayes setting, which we discuss in \Cref{thm:simple_prior_polynomial}. We also refer to \citep{efron2024empirical,ignatiadis2025empirical} for surveys on EB. 

\paragraph{Meta-learning.} Our use of a pretrained transformer, trained over a distribution of priors and applied to test sequences without per-instance optimization at test time, falls within the framework of \emph{meta-learning} \citep{hospedales2021meta}. A rich line of recent work studies in-context learning of transformers, a form of meta learning usually instantiated through autoregressive next-token prediction \citep{garg2022can,bai2023transformers,akyurek2023learning,von2023transformers}. In particular, connections between in-context learning and Bayesian inference have been observed both empirically \citep{muller2022transformers,panwar2024context,aggarwal2025bayesian} and theoretically \citep{xie2022explanation,wakayama2025context,ma2025provable}. This connection inspired the seminal TabPFN framework \citep{hollmann2022tabpfn,hollmann2025accurate} and several follow-up works on amortized inference for tabular data \citep{ma2024tabdpt,reuter2025can,mittal2025context,mittal2025amortized}. Similar posterior-contraction ideas have appeared in autoregressive \citep{xie2022explanation,ma2025provable} and diffusion settings \citep{jia2026weak}; by contrast, our EB application views the transformer as a sequence-to-sequence map \citep{teh2025solving}. This Bayesian view of this sequence-to-sequence map also gives a statistical explanation of length generalization \citep{anil2022exploring,peng2024yarn,zhou2024algorithms,zhou2024transformers} through fractional posteriors, rather than transformer-specific architectural or algorithmic mechanisms \citep{ahuja2024provable,izzo2025quantitative,huang2025formal}.

\paragraph{``Bayes'' empirical Bayes.} 
The use of a hierarchical/nonparametric Bayes model \eqref{eq:hierarchical} to solve EB falls into the framework of ``Bayes empirical Bayes'' by \cite{deely1981bayes}, where a line of work \citep{antoniak1974mixtures,gu2017empirical} chooses the Dirichlet Process (DP) \citep{ferguson1973bayesian} as the prior. Statistical guarantees for hierarchical Bayes estimators in EB date back to the asymptotic optimality results of \cite{datta1991asymptotic}, with related developments in \cite{petrone2012bayes,rizzelli2024empirical}; non-asymptotic guarantees were only established in our work and the independent concurrent work of \cite{ignatiadis2026compound}. Both works establish regret guarantees via posterior contraction, focusing on Poisson and Gaussian EB problems, respectively. The two papers differ conceptually in their motivations. \cite{ignatiadis2026compound} study the optimality of nonparametric Bayes estimators under a DP prior, with computation carried out by a Gibbs sampler \citep{neal2000markov}. By contrast, our motivation is more ML-oriented: we focus on pretrained transformers, where computation is amortized through pretraining, and thereby establish length generalization results. Other minor differences include that we establish minimax theorems in \prettyref{prop:worst-pop}, whereas their work extends posterior contraction arguments to compound decision problems; we also note that our admissibility results for Poisson EB in \prettyref{prop:admissible} are motivated by theirs.

\section{Analysis via posterior contraction}\label{sec:posterior_contraction}
\subsection{Preliminaries}\label{subsec:prelim}
We first provide some preliminaries on the Poisson mixture model and the hierarchical Bayes model induced by a general training PoP $\Pi$. For $\theta\sim G$ and $X|\theta \sim \Poi(\theta)$, the Bayes estimator of $\theta$ under the squared loss is the posterior mean, defined as
\begin{align}\label{eq:theta_G}
	\theta_G(x) = \bE_G[\theta | X=x] = (x+1)\frac{f_G(x+1)}{f_G(x)}, 
\end{align}
where $f_G(x) = \int \Poi(x;\theta)G(\rmd \theta)$ is the marginal pmf of $X$. Here and throughout, we will use $\bE_G$ to denote the expectation with respect to the prior $G$. The following result, taken from \cite[Lemma 4]{jana2025optimal}, is a regret-Hellinger inequality that relates the posterior mean difference $\theta_G-\theta_{G_0}$ to the Hellinger distance $H(f_G,f_{G_0})$.\footnote{The squared Hellinger distance between $P$ and $Q$ is defined as $H^2(P,Q):=\int(\sqrt{\rmd P}-\sqrt{\rmd Q})^2$.} 

\begin{lmm}\label{lemma:hellinger-to-regret-1}
	Let $G_0, G$ be two priors supported on $[0,A]$, and $\varepsilon\in (0,e^{-e})$ be any real number. Then for an absolute constant $C=C(A)>0$, 
	\begin{align*}
		\bE_{X\sim f_{G_0}}\qth{ \pth{\theta_G(X) - \theta_{G_0}(X)}^2 } \le C\bpth{\frac{\log (1/\varepsilon)}{\log\log (1/\varepsilon)} H^2(f_G, f_{G_0}) + \varepsilon } . 
	\end{align*}
\end{lmm}

For the class of Poisson mixtures $\calP = \{f_G: \mathrm{supp}(G) \subseteq [0, A]\}$, we will use a metric entropy upper bound under the Hellinger metric. Let $N(\varepsilon, \calP, H)$ denote the $\varepsilon$-covering number of $\calP$ under the Hellinger metric (i.e. the minimum number of $\varepsilon$-balls to cover $\calP$), we define the \emph{local} covering number as $\Nloc(\varepsilon,\calP,H) := \sup_{P_0\in \calP}\sup_{\rho\ge \varepsilon} N(\rho, \calP\cap \{P: H(P_0, P)\le 2\rho\}, H)$. 

\begin{lmm}\label{lemma:metric_entropy}
	For $\calP = \{f_G: \mathrm{supp}(G) \subseteq [0,A]\}$, there is an absolute constant $C = C(A) >0$ with $
	\log \Nloc(\varepsilon, \calP, H) \le C\frac{\log(1/\varepsilon)}{\log\log(1/\varepsilon)}$ for all $\varepsilon\in (0,e^{-e})$. 
\end{lmm}

The last lemma gives a simple leave-one-out structure of the hierarchical Bayes estimator $\theta_\Pi^n$. 

\begin{lmm}\label{lemma:posterior_mean_training}
	For $i\in [n]$, it holds that $
	\bE_\Pi[\theta_i | X^n] = \bE_{G\sim \Pi_{G|X^n}}[\theta_G(X_i)] = \theta_{G_i}(X_i). $
	Here $G_i := \bE_{G\sim \Pi_{G|X_{\backslash i}}}[G]$, and $\theta_G$ is the posterior mean defined in \Cref{eq:theta_G}. 
\end{lmm}

\subsection{Posterior contraction}\label{subsec:posterior_contraction}
The key technical step in proving \Cref{thm:general} is the following posterior contraction lemma. Below we will state a general version not limited to Poisson EB which will later be applied to a Gaussian example in \Cref{subsec:Gaussian_EB}. Let the true data $X_1,\dots,X_n\sim f_{G_0}$ be i.i.d. with $G_0\in \calG$, where $f_{G_0}$ is a general density which may or may not take the form of mixture distributions. In the definition of the Bayes estimator, let $G\sim \Pi$ and $X_1,\dots,X_n \mid G\sim f_G$, where $\Pi$ is a prior distribution over $\calG$ (a PoP in the hierarchical setting). Finally, we set $\calP=\{f_G: G\in \calG\}$, and note that the thickness definition in \prettyref{definition:PoP} naturally extends to the prior $\Pi$ over the density class $\calP$. 

\begin{lmm}\label{lemma:posterior_contraction}
	Let $\Pi$ be thick over $\calP$ in the sense of \prettyref{definition:PoP}, with rate function $E$. Let $\varepsilon_n, r_n>0$ satisfy $
	\log \Nloc(\varepsilon_n, \calP, H) \le n\varepsilon_n^2$ and $E(n^{-2}, r_n) \le nr_n^2. $
	Then there exist absolute constants $c, C >0$ such that for every $G_0\in \calG$ and every $\varepsilon^2\ge \varepsilon_n^2 + r_n^2 + \frac{1}{n}$, 
	\begin{align*}
		f_{G_0}^{\otimes (n-1)}\bpth{ \bsth{H^2(f_{G_0}, \Pi_{X_n|X^{n-1}}) > C\varepsilon^2}} \le n^{-1} + e^{-cn\varepsilon^2}. 
	\end{align*}
	Here the probability is with respect to $X^{n-1}\sim f_{G_0}^{\otimes (n-1)}$, and $\Pi_{X_n|X^{n-1}}$ is the conditional distribution of $X_n$ given $X^{n-1}$ under the Bayes model with prior $\Pi$. 
\end{lmm}

The Poisson EB setting is a special case of \prettyref{lemma:posterior_contraction} by marginalizing out $\theta^n$ from the joint distribution $\Pi_{G,\theta^n,X^n}$. Posterior contraction asserts that, given test data $X_1,\dots,X_{n-1}\sim f_{G_0}$, the posterior distribution $\Pi_{G|X^{n-1}}$ ``concentrates'' around the true prior $G_0$ with high probability, and this concentration is characterized through the induced pushforward measure from $G$ to $X$ (i.e. $\Pi_{X_n|X^{n-1}}$ is statistically close to $f_{G_0}$). This is precisely the statement of Lemma \ref{lemma:posterior_contraction}. The proof of \prettyref{lemma:posterior_contraction} mirrors the classical posterior contraction arguments \citep{ghosal2000convergence,ghosal2008nonparametric,ghosal2017fundamentals}, with a few adaptations to obtain a high-probability statement; the proof details are postponed to the appendix. 

\paragraph{Proof of \Cref{thm:general}, assuming \prettyref{lemma:posterior_contraction}.}
By the regret definition in \eqref{eq:regret} and symmetry, we focus on the last coordinate and write
\begin{align*}
	\reg(\theta_{\Pi}^n; G_0) &= \bE_{X^n\sim f_{G_0}^{\otimes n}}\qth{ \pth{\theta_{\Pi,n}(X^n) - \theta_{G_0}(X_n) }^2 } \\
	&\le C \bpth{\frac{\log n}{\log\log n}\cdot \bE_{X^{n-1}\sim f_{G_0}^{\otimes (n-1)}} \qth{H^2(f_{G_n}, f_{G_0})} + \frac{1}{n}}. 
\end{align*}
Here the last step uses $
\theta_{\Pi,n}(X^n) = \theta_{G_n}(X_n)$ by \prettyref{lemma:posterior_mean_training}, and \prettyref{lemma:hellinger-to-regret-1} applied to $G=G_n$ and $\varepsilon=\frac{1}{n}$. Since $G_n = \bE_{G\sim \Pi_{G|X^{n-1}}}[G]$, we have $f_{G_n}(x_n) = \bE_{G\sim \Pi_{G|X^{n-1}}}[f_G(x_n)] = \Pi_{X_n=x_n|X^{n-1}}$. By \prettyref{lemma:metric_entropy}, we can choose $\varepsilon_n^2=O(\frac{\log n}{n\log\log n})$. Finally, \Cref{thm:general} follows from \prettyref{lemma:posterior_contraction} and integrating the tails up to the deterministic upper bound $H^2\le 2$. 

\subsection{Posterior contraction for the $\alpha$-posterior}
Similar to \prettyref{lemma:posterior_contraction}, we have the following posterior contraction result for $\alpha$-posteriors in \eqref{eq:alpha-posterior}. 

\begin{lmm}\label{lemma:posterior_contraction_alpha}
	Let $\Pi$ be thick over $\calP$ in the sense of \prettyref{definition:PoP}, with rate function $E$, and $\alpha\in (0,1]$. Let $\varepsilon_n, r_{n,\alpha}>0$ satisfy $\log \Nloc(\varepsilon_n, \calP, H) \le n\varepsilon_n^2$ and $E(n^{-2}, r_{n,\alpha}) \le \alpha nr_{n,\alpha}^2$.
	Then there exist absolute constants $c, C >0$ such that for every $G_0\in \calG$ and every $\varepsilon^2\ge \varepsilon_n^2 + r_{n,\alpha}^2 + \frac{1}{\alpha n}$, 
	\begin{align*}
		f_{G_0}^{\otimes (n-1)}\bpth{ \bsth{\bE_{G\sim \Pi^\alpha_{G|X^{n-1}}}[H^2(f_{G_0}, f_G)] > C\varepsilon^2 }} \le n^{-1} + e^{-c\alpha n\varepsilon^2}, 
	\end{align*}
	where $\Pi^\alpha$ is the $\alpha$-posterior defined in \eqref{eq:alpha-posterior} with $\ntest$ replaced by $n$. 
\end{lmm}

Compared with \prettyref{lemma:posterior_contraction}, the inequality $E(n^{-2},r)\le \alpha nr^2$ involves an additional factor of $\alpha$. When $E(n^{-2},r)=\widetilde{O}(1)$, this gives $r_{n,\alpha}^2=\widetilde{O}(\frac{1}{\alpha n})$, hence the squared Hellinger rate in posterior contraction increases from $\widetilde{O}(\frac{1}{n})$ to $\widetilde{O}(\frac{1}{\alpha n})$. For $(n,\alpha) = (\ntest,\frac{n}{\ntest})$, this rate is $\widetilde{O}(\frac{1}{n})$ rather than $\widetilde{O}(\frac{1}{\ntest})$. This additional factor of $\frac{1}{\alpha}$ is unsurprising, as it can be easily seen in the special example of normal location models with a normal prior. This change of scaling has also been observed in \cite[Theorem 3.2]{bhattacharya2019bayesian} on fractional posteriors with a different error probability.

Finally, using generalizations of \prettyref{lemma:posterior_mean_training} and \prettyref{lemma:hellinger-to-regret-1} (cf. \prettyref{lemma:length-generalization} and \prettyref{lemma:hellinger-to-regret-2} in the appendix), we deduce the final regret bound from the Hellinger rate. This precisely explains why the regret upper bound for length generalization saturates at $\widetilde{O}(\frac{1}{n})$ even if $\ntest \gg n$.

\section{Discussions}\label{sec:discussion}
In this section, we provide discussions on pretraining on a finite number of batches, and extension beyond our Poisson EB setting with a bounded support of the prior. 
Deferred proofs are in \Cref{app:deferred-proofs-disc}. 

\subsection{Finite number of batches}\label{subsec:batch}
The population ERM solution in \Cref{eq:ERM} coincides with the hierarchical Bayes estimator $\theta_\Pi^n$ only if $M\to\infty$. In practice, the number of batches $M$ used in pretraining is finite and large. To this end, 
we consider a finite-sample ERM with a finite $M$ and define
\begin{align}\label{eq:ERM-PI}
	\widehat{\theta}^n = \argmin_{f\in \calF^{\mathrm{PE}}} \frac{1}{M}\sum_{m=1}^M \| \theta^{n,(m)} - f(X^{n,(m)}) \|_2^2,  
\end{align}
where $\calF^{\mathrm{PE}}$ denotes the class of all permutation-equivariant functions $f(X^n)$, i.e. $\pi\circ f(X^n) = f(\pi\circ X^n)$ for all permutations $\pi\in S_n$. 
Note that most EB estimators (Robbins, NPMLE, ERM-monotone, etc.), 
as well as a transformer without positional encoding or masking, 
belong to $\calF^{\mathrm{PE}}$; any estimator $\widehat{\theta}^n$ can also be symmetrized into a PE estimator $\widehat{\theta}^{n,\mathrm{PE}} = \frac{1}{n!}\sum_{\pi\in S_n}\pi^{-1}\circ \widehat{\theta}^n(\pi\circ X^n)$ without increasing the pointwise regret. The performance of this permutation-equivariant ERM in the Poisson model is summarized in the following result. 

\begin{lmm}\label{lemma:finite_M}
	Let $r_n$ be defined in \Cref{thm:general}. The estimator in \eqref{eq:ERM-PI} with $M \ge \exp(C(\frac{\log^2 n}{\log\log n}+nr_n^2))$ and an absolute constant $C=C(A)$ satisfies the same regret guarantee in \Cref{thm:general}. 
\end{lmm}

\prettyref{lemma:finite_M} shows that for the simple PoP in \Cref{alg:universal_prior} to attain the same regret of \Cref{thm:simple_prior}, $M=\exp(O(\frac{\log^2 n}{\log\log n}))$ batches suffices. 
This is superpolynomial in $n$, but it agrees with the practical intuition that pretraining typically requires a huge amount of training data. We leave it to future study if a better condition on $M$ can be obtained by assuming different structures of $\widehat{\theta}^n$.

\subsection{More discussions on Poisson EB}\label{sec:subexpo}
The compact support condition $G\in \calP([0,A])$ can be generalized. First assume that $G_0$ is a subexponential prior, i.e. $G_0\in \subexpo(s) := \{G: \forall t\ge 0: \mathbb{P}_G[\theta > t] < 2e^{-t/s}\}$. Since $G_0\in \subexpo(s)$ is effectively supported on $[0,O(\log n)]$, we modify \Cref{alg:universal_prior} to pick $k=\lceil c_0(s)\log n\rceil$ atoms for $G\in \subexpo(s)$, chosen uniformly at random on $[0, c_1(s)\log n]$. 
The following regret bound holds for the hierarchical Bayes estimator $\theta_\Pi^n$, which is near-optimal \citep{polyanskiy2021sharp}.   
\begin{thm}\label{thm:simple_prior_subexpo}
	For large enough hyperparameters $c_0(s), c_1(s)>0$, the hierarchical Bayes estimator $\theta_\Pi^n$ achieves a worst-case regret of $O_s(\frac{\log^4 n}{n})$ over all $G_0\in \subexpo(s)$. 
\end{thm}
The next case is the regime where $G_0$ is supported on $[0,A]$ with $A = A_n \gg \log n$, where \cite{shen2022empirical} showed that the optimal regret scales as $\widetilde{\Theta}(\frac{A^{1.5}}{n})$ for $A = O(n^2)$. This regret bound can be recovered by the hierarchical Bayes estimator with $k=\widetilde{\Theta}(\sqrt{A})$ atoms drawn uniformly at random on $[0,A]$, \emph{provided that} one can establish $-\log \Pi_{X_n|X^{n-1}} = O(\mathsf{polylog}(n))$ with a sufficiently high probability over the randomness of the test data $X^n\sim f_{G_0}^{\otimes n}$. This condition is due to the regularization parameter $\rho = \exp(-O(\mathsf{polylog}(n)))$ used in the regret-Hellinger inequality \cite[Theorem 3.3]{shen2022empirical}. We leave the verification of this condition as a conjecture.

Finally we comment on a potential route toward removing the extra $\log n$ factor in \Cref{thm:simple_prior} relative to the optimal regret. Specializing to Poisson mixtures, a peeling argument into shells in \cite[Theorem 2.4]{ghosal2000convergence} could improve the rate of posterior contraction, \emph{provided that} one can show a local prior-doubling bound of the form $\log(\Pi(\{G: H^2(f_{G_0},f_G)\le Cr^2 \}) / \Pi(\{G: \chi^2(f_{G_0}\|f_G)\le r^2 \})) \le D=O(\frac{\log n}{\log\log n})$. Such a bound is plausible as $f_G$ is effectively supported on $[0,D]$, but a rigorous proof faces two obstacles. First, unlike $H^2(f_{G_0},f_G)\asymp \chi^2(f_{G_0}\|f_G)$ for Gaussian mixtures \citep{jia2023entropic}, this is generally false for Poisson mixtures due to the singularity of $\Poi(\theta)$ at $\theta= 0$. Second, although \cite{gassiat2014local} establishes a local equivalence between $H(f_G,f_{G_0})$ and a weighted $L_2$ norm, a uniform non-asymptotic version of such a result is currently unavailable, which prevents us from carrying out the volume calculation. 

\subsection{Gaussian EB}\label{subsec:Gaussian_EB}

In Gaussian EB, we have $\theta^n \sim G_0^{\otimes n}$ with some prior $G_0\in \calP([-A,A])$, and $X^n|\theta^n \sim \otimes_{i=1}^n \calN(\theta_i, 1)$.  
\Cref{alg:universal_prior} is modified as follows: we still take $k=\lceil c_0\frac{\log n}{\log\log n}\rceil $ for a large enough constant $c_0>0$, but now draw $\lambda_1,\dots,\lambda_k\sim \Unif([-A,A])$ and replace $\Poi(\lambda_j)$ by $\calN(\lambda_j,1)$. The following theorem summarizes the regret bound of the resulting hierarchical Bayes estimator $\theta_\Pi^n$, which is near-optimal compared with the lower bound $\Omega(\frac{1}{n}(\frac{\log n}{\log\log n})^2)$ in \cite{polyanskiy2021sharp}.  

\begin{thm}\label{thm:simple_prior_gaussian}
	For a large enough hyperparameter $c_0>0$, the hierarchical Bayes estimator $\theta_\Pi^n$ achieves a worst-case regret of $O_A(\frac{\log^3 n}{n(\log\log n)^2})$ over all $G_0\in \calP([-A,A])$. 
\end{thm}

\subsection{Function estimation}
Motivated by the recent work \citep{kang2026function}, our result also extends to the setting where one is interested in estimating a function $g(\theta)$ of $\theta$. In this case, the hierarchical Bayes estimator becomes $g_{\Pi}^n = \bE_\Pi[g(\theta^n)|X^n]$, and the oracle estimator is $g_{G_0}(X) = \bE_{G_0}[g(\theta)|X]$. For the special case of the polynomial function $g(\theta) = \theta^p$ with $p\in \naturals$, we establish the following regret bound which is near-optimal compared with the minimax regret of $\Theta(\frac{1}{n}(\frac{\log n}{\log\log n})^{p+1})$ in \cite{kang2026function}. 

\begin{thm}\label{thm:simple_prior_polynomial}
	Let the PoP $\Pi$ be the same as \Cref{alg:universal_prior}, with a large enough hyperparameter $c_0>0$. Then for $g(\theta)=\theta^p$, the hierarchical Bayes estimator $g_\Pi^n$ achieves a worst-case regret of $O_{A,p}(\frac{1}{n}\frac{\log^{p+2}n}{(\log\log n)^{p+1}})$ over all $G_0\in \calP([0,A])$.  
\end{thm}

\section{Numerical experiments}\label{sec:numerics}

\UnifiedFigure{t}{fig:performance}{
Regrets of different estimators with different test sequence lengths and test priors. For all transformers, the training sequence length is fixed to be $n=512$ (indicated by the vertical dotted green line)
}{
    \subfloat[Test prior from neural PoP]{
    		\label{fig:neural-pop}
    		\includegraphics[scale=0.35]{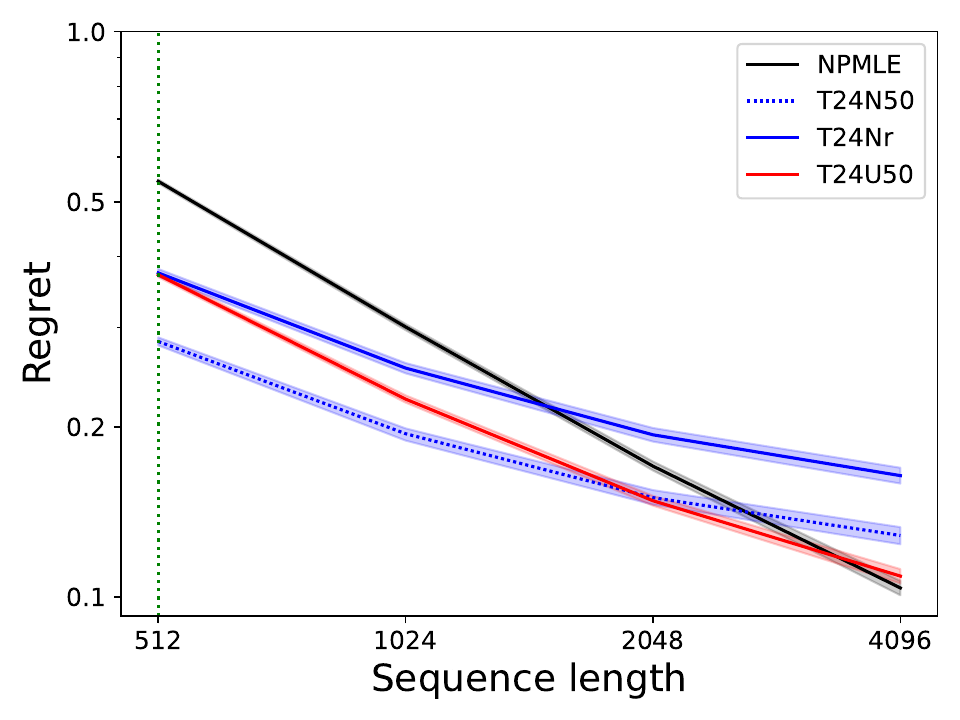}
    	}
        \subfloat[Test prior from multinomial PoP]{
        	\label{fig:multn-pop}
        	\includegraphics[scale=0.35]{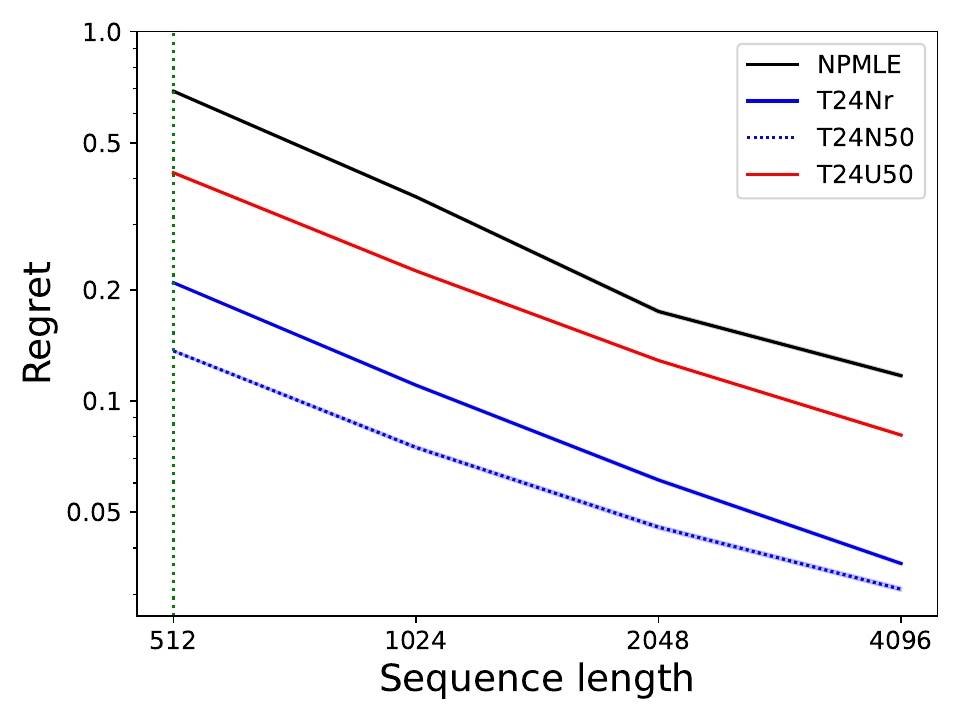}
        }
}

\UnifiedFigure{t}{fig:hb_regret}{The regret of the hierarchical Bayes estimator, as well as its mean squared distance to the trained transformer, under simple training PoPs $\Pi_m = \frac{1}{m}\sum_{i=1}^m G_i^{\otimes n}$.}
    {
        \subfloat[$m = 2$]{
			\includegraphics[scale=0.28]{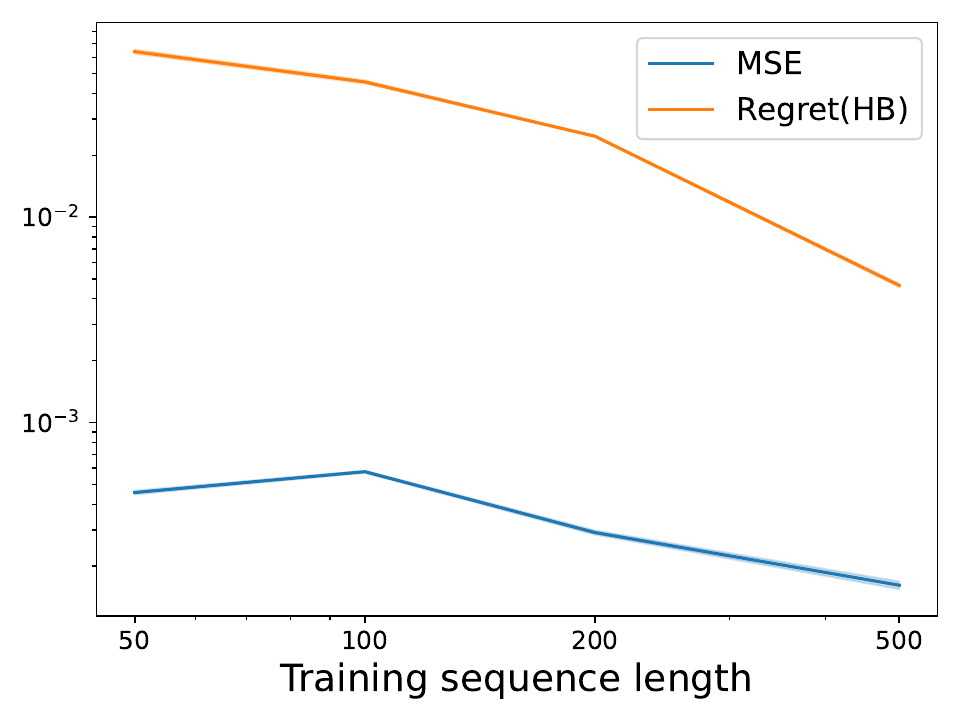}
		}
        \subfloat[$m = 5$]{
			\includegraphics[scale=0.28]{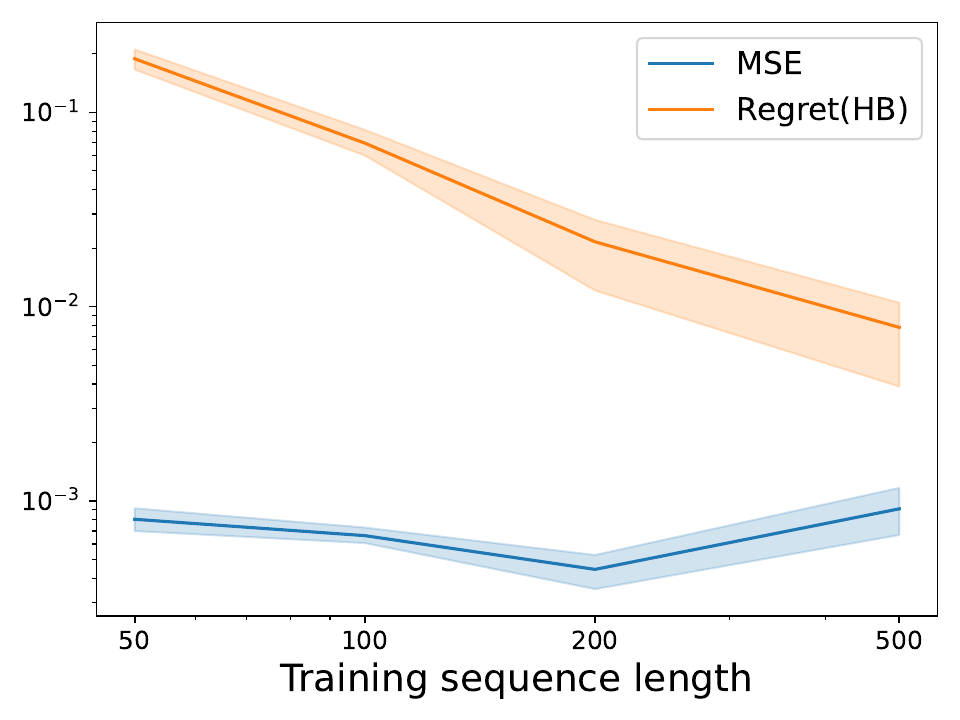}
		}
        \subfloat[$m = 10$]{
			\includegraphics[scale=0.28]{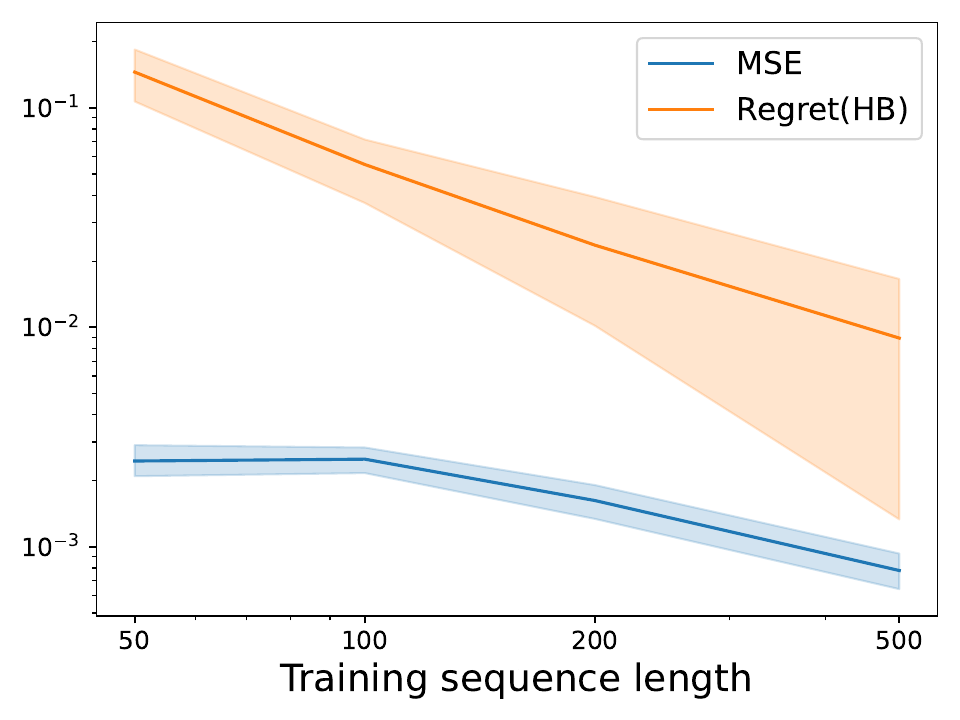}
		}
    }

\subsection{Regret performance and length generalization}\label{sec:performance}
Our first experiment validates the universality of our simple PoP in \Cref{alg:universal_prior} through transformers.
Specifically, we pretrain a transformer T24U50 according to \Cref{alg:universal_prior} with $A = 50, k = 10$, and $n = 512$, where both the training procedure and transformer architecture are identical to those in \cite{teh2025solving}. The test priors are generated from the neural and multinomial PoPs also introduced in \cite{teh2025solving}, and we compare against the classical NPMLE-based estimator as well as two other pretrained transformers (T24N50 and T24Nr) from \cite{teh2025solving}. For each test sequence length $\ntest$, \Cref{fig:performance} displays the average regret of these estimators over 4096 runs, across different test priors. We make the following observations.
First, when $\ntest = n$, all pretrained transformers outperform the NPMLE-based estimator, despite being a strong baseline, under both test priors.
Second, as $\ntest$ increases, the regrets of all pretrained transformers continue to decrease, although this improvement may saturate for large $\ntest$. These observations are consistent with both the theory of universal priors in \Cref{thm:simple_prior} and the length generalization results in \Cref{thm:length-generalization}.
Finally, we emphasize that our experiments with the simple PoP are primarily intended as a proof of concept, and we do not claim its practical optimality; we believe that a more finely tuned PoP (as used by the other two transformers) can achieve better empirical performance.

\subsection{Evidence on Bayesian inference}\label{sec:alpha-posterior}
Our second set of experiments is designed to validate the assumptions underlying our main theorems, namely that (1) the pretrained transformer approximately implements the Bayes estimator under the training PoP, and (2) for length generalization, the pretrained transformer indeed performs Bayesian inference under $\alpha$-posteriors.

\UnifiedFigure{t}{fig:alpha-posterior-2prior}{Plots of the mean squared distance between transformer output and the hierarchical Bayes estimator using various $\alpha$-posteriors, with different training lengths $n$ and test lengths $\ntest$. This distance is indeed minimized at $\alpha \simeq \frac{n}{\ntest}$.}
	{\subfloat[$m = 2, n=50$]{
			\includegraphics[scale=0.37]{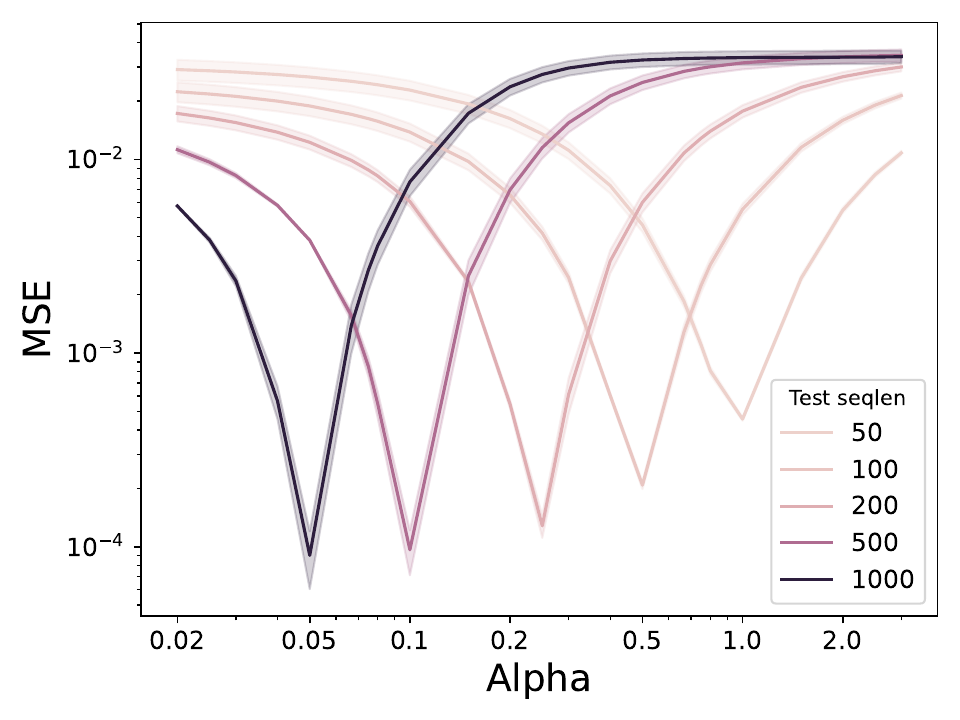}
		}
	  \subfloat[$m = 2, n=100$]{
	  	\includegraphics[scale=0.37]{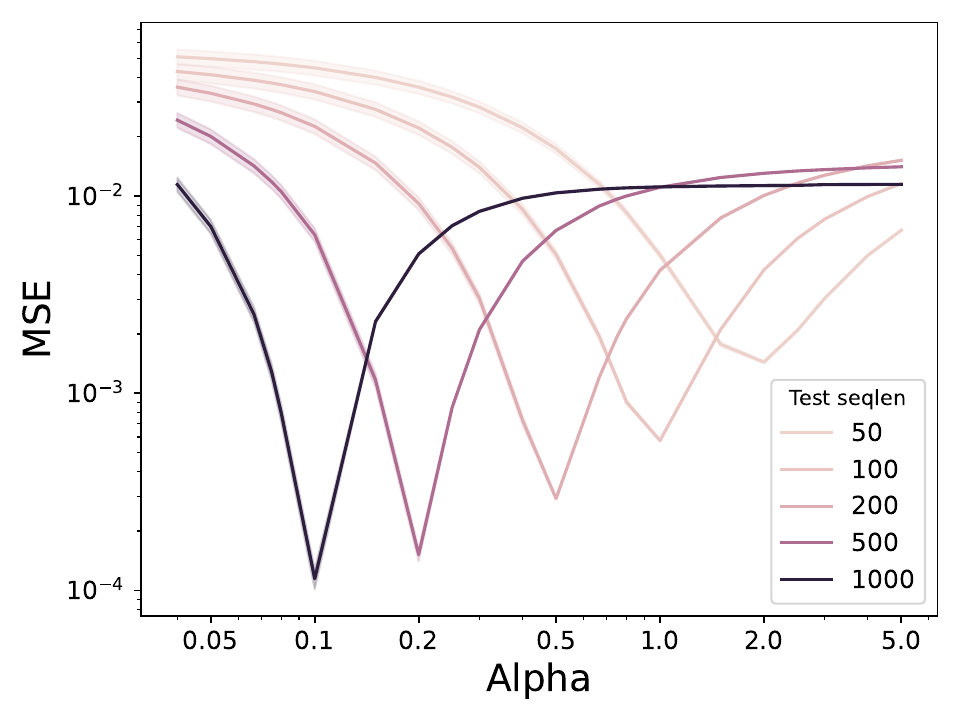}
	  }
      
      \subfloat[$m = 2, n=200$]{
      	\includegraphics[scale=0.37]{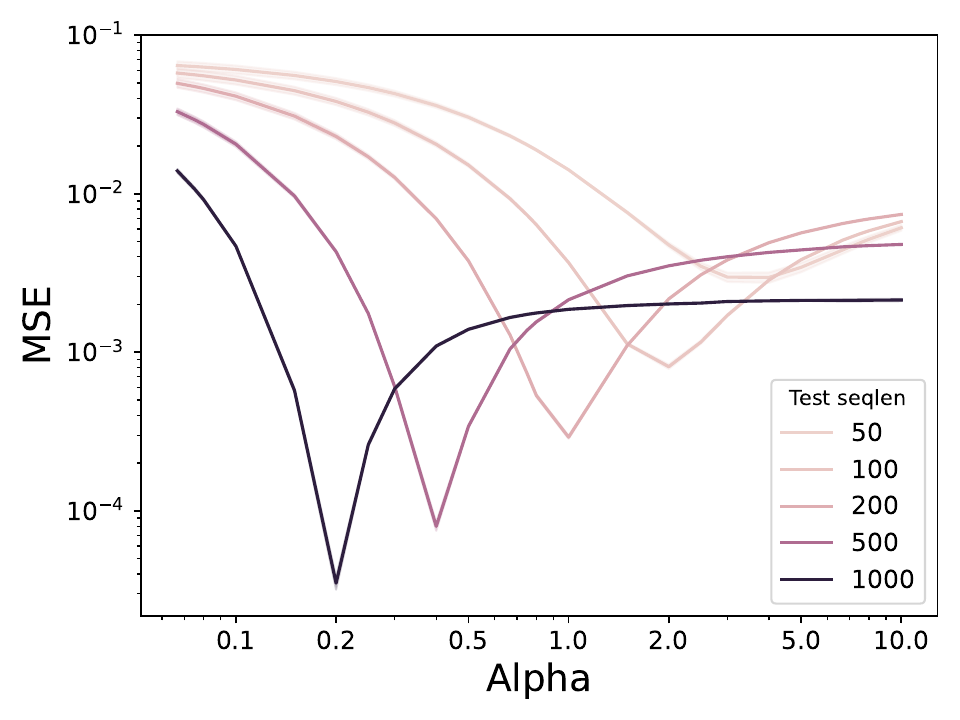}
      }
      \subfloat[$m = 2, n=500$]{
      	\includegraphics[scale=0.37]{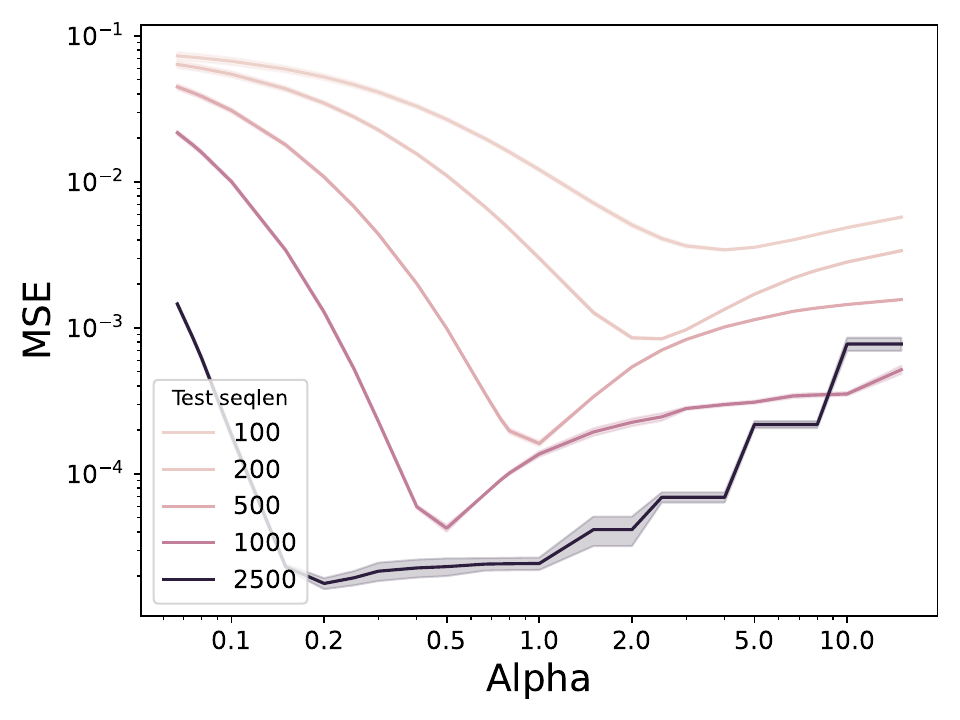}
      }
    }

For (1), we compare our pretrained transformer with the true hierarchical Bayes (HB) estimator $\theta_\Pi^n$, in scenarios where finding $\theta_\Pi^n$ is computationally easy. Specifically, we choose a simple PoP $\Pi_m = \frac{1}{m}\sum_{i=1}^m \delta_{G_i}$ with small $m$, where $G_1,\dots,G_m$ are randomly generated from our PoP in \Cref{alg:universal_prior}, conditioned on the event that $\theta_\Pi^n$ does not perform uniformly well on $G_1,\dots,G_m$. This conditioning ensures that the regret incurred by both $\theta_\Pi^n$ and the transformer is reasonably large and \emph{much larger} than their difference; implementation details are deferred to the appendix. \Cref{fig:hb_regret} displays the regret of $\theta_\Pi^n$, along with its mean squared distance to the trained transformer, for $m\in \{2,5,10\}$ under the training PoP. We observe that, relative to the regret of $\theta_\Pi^n$, the trained transformer is remarkably close to $\theta_\Pi^n$, indicating that the pretrained transformer indeed performs approximate Bayesian inference under the training distribution.

For (2), we use the transformer trained on the above PoP $\Pi_m$, compute its outputs for various test lengths $\ntest$, and compare them with the HB estimator based on $\alpha$-posteriors for different values of $\alpha$. \Cref{fig:alpha-posterior-2prior} presents the results for $m=2$ across multiple training lengths $n$ and test lengths $\ntest$; results for larger $m$ are deferred to the appendix. The quality of fit is striking: for each pair $(n, \ntest)$, the transformer output is extremely close to that of the HB estimator using $\alpha$-posteriors, with $\alpha\simeq \frac{n}{\ntest}$. This demonstrates that our $\alpha$-posterior conclusion in \Cref{thm:length-generalization}, although derived under a specific length generalization model, aligns well with the experimental findings.

\paragraph{Acknowledgement.} We thank Nikolaos Ignatiadis for sharing the concurrent work \citep{ignatiadis2026compound}, and anonymous reviewers for helpful comments on improving the organization of this paper. N.C. is supported by the Dean’s Undergraduate Research Fund Grant at NYU. 
The authors acknowledge the MIT SuperCloud and Lincoln Laboratory Supercomputing Center for providing compute resources that have contributed to the experimental results reported in \Cref{sec:numerics}.

\clearpage 

\bibliographystyle{alpha}
\bibliography{refs}

\newcommand{\etalchar}[1]{$^{#1}$}
\begin{thebibliography}{HNWSW25}

\bibitem[ADM25]{aggarwal2025bayesian}
Naman Aggarwal, Siddhartha~R Dalal, and Vishal Misra.
\newblock The {B}ayesian geometry of transformer attention.
\newblock {\em arXiv preprint arXiv:2512.22471}, 2025.

\bibitem[AM24]{ahuja2024provable}
Kartik Ahuja and Amin Mansouri.
\newblock On provable length and compositional generalization.
\newblock {\em arXiv preprint arXiv:2402.04875}, 2024.

\bibitem[Ant74]{antoniak1974mixtures}
Charles~E Antoniak.
\newblock Mixtures of {Dirichlet} processes with applications to bayesian
  nonparametric problems.
\newblock {\em The Annals of Statistics}, pages 1152--1174, 1974.

\bibitem[ASA{\etalchar{+}}23]{akyurek2023learning}
Ekin Aky{\"u}rek, Dale Schuurmans, Jacob Andreas, Tengyu Ma, and Denny Zhou.
\newblock What learning algorithm is in-context learning? {I}nvestigations with
  linear models.
\newblock In {\em The Eleventh International Conference on Learning
  Representations}, 2023.

\bibitem[AWA{\etalchar{+}}22]{anil2022exploring}
Cem Anil, Yuhuai Wu, Anders Andreassen, Aitor Lewkowycz, Vedant Misra, Vinay
  Ramasesh, Ambrose Slone, Guy Gur-Ari, Ethan Dyer, and Behnam Neyshabur.
\newblock Exploring length generalization in large language models.
\newblock {\em Advances in Neural Information Processing Systems},
  35:38546--38556, 2022.

\bibitem[BCW{\etalchar{+}}23]{bai2023transformers}
Yu~Bai, Fan Chen, Huan Wang, Caiming Xiong, and Song Mei.
\newblock Transformers as statisticians: Provable in-context learning with
  in-context algorithm selection.
\newblock {\em Advances in neural information processing systems},
  36:57125--57211, 2023.

\bibitem[BGR13]{brown2013poisson}
Lawrence~D Brown, Eitan Greenshtein, and Ya'acov Ritov.
\newblock The {P}oisson {C}ompound {D}ecision {P}roblem {R}evisited.
\newblock {\em Journal of the American Statistical Association}, pages
  741--749, 2013.

\bibitem[Bir83]{birge1983approximation}
Lucien Birg{\'e}.
\newblock Approximation dans les espaces m{\'e}triques et th{\'e}orie de
  l'estimation.
\newblock {\em Zeitschrift f{\"u}r Wahrscheinlichkeitstheorie und verwandte
  Gebiete}, 65(2):181--237, 1983.

\bibitem[BK12]{bickel2012semiparametric}
PJ~Bickel and BJK Kleijn.
\newblock The semiparametric {Bernstein-von Mises} theorem.
\newblock {\em The Annals of Statistics}, 40(1):206--237, 2012.

\bibitem[BPY19]{bhattacharya2019bayesian}
Anirban Bhattacharya, Debdeep Pati, and Yun Yang.
\newblock Bayesian {F}ractional {P}osteriors.
\newblock {\em The Annals of Statistics}, 47(1):39--66, 2019.

\bibitem[BZ22]{barbehenn2022nonparametric}
Alton Barbehenn and Sihai~Dave Zhao.
\newblock A nonparametric regression alternative to empirical {Bayes}
  approaches to simultaneous estimation.
\newblock {\em arXiv preprint arXiv:2205.00336}, 2022.

\bibitem[CC25]{chen2025score}
Gongyu Chen and Ying Cui.
\newblock Score-based diffusion modeling for nonparametric empirical {Bayes} in
  heteroscedastic {Gaussian} mixtures.
\newblock In {\em Thirty-ninth Annual Conference on Neural Information
  Processing Systems}, 2025.

\bibitem[CW26]{chen2026sharp}
Jiafeng Chen and Yihong Wu.
\newblock Sharp regret-{Hellinger} bounds for {Gaussian} empirical {Bayes} via
  polynomial approximation.
\newblock {\em arXiv preprint arXiv:2605.02070}, 2026.

\bibitem[Dat91]{datta1991asymptotic}
Somnath Datta.
\newblock Asymptotic optimality of bayes compound estimators in compact
  exponential families.
\newblock {\em The Annals of Statistics}, 19(1):354--365, 1991.

\bibitem[DL81]{deely1981bayes}
JJ~Deely and DV~Lindley.
\newblock Bayes empirical {Bayes}.
\newblock {\em Journal of the American Statistical Association},
  76(376):833--841, 1981.

\bibitem[Dud18]{dudley2018real}
Richard~M Dudley.
\newblock {\em Real analysis and probability}.
\newblock Chapman and Hall/CRC, 2018.

\bibitem[Efr19]{efron2019bayes}
Bradley Efron.
\newblock Bayes, oracle {B}ayes and empirical {B}ayes.
\newblock {\em Statistical science}, 34(2):177--201, 2019.

\bibitem[Efr24]{efron2024empirical}
Bradley Efron.
\newblock {Empirical Bayes: Concepts and Methods}.
\newblock In {\em Handbook of Bayesian, Fiducial, and Frequentist Inference},
  pages 8--34. Chapman and Hall/CRC, 2024.

\bibitem[Fan53]{fan1953minimax}
Ky~Fan.
\newblock Minimax theorems.
\newblock {\em Proceedings of the National Academy of Sciences}, 39(1):42--47,
  1953.

\bibitem[FdHP25]{furuya2025transformers}
Takashi Furuya, Maarten~V de~Hoop, and Gabriel Peyr{\'e}.
\newblock Transformers are {U}niversal {I}n-context {L}earners.
\newblock In {\em The Thirteenth International Conference on Learning
  Representations}, 2025.

\bibitem[Fer73]{ferguson1973bayesian}
Thomas~S Ferguson.
\newblock A {Bayesian} analysis of some nonparametric problems.
\newblock {\em The annals of statistics}, pages 209--230, 1973.

\bibitem[GGVDV00]{ghosal2000convergence}
Subhashis Ghosal, Jayanta~K Ghosh, and Aad~W Van Der~Vaart.
\newblock Convergence rates of posterior distributions.
\newblock {\em Annals of Statistics}, pages 500--531, 2000.

\bibitem[GIKL25]{ghosh2025stein}
Sulagna Ghosh, Nikolaos Ignatiadis, Frederic Koehler, and Amber Lee.
\newblock Stein's unbiased risk estimate and {H}yv{\"a}rinen's score matching.
\newblock {\em arXiv preprint arXiv:2502.20123}, 2025.

\bibitem[GK17]{gu2017empirical}
Jiaying Gu and Roger Koenker.
\newblock {Empirical Bayesball remixed: Empirical Bayes methods for
  longitudinal data}.
\newblock {\em Journal of Applied Econometrics}, 32(3):575--599, 2017.

\bibitem[GLvdV08]{ghosal2008nonparametric}
Subhashis Ghosal, J{\"u}ri Lember, and Aad van~der Vaart.
\newblock Nonparametric {Bayesian} model selection and averaging.
\newblock {\em Electronic Journal of Statistics}, 2:63--89, 2008.

\bibitem[GTLV22]{garg2022can}
Shivam Garg, Dimitris Tsipras, Percy~S Liang, and Gregory Valiant.
\newblock What can transformers learn in-context? {A} case study of simple
  function classes.
\newblock {\em Advances in neural information processing systems},
  35:30583--30598, 2022.

\bibitem[GVdV17]{ghosal2017fundamentals}
Subhashis Ghosal and Aad~W Van~der Vaart.
\newblock {\em Fundamentals of nonparametric {Bayesian} inference}, volume~44.
\newblock Cambridge University Press, 2017.

\bibitem[GVH14]{gassiat2014local}
Elisabeth Gassiat and Ramon Van~Handel.
\newblock The local geometry of finite mixtures.
\newblock {\em Transactions of the American Mathematical Society},
  366(2):1047--1072, 2014.

\bibitem[HAMS21]{hospedales2021meta}
Timothy Hospedales, Antreas Antoniou, Paul Micaelli, and Amos Storkey.
\newblock Meta-learning in neural networks: A survey.
\newblock {\em IEEE transactions on pattern analysis and machine intelligence},
  44(9):5149--5169, 2021.

\bibitem[HMEH23]{hollmann2022tabpfn}
Noah Hollmann, Samuel M{\"u}ller, Katharina Eggensperger, and Frank Hutter.
\newblock Tab{P}{F}{N}: A {T}ransformer {T}hat {S}olves {S}mall {T}abular
  {C}lassification {P}roblems in a {S}econd.
\newblock In {\em The Eleventh International Conference on Learning
  Representations}, 2023.

\bibitem[HMP{\etalchar{+}}25]{hollmann2025accurate}
Noah Hollmann, Samuel M{\"u}ller, Lennart Purucker, Arjun Krishnakumar, Max
  K{\"o}rfer, Shi~Bin Hoo, Robin~Tibor Schirrmeister, and Frank Hutter.
\newblock Accurate predictions on small data with a tabular foundation model.
\newblock {\em Nature}, 637(8045):319--326, 2025.

\bibitem[HNWSW25]{han2025besting}
Yanjun Han, Jonathan Niles-Weed, Yandi Shen, and Yihong Wu.
\newblock Besting {G}ood--{T}uring: {O}ptimality of {N}on-{P}arametric
  {M}aximum {L}ikelihood for {D}istribution {E}stimation.
\newblock {\em arXiv preprint arXiv:2509.07355}, 2025.

\bibitem[HYB{\etalchar{+}}25]{huang2025formal}
Xinting Huang, Andy Yang, Satwik Bhattamishra, Yash Sarrof, Andreas Krebs,
  Hattie Zhou, Preetum Nakkiran, and Michael Hahn.
\newblock A formal framework for understanding length generalization in
  transformers.
\newblock In {\em The Thirteenth International Conference on Learning
  Representations}, 2025.

\bibitem[IK26]{ignatiadis2026compound}
Nikolaos Ignatiadis and Sid Kankanala.
\newblock Compound decisions and empirical {Bayes via Bayesian} nonparametrics.
\newblock {\em arXiv preprint arXiv:2602.20115}, 2026.

\bibitem[INL26]{izzo2025quantitative}
Zachary Izzo, Eshaan Nichani, and Jason~D Lee.
\newblock Quantitative bounds for length generalization in transformers.
\newblock In {\em International Conference on Learning Representations}, 2026.

\bibitem[IS25]{ignatiadis2025empirical}
Nikolaos Ignatiadis and Bodhisattva Sen.
\newblock {Empirical Bayes}.
\newblock {\em Lecture notes}, 2025.

\bibitem[JPTW23]{jana2023empirical}
Soham Jana, Yury Polyanskiy, Anzo~Z Teh, and Yihong Wu.
\newblock Empirical {B}ayes via {E}{R}{M} and {R}ademacher complexities: the
  poisson model.
\newblock In {\em The Thirty Sixth Annual Conference on Learning Theory}, pages
  5199--5235. PMLR, 2023.

\bibitem[JPW23]{jia2023entropic}
Zeyu Jia, Yury Polyanskiy, and Yihong Wu.
\newblock Entropic characterization of optimal rates for learning gaussian
  mixtures.
\newblock In {\em The Thirty Sixth Annual Conference on Learning Theory}, pages
  4296--4335. PMLR, 2023.

\bibitem[JPW25]{jana2025optimal}
Soham Jana, Yury Polyanskiy, and Yihong Wu.
\newblock Optimal empirical {B}ayes estimation for the {P}oisson model via
  minimum-distance methods.
\newblock {\em Information and Inference: A Journal of the IMA}, 14(4):iaaf027,
  2025.

\bibitem[JS61]{james1961estimation}
William James and Charles Stein.
\newblock Estimation with quadratic loss.
\newblock In {\em Proceedings of the fourth Berkeley symposium on mathematical
  statistics and probability}, volume~1, pages 361--379. University of
  California Press, 1961.

\bibitem[JYL{\etalchar{+}}26]{jia2026weak}
Jing Jia, Wei Yuan, Sifan Liu, Liyue Shen, and Guanyang Wang.
\newblock Weak diffusion priors can still achieve strong inverse-problem
  performance.
\newblock {\em arXiv preprint arXiv:2601.22443}, 2026.

\bibitem[JZ09]{jiang2009general}
Wenhua Jiang and Cun-Hui Zhang.
\newblock General {M}aximum {L}ikelihood {E}mpirical {B}ayes {E}stimation of
  {N}ormal {M}eans.
\newblock {\em The Annals of Statistics}, pages 1647--1684, 2009.

\bibitem[KPT26]{kang2026function}
Benjamin Kang, Yury Polyanskiy, and Anzo Teh.
\newblock Function estimation in the empirical {B}ayes setting.
\newblock {\em arXiv preprint arXiv:2601.18689}, 2026.

\bibitem[KW56]{kiefer1956consistency}
Jack Kiefer and Jacob Wolfowitz.
\newblock Consistency of the maximum likelihood estimator in the presence of
  infinitely many incidental parameters.
\newblock {\em The Annals of Mathematical Statistics}, pages 887--906, 1956.

\bibitem[LC86]{le1986asymptotic}
Lucien Le~Cam.
\newblock Asymptotic {M}ethods in {S}tatistical {D}ecision {T}heory.
\newblock {\em Springer Series in Statistics}, 1986.

\bibitem[LC98]{lehmann1998theory}
Erich~Leo Lehmann and George Casella.
\newblock {\em Theory of point estimation}.
\newblock Springer, 1998.

\bibitem[Lin83]{lindsay1983geometry}
Bruce~G Lindsay.
\newblock The geometry of mixture likelihoods: a general theory.
\newblock {\em The Annals of Statistics}, pages 86--94, 1983.

\bibitem[MBL{\etalchar{+}}25]{mittal2025amortized}
Sarthak Mittal, Niels~Leif Bracher, Guillaume Lajoie, Priyank Jaini, and Marcus
  Brubaker.
\newblock Amortized in-context bayesian posterior estimation.
\newblock {\em arXiv preprint arXiv:2502.06601}, 2025.

\bibitem[MBML25]{mittal2025context}
Sarthak Mittal, Yoshua Bengio, Nikolay Malkin, and Guillaume Lajoie.
\newblock In-context parametric inference: Point or distribution estimators?
\newblock {\em arXiv preprint arXiv:2502.11617}, 2025.

\bibitem[MHA{\etalchar{+}}22]{muller2022transformers}
Samuel M{\"u}ller, Noah Hollmann, Sebastian~Pineda Arango, Josif Grabocka, and
  Frank Hutter.
\newblock Transformers can do {Bayesian} inference.
\newblock In {\em International Conference on Learning Representations}, 2022.

\bibitem[MORV22]{medina2022robustness}
Marco~Avella Medina, Jos{\'e} Luis~Montiel Olea, Cynthia Rush, and Amilcar
  Velez.
\newblock On the {R}obustness to {M}isspecification of $\alpha$-posteriors and
  {T}heir {V}ariational {A}pproximations.
\newblock {\em Journal of Machine Learning Research}, 23(147):1--51, 2022.

\bibitem[MTH{\etalchar{+}}25]{ma2024tabdpt}
Junwei Ma, Valentin Thomas, Rasa Hosseinzadeh, Alex Labach, Jesse Cresswell,
  Keyvan Golestan, Guangwei Yu, Anthony~L Caterini, and Maks Volkovs.
\newblock {TabDPT}: Scaling tabular foundation models on real data.
\newblock {\em Advances in Neural Information Processing Systems},
  38:172692--172722, 2025.

\bibitem[MWS25]{ma2025provable}
Tianyi Ma, Tengyao Wang, and Richard~J Samworth.
\newblock Provable test-time adaptivity and distributional robustness of
  in-context learning.
\newblock {\em arXiv preprint arXiv:2510.23254}, 2025.

\bibitem[Nea00]{neal2000markov}
Radford~M Neal.
\newblock Markov chain sampling methods for dirichlet process mixture models.
\newblock {\em Journal of computational and graphical statistics},
  9(2):249--265, 2000.

\bibitem[PAG24]{panwar2024context}
Madhur Panwar, Kabir Ahuja, and Navin Goyal.
\newblock In-{C}ontext {L}earning through the {B}ayesian {P}rism.
\newblock In {\em The Twelfth International Conference on Learning
  Representations}, 2024.

\bibitem[PQFS24]{peng2024yarn}
Bowen Peng, Jeffrey Quesnelle, Honglu Fan, and Enrico Shippole.
\newblock Yarn: Efficient context window extension of large language models.
\newblock In {\em The Twelfth International Conference on Learning
  Representations}, 2024.

\bibitem[PRS14]{petrone2012bayes}
Sonia Petrone, Judith Rousseau, and Catia Scricciolo.
\newblock Bayes and {empirical Bayes}: do they merge?
\newblock {\em Biometrika}, pages 285--302, 2014.

\bibitem[PW21]{polyanskiy2021sharp}
Yury Polyanskiy and Yihong Wu.
\newblock Sharp regret bounds for empirical {B}ayes and compound decision
  problems.
\newblock {\em arXiv preprint arXiv:2109.03943}, 2021.

\bibitem[Rob51]{Rob51}
Herbert Robbins.
\newblock Asymptotically subminimax solutions of compound statistical decision
  problems.
\newblock {\em Proceedings of the second Berkeley symposium on mathematical
  statistics and probability}, abs/1904.10040:131--149, 1951.

\bibitem[Rob56]{Rob56}
Herbert Robbins.
\newblock An {E}mpirical {B}ayes {A}pproach to {S}tatistics.
\newblock In {\em Proceedings of the Third Berkeley Symposium on Mathematical
  Statistics and Probability, Volume 1: Contributions to the Theory of
  Statistics}. The Regents of the University of California, 1956.

\bibitem[RRFR25]{reuter2025can}
Arik Reuter, Tim~GJ Rudner, Vincent Fortuin, and David R{\"u}gamer.
\newblock Can transformers learn full {Bayesian} inference in context?
\newblock In {\em Forty-second International Conference on Machine Learning},
  2025.

\bibitem[RRP24]{rizzelli2024empirical}
Stefano Rizzelli, Judith Rousseau, and Sonia Petrone.
\newblock Empirical {B}ayes in {B}ayesian learning: understanding a common
  practice.
\newblock {\em arXiv preprint arXiv:2402.19036}, 2024.

\bibitem[SG20]{saha2020nonparametric}
Sujayam Saha and Adityanand Guntuboyina.
\newblock On the nonparametric maximum likelihood estimator for {G}aussian
  location mixture densities with application to {G}aussian denoising.
\newblock {\em The Annals of Statistics}, 48(2):738--762, 2020.

\bibitem[SGS25]{soloff2025multivariate}
Jake~A Soloff, Adityanand Guntuboyina, and Bodhisattva Sen.
\newblock Multivariate, heteroscedastic empirical {Bayes} via nonparametric
  maximum likelihood.
\newblock {\em Journal of the Royal Statistical Society Series B: Statistical
  Methodology}, 87(1):1--32, 2025.

\bibitem[Ste56]{stein1956inadmissibility}
Charles Stein.
\newblock Inadmissibility of the usual estimator for the mean of a multivariate
  normal distribution.
\newblock In {\em Proceedings of the Third Berkeley symposium on mathematical
  statistics and probability}, volume~1, pages 197--206, 1956.

\bibitem[SW01]{shen2001rates}
Xiaotong Shen and Larry Wasserman.
\newblock Rates of convergence of posterior distributions.
\newblock {\em The Annals of Statistics}, 29(3):687--714, 2001.

\bibitem[SW26]{shen2022empirical}
Yandi Shen and Yihong Wu.
\newblock Empirical {B}ayes estimation: When does $ g $-modeling beat $ f
  $-modeling in theory (and in practice)?
\newblock {\em The Annals of Statistics}, 54(1):146--175, 2026.

\bibitem[TJP25]{teh2025solving}
Anzo Teh, Mark Jabbour, and Yury Polyanskiy.
\newblock Solving empirical {B}ayes via transformers.
\newblock {\em arXiv preprint arXiv:2502.09844}, 2025.

\bibitem[vdVW96]{van1996weak}
Aad van~der Vaart and Jon~A Wellner.
\newblock {\em Weak {C}onvergence and {E}mpirical {P}rocesses: With
  {A}pplications to {S}tatistics}.
\newblock Springer Science \& Business Media, 1996.

\bibitem[VKWL21]{vandegar2021neural}
Maxime Vandegar, Michael Kagan, Antoine Wehenkel, and Gilles Louppe.
\newblock Neural empirical {B}ayes: Source distribution estimation and its
  applications to simulation-based inference.
\newblock In {\em International Conference on Artificial Intelligence and
  Statistics}, pages 2107--2115. PMLR, 2021.

\bibitem[VONR{\etalchar{+}}23]{von2023transformers}
Johannes Von~Oswald, Eyvind Niklasson, Ettore Randazzo, Jo{\~a}o Sacramento,
  Alexander Mordvintsev, Andrey Zhmoginov, and Max Vladymyrov.
\newblock Transformers learn in-context by gradient descent.
\newblock In {\em International Conference on Machine Learning}, pages
  35151--35174. PMLR, 2023.

\bibitem[WMB19]{wang2019comment}
Yixin Wang, Andrew~C Miller, and David~M Blei.
\newblock Comment: {V}ariational autoencoders as empirical {B}ayes.
\newblock 2019.

\bibitem[WS25]{wakayama2025context}
Tomoya Wakayama and Taiji Suzuki.
\newblock In-context learning is provably {Bayesian} inference: a
  generalization theory for meta-learning.
\newblock {\em arXiv preprint arXiv:2510.10981}, 2025.

\bibitem[WV11]{wu2011functional}
Yihong Wu and Sergio Verd{\'u}.
\newblock Functional properties of minimum mean-square error and mutual
  information.
\newblock {\em IEEE Transactions on Information Theory}, 58(3):1289--1301,
  2011.

\bibitem[WY16]{wu2016minimax}
Yihong Wu and Pengkun Yang.
\newblock Minimax rates of entropy estimation on large alphabets via best
  polynomial approximation.
\newblock {\em IEEE Transactions on Information Theory}, 62(6):3702--3720,
  2016.

\bibitem[WY20]{wu2020optimal}
Yihong Wu and Pengkun Yang.
\newblock Optimal estimation of gaussian mixtures via denoised method of
  moments.
\newblock {\em The Annals of Statistics}, 48(4):1981--2007, 2020.

\bibitem[XRLM22]{xie2022explanation}
Sang~Michael Xie, Aditi Raghunathan, Percy Liang, and Tengyu Ma.
\newblock An {E}xplanation of {I}n-context {L}earning as {I}mplicit {B}ayesian
  {I}nference.
\newblock In {\em International Conference on Learning Representations}, 2022.

\bibitem[ZAC{\etalchar{+}}24]{zhou2024transformers}
Yongchao Zhou, Uri Alon, Xinyun Chen, Xuezhi Wang, Rishabh Agarwal, and Denny
  Zhou.
\newblock Transformers can achieve length generalization but not robustly.
\newblock {\em arXiv preprint arXiv:2402.09371}, 2024.

\bibitem[ZBL{\etalchar{+}}24]{zhou2024algorithms}
Hattie Zhou, Arwen Bradley, Etai Littwin, Noam Razin, Omid Saremi, Josh
  Susskind, Samy Bengio, and Preetum Nakkiran.
\newblock What algorithms can transformers learn? {A} study in length
  generalization.
\newblock In {\em The Twelfth International Conference on Learning
  Representations}, 2024.

\bibitem[ZMSDH25]{zammit2025neural}
Andrew Zammit-Mangion, Matthew Sainsbury-Dale, and Rapha{\"e}l Huser.
\newblock Neural methods for amortized inference.
\newblock {\em Annual Review of Statistics and Its Application},
  12(1):311--335, 2025.

\end{thebibliography}

\clearpage 

\appendix

\section{Deferred Proofs in \Cref{sec:intro}}

\subsection{Proof of Proposition \ref{prop:worst-pop}}\label{app:worst-pop-proof}
We first utilize the continuity of regret as follows. 
\begin{lmm}\label{lmm:regret-cont}
	For all estimators $\widehat{\theta}^n: \naturals^n\to [0,A]^n$ and $G\in \mathcal{P}([0, A])$, 
	$\reg(\widehat{\theta}^n; G)$ is continuous in $G$ under the metric of weak convergence. 
\end{lmm}
The proof can be found in \prettyref{app:regret-cont-proof}. Next, note that since $[0, A]$ with the Euclidean metric is a compact Hausdorff space, 
by Prokhorov's Theorem, $\mathcal{P}([0, A])$ is compact under the metric of weak convergence (see, for example, \cite[Chapter 11.5]{dudley2018real}). 
This means the supremum can always be attained, i.e. for fixed estimator $\widehat{\theta}^n$, 
\[
\sup_{G\in\mathcal{P}([0, A])}\reg(\widehat{\theta}^n; G)
=\max_{G\in\mathcal{P}([0, A])}\reg(\widehat{\theta}^n; G).
\]

Next, we note that the LHS has the following equivalent formulation:
\[
\inf_{\widehat{\theta}^n}\max_{G\in\mathcal{P}([0, A])}\reg(\widehat{\theta}^n; G) = 
\inf_{\widehat{\theta}^n}\max_{\Pi\in\mathcal{P}(\mathcal{P}([0, A]))}\mathbb{E}_{G\sim\Pi}[\reg(\widehat{\theta}^n; G)]
\]
Clearly, $\mathbb{E}_{G\sim\Pi}[\reg(\widehat{\theta}^n; G)]$ is convex in $\widehat{\theta}^n$ for fixed $\Pi$ and affine in $\Pi$ for fixed $\widehat{\theta}^n$. 
Therefore, combined with \prettyref{lmm:regret-cont}, we use the Ky Fan minimax theorem \citep[Theorem 2]{fan1953minimax} to establish that 
\begin{align*}
	\inf_{\widehat{\theta}^n}\max_{\Pi\in\mathcal{P}(\mathcal{P}([0, A]))}\mathbb{E}_{G\sim\Pi}[\reg(\widehat{\theta}^n; G)]
	=\max_{\Pi\in\mathcal{P}(\mathcal{P}([0, A]))}\inf_{\widehat{\theta}^n}\mathbb{E}_{G\sim\Pi}[\reg(\widehat{\theta}^n; G)]. 
\end{align*}
In addition, the final $\inf$ can be replaced by $\min$ as given $\Pi$, the minimizer is the hierarchical Bayes estimator 
$\theta_{\Pi}^n (x^n) = \mathbb{E}_{\Pi}[\theta^n | X^n = x^n]$. 

\subsubsection{Proof of \prettyref{lmm:regret-cont}}\label{app:regret-cont-proof}
We note the following expansion of regret:
\begin{align*}
	\reg(\widehat{\theta}^n; G) &= \frac 1n\mathbb{E}_{X^n\sim f_G^{\otimes n}} [\|\widehat{\theta}^n(X^n) - \theta_G(X^n)\|_2^2] \\
	&=\frac 1n\sum_{x^n\in\naturals^n} \left(\prod_{i=1}^n f_G(x_i) \right)(\|\widehat{\theta}^n\|_2^2 + \|\theta_G\|_2^2 - 2\widehat{\theta}^n\cdot \theta_G).
\end{align*}
Since $f_G(x)$ decays uniformly in $G\in \calP([0,A])$ for $x\gg A$, by dominated convergence, it suffices to show that 
for each $x^n\in \naturals^n$, each of the following three terms is continuous in $G$ under the metric of weak convergence: 
\[
L_G(x^n) \|\widehat{\theta}^n(x^n)\|_2^2, \quad L_G(x^n)\cdot\widehat{\theta}^n(x^n)\cdot \theta_G(x^n), \quad \text{ and }
L_G(x^n) \|\theta_G(x^n)\|_2^2, 
\]
where $L_G(x^n) = \prod_{i=1}^n f_G(x_i)$ is the likelihood of $x^n$ under $f_G$. 

Now for each $x\ge 0$, the Poisson PMF $\Poi(x;\theta)$ is bounded and continuous in $\theta$. 
Therefore, the marginal PMF $f_G(x)=\bE_G[\Poi(x;\theta)]$ is continuous in $G$ under the metric of weak convergence, and so is the product $L_G(x^n)$. 
This gives the continuity of the first term $L_G(x^n) \|\widehat{\theta}^n(x^n)\|_2^2$. 

For the other two terms, we use the following identity of the Bayes estimator $\theta_G(x) = (x + 1)\frac{f_G(x + 1)}{f_G(x)}$. 
In addition, it suffices to consider each coordinate separately; w.l.o.g. we will take the first coordinate. The cross term now becomes 
\begin{align*}
	L_G(x^n)\cdot(\widehat{\theta}^n(x^n))_1\cdot \theta_G(x_1)
	&=\left(\prod_{i=1}^n f_G(x_i) \right)(\widehat{\theta}^n(x^n))_1\theta_G(x_1)
	\\
	&=(x_1+1)\left(\prod_{i=2}^n f_G(x_i) \right)(\widehat{\theta}^n(x^n))_1f_G(x_1+1), 
\end{align*}
which is also continuous in $G$. 
The final term, meanwhile, becomes 
\[
\left(\prod_{i=1}^n f_G(x_i) \right)\theta_G(x_1)^2
=(x_1+1)^2\left(\prod_{i=2}^n f_G(x_i) \right) \frac{f_G(x_1+1)^2}{f_G(x_1)}, 
\]
where the last term is defined to be $0$ if $f_G(x_1)=0$. 
If $f_G(x_1) > 0$, its continuity follows from the continuity of $f_G$ in $G$. As for the continuity at $f_G(x_1)=0$, note that
\begin{align*}
	0\le \frac{f_G(x_1+1)^2}{f_G(x_1)} = \frac{(\bE_G[e^{-\theta}\frac{\theta^{x_1+1}}{(x_1+1)!}])^2}{\bE_G[e^{-\theta}\frac{\theta^{x_1}}{x_1!}]} \le \frac{x_1!}{[(x_1+1)!]^2}\bE_G\bqth{ e^{-\theta} \theta^{x_1+2} } \le \frac{A^2}{(x_1+1)^2} f_G(x_1)
\end{align*}
by Cauchy--Schwarz and $G\in \calP([0,A])$. Taking $f_G(x_1)\to 0$ indeed gives the desired continuity. 

\subsection{Proof of Proposition \ref{prop:admissible}}
For the first statement, suppose that $\theta_\Pi^n$ is inadmissible and dominated by another estimator $\widehat{\theta}^n$. By inadmissibility, 
\begin{align*}
	0 &\le \bE_{G\sim \Pi}[\reg(\theta_\Pi; G)] - \bE_{G\sim \Pi}[\reg(\widehat{\theta}^n; G)] \\
	&= \bE_{G\sim \Pi}\bE_{G}\qth{ \frac{1}{n}\| \theta_\Pi^n(X^n) - \theta^n \|_2^2 - \frac{1}{n}\| \widehat{\theta}^n(X^n) - \theta^n \|_2^2 } \\
	&= \bE_{G\sim \Pi} \bE_{G} \qth{-\frac{2}{n}\langle \widehat{\theta}^n(X^n) - \theta_\Pi^n(X^n), \theta_\Pi^n(X^n) - \theta^n \rangle -\frac{1}{n}\|\widehat{\theta}^n(X^n) - \theta_\Pi^n(X^n)\|_2^2 } \\
	&= -\frac{2}{n}\bE_{X^n} \langle \widehat{\theta}^n(X^n) - \theta_\Pi^n(X^n), \theta_\Pi^n(X^n) - \bE_{\Pi}[\theta^n|X^n] \rangle - \frac{1}{n} \bE_{G\sim \Pi} \bE_{G} \|\widehat{\theta}^n(X^n) - \theta_\Pi^n(X^n)\|_2^2 \\
	&= -\frac{1}{n}\bE_{G\sim \Pi} \bE_{G} \|\widehat{\theta}^n(X^n) - \theta_\Pi^n(X^n)\|_2^2 \\
	&= -\frac{1}{n}\sum_{x^n \in \naturals^n} \|\widehat{\theta}^n(x^n) - \theta_\Pi^n(x^n)\|_2^2\cdot \bE_{G\sim \Pi}\qth{\prod_{i=1}^n f_G(x_i)}. 
\end{align*}
Since $f_G(x) > 0$ for all $G\neq \delta_0$ and $x\in \naturals$, the assumption $\Pi \neq \delta_{\delta_0}$ implies that $\bE_{G\sim \Pi}[\prod_{i=1}^n f_G(x_i)] > 0$ for all $x^n \in \naturals^n$. Therefore, $\widehat{\theta}^n(x^n) = \theta_\Pi^n(x^n)$ for all $x^n \in \naturals^n$, so $\widehat{\theta}^n$ cannot achieve a strictly smaller regret at some $G_0\in \calP([0,A])$. This is the desired contradiction, and $\theta_\Pi^n$ is admissible. 

For the second statement, we need the following property of the NPMLE-based estimator: for every integer $x\in [0,A]$, 
\begin{align}\label{eq:NPMLE_property}
	\widehat{\theta}^{\mathrm{NPMLE}}(x,\dots,x) = (x,\dots,x). 
\end{align}
To see this, recall that \cite[Theorem 1]{jana2025optimal} shows that the NPMLE $\widehat{G}$ is always supported on $[\min_i X_i, \max_i X_i]$. Therefore, when $X^n = (x,\dots,x)$, the NPMLE is $\widehat{G} = \delta_x$, so the Bayes estimator is $\theta_{\widehat{G}}(x) = x$. We claim that, for integer $A\ge 3$, any admissible estimator cannot satisfy \Cref{eq:NPMLE_property} for all $x\in \{A-2,A-1,A\}$. 

To prove the claim, assume by contradiction that there exists an admissible $\widehat{\theta}: \mathbb{Z}_+^n\to\mathbb{R}_+^n$ that satisfies \Cref{eq:NPMLE_property} for $x\in \{A-2,A-1,A\}$. By the complete class theorem (e.g. \cite[Theorem 7.15]{lehmann1998theory} with restriction $G\neq \delta_0$), this $\widehat{\theta}$ must be a pointwise limit of Bayes estimators, i.e. there exists a sequence of PoPs $\Pi_m\in \calP(\calP([0,A]))$ such that
\begin{align*}
	\lim_{m\to\infty} \theta_{\Pi_m}(x,\dots,x) = \widehat{\theta}(x,\dots,x) = (x,\dots,x), \quad \forall x\in \{A-2,A-1,A\}. 
\end{align*}
Writing $m_x(G) := \bE_{\theta\sim G}[e^{-\theta}\theta^x]$, the above display implies that
\begin{align}\label{eq:three-points}
	\lim_{m\to\infty} \frac{\bE_{\Pi_m}[m_x(G)^{n-1}m_{x+1}(G)]}{\bE_{\Pi_m}[m_x(G)^n]} = x, \quad \forall x\in \{A-2,A-1,A\}. 
\end{align}
Since $m_x(G)$ is a sequence of moments, the map $x\mapsto \frac{m_{x+1}(G)}{m_x(G)} = \theta_G(x)$ is non-decreasing; in addition, since $G\in \calP([0,A])$, it holds that $\frac{m_{x+1}(G)}{m_x(G)}\in [0,A]$ for all $x\in \naturals$. Let $\widetilde{\Pi}_m(\rmd G) \propto \Pi_m(\rmd G)m_{A-2}(G)^n$ be a tilt of $\Pi_m$, and $\mu_m$ be the pushforward measure of $\widetilde{\Pi}_m$ of the map
\begin{align}\label{eq:map}
	h: G\mapsto \pth{ \frac{m_{A-1}(G)}{m_{A-2}(G)}, \frac{m_A(G)}{m_{A-1}(G)}, \frac{m_{A+1}(G)}{m_A(G)} }. 
\end{align}
Then $\mu_m$ is a probability measure over $\Omega = \{(r,s,t)\in [0,A]^3: r\le s\le t\}$, and \eqref{eq:three-points} can be rewritten as
\begin{align*}
	\lim_{m\to\infty} \int r\rmd \mu_m = A-2, \quad \lim_{m\to\infty} \int r^n (A-1-s)\rmd \mu_m = 0, \quad \lim_{m\to\infty} \int r^n s^n (A-t)\rmd \mu_m = 0. 
\end{align*}
Since $\Omega$ is compact, a proper subsequence of $\{\mu_m\}$ admits a weak limit $\mu$, with 
\begin{align}\label{eq:mu-condition}
	\int r\rmd \mu = A-2, \quad \int r^n (A-1-s)\rmd \mu = 0, \quad \int r^n s^n (A-t)\rmd \mu = 0. 
\end{align}
To proceed, we need a technical lemma on the image set of $h$ in \Cref{eq:map}. 

\begin{lmm}\label{lmm:endpoints}
	Let $K$ be the closure of the image of $h$ in \Cref{eq:map}. If $(r,s,t)\in K$ with $r>0$ and $t=A$, then $s=A$. 
\end{lmm}

We postpone the proof of \prettyref{lmm:endpoints} to the end of this section. Since each $\mu_m$ is supported on $K$, so is the weak limit $\mu$. Since $A\ge 3$, the first identity of \Cref{eq:mu-condition} implies that $\mu(\{r>0\}) > 0$. Since $s\ge r$, on the set $\{r>0\}$ we must also have $s>0$. Finally, since $t\le A$, the last identity of \Cref{eq:mu-condition} implies that $t=A$ $\mu$-a.s. on the set $\{r>0\}$. By \prettyref{lmm:endpoints}, this also implies $s=A$ $\mu$-a.s. on the set $\{r>0\}$. Then the middle identity of \Cref{eq:mu-condition} implies
\begin{align*}
	0 = \int r^n(A-1-s)\rmd \mu = \int_{r>0} r^n(A-1-s)\rmd \mu = -\int_{r>0} r^n \rmd \mu, 
\end{align*}
which forces $\mu(\{r>0\}) = 0$, a contradiction! This contradiction shows that no admissible estimator can satisfy \eqref{eq:NPMLE_property} at all three points $x=A-2,A-1,A$. Since the NPMLE-based estimator does satisfy \eqref{eq:NPMLE_property} for all integers $x$, it is inadmissible. 

\subsubsection{Proof of \prettyref{lmm:endpoints}}
Suppose $G_m\in \calP([0,A])$ is a sequence of priors such that $h(G_m) = (r_m,s_m,t_m)\to (r,s,t)$ as $m\to\infty$, with $r>0$ and $t=A$. Since $Am_x(G) - m_{x+1}(G) = \bE_{\theta\sim G}[(A-\theta)e^{-\theta}\theta^x]$, for $G\in \calP([0,A])$, Cauchy--Schwarz gives
\begin{align*}
	(Am_A(G) - m_{A+1}(G))(Am_{A-2}(G) - m_{A-1}(G)) \ge (Am_{A-1}(G) - m_{A}(G))^2. 
\end{align*}
Applying the above inequality to $G=G_m$ yields
\begin{align*}
	r_ms_m(A-t_m)\cdot (A-r_m) \ge [r_m(A-s_m)]^2 \Longrightarrow r_m(A-s_m)^2 \le (A-r_m)s_m(A-t_m). 
\end{align*}
As $t_m\to t=A$ and $r_m\to r>0$, passing to the limit clearly gives $s_m\to A$, as claimed.

\subsection{Proof of Lemma \ref{lemma:simple_prior}}
Since $r\mapsto E(\delta,r)$ can be made a non-increasing function, we verify \prettyref{definition:PoP} for $\delta=n^{-2}$ and $r^2=n^{-1}$. Our first step is to first identify an intermediate prior $G$ with $\TV(f_{G_0},f_G)\le n^{-2}$. To this end, we show that for $L\asymp_A \frac{\log n}{\log\log n}$ and any $G_0\in \mathcal{P}([0, A])$, 
there exists a $G\in \mathcal{P}([0, A])$ such that: 
\begin{itemize}
	\item $G = \sum_{j = 1}^L w_j \delta_{\lambda_j}$ is supported on $L$ atoms $\lambda_1, \dots, \lambda_L$ such that $\min_{j\le L}\{\lambda_j\}\ge \frac{1}{4n^2}$ and $\min_{j\le L} \{w_j\}\ge \frac{1}{8n^2L}$. 
	
	\item $\TV(f_G, f_{G_0}) \le n^{-2}$. 
\end{itemize}
To begin, we use the following result from \cite[Lemma 3]{wu2016minimax}. 

\begin{lmm}\label{lemma:moment_matching}
	Let $V$ and $V'$ be random variables on $[0,A]$. If $\bE[V^j] = \bE[V'^j]$ for $j=0,1,\dots,L$ with $L>2eA$, then
	\begin{align*}
		\TV(\bE[\Poi(V)], \bE[\Poi(V')]) \le \pth{\frac{2eA}{L}}^L. 
	\end{align*}
\end{lmm}

To apply this lemma, choose $L \asymp_A \frac{\log n}{\log\log n}$ such that the RHS equals $\frac{1}{2n^2}$. Next, let $V\sim G_0$, and $V'\sim G_1$ be any random variable on $[0,A]$ that matches all first $L$ moments of $V$. By Carath\'eodory's theorem, $G_1$ can be chosen as a discrete distribution with support size at most $L$. By Lemma \ref{lemma:moment_matching}, we have $\TV(f_{G_0},f_{G_1})\le \frac{1}{2n^2}$. 

Next, we consider the following two identities, which we defer the proof to \Cref{sec:tv-chi-mixture-proof}. 
\begin{lmm}\label{lemma:tv-chi-mixture}
	Let $L > 0$ be an integer. Denote $G_1 = \sum_{j=1}^L w_j\delta_{\lambda_j}$ and $G_2 = \sum_{j = 1}^L w'_j\delta_{\lambda'_j}$. 
	Consider the Poisson mixtures $f_{G_1}$ and $f_{G_2}$. 
	Then the following holds true: 
	
	\begin{itemize}
		\item $\TV(f_{G_1}, f_{G_2})\le \sum_{j=1}^L |w_j - w_j'| + \max_{j=1}^L \TV(\Poi(\lambda_j), \Poi(\lambda_j'))$. 
		
		\item $\chi^2(f_{G_1} \| f_{G_2})\le (1 + \chi^2(w \| w'))(1+\max_{j=1}^L \chi^2(\Poi(\lambda_j) \| \Poi(\lambda_j'))) - 1$. 
	\end{itemize}
\end{lmm}

To apply this lemma, write $G_1 = \sum_{j=1}^L w_{j,1}\delta_{\lambda_{j,1}}$. Our goal is to construct 
$G := \sum_{j=1}^L w_j\delta_{\lambda_j}$ such that: 
\begin{itemize}
	\item $\min_{j=1}^L\{w_j\}\ge \frac{1}{8n^2L}$, $\sum_{j=1}^L w_j = 1$, $\min_{j=1}^L\{\lambda_j\}\ge \frac{1}{4n^2}$.  
	\item $\sum_{j=1}^L |w_j - w_{j,1}|\le \frac{1}{4n^2}, \max \TV(\Poi(\lambda_j), \Poi(\lambda_{j,1}))\le \frac{1}{4n^2}$. 
\end{itemize}
Then Lemma \ref{lemma:tv-chi-mixture} would ensure $\TV(f_{G_1}, f_{G})\le \frac{1}{2n^2}$ and thus $\TV(f_{G_0}, f_{G})\le \frac{1}{n^2}$ .

We first claim that $\TV(\Poi(\mu), \Poi(\nu))\le 1-e^{-|\mu -\nu|}\le |\mu-\nu|$ for $\mu, \nu > 0$. 
Indeed, let $\mu>\nu$, $X\sim \Poi(\nu), Z\sim\Poi(\mu - \nu)$, and $Y=X+Z$. 
By coupling, $\TV(\Poi(\mu), \Poi(\nu))\le \mathbb{P}[Z > 0] = 1-e^{-(\mu-\nu)}$, as claimed. 
In this sequel, the condition $ \TV(\Poi(\lambda_j), \Poi(\lambda_{j,1}))\le \frac{1}{4n^2}$ may be replaced with 
$|\lambda_j - \lambda_{j,1}|\le \frac{1}{4n^2}$, so that a sufficient construction is $\lambda_j = \lambda_{j,1}\vee \frac{1}{4n^2}$. To construct $w_j$, let $m = \argmax_j w_{j,1}$, and
\[
w_j = 
\begin{cases}
	w_{j,1}\vee \frac{1}{8n^2L} & \text{if }j\neq m,\\
	1 - \sum_{j\neq m} w_j & \text{if }j = m. 
\end{cases}
\]
By simple algebra, $\min_j w_j \ge \frac{1}{8n^2L}$ for large $n$ and $\sum_j |w_j-w_{j,1}|\le \frac{1}{4n^2}$. Thus $G=\sum_{j=1}^L w_j\delta_{\lambda_j}$ satisfies the above conditions, and we take it to be the intermediate prior in \prettyref{definition:PoP}. 

Next we prove the upper bound on $E$. Recall that $k=\lceil c_0\frac{\log n}{\log\log n}\rceil$ is the number of atoms in \Cref{alg:universal_prior}; choose $c_0>0$ large enough such that $k\ge L$. In this case, we write $G = \sum_{j=1}^k w_{j}\delta_{\lambda_{j}}$ by allowing repetitions in the atoms. Since
\begin{align*}
	\chi^2\bpth{ \sum_{j=1}^k w_j\Poi(\lambda_j) \| \sum_{j=1}^k w_j'\Poi(\lambda_j')} +1 &\le \pth{\chi^2(w\|w')+1}\pth{\max_{j\in [k]} \chi^2(\Poi(\lambda_j) \| \Poi(\lambda_j'))+1}
	\\&=\pth{\chi^2(w\|w')+1}\pth{\max_{j\in [k]} \exp\left(\frac{(\lambda_j - \lambda_j')^2}{\lambda_j'}\right)}
\end{align*}
a sufficient condition for the $\chi^2$-divergence to be at most $n^{-1}$ is $\chi^2(w\|w')\le \frac{1}{3n}$ and 
\begin{align*}
	\max_{j\in [k]} \frac{(\lambda_j - \lambda_j')^2}{\lambda_j'} \le \frac{1}{3n}. 
\end{align*} 
Since $w_j, \lambda_j=\Omega(n^{-3})$,
a further sufficient condition is to make $\|w-w'\|_\infty \le cn^{-3}$ and $\|\lambda-\lambda'\|_\infty \le cn^{-3}$. Since $w, w'$ are both on the simplex $\Delta^{k-1}$ and $\lambda, \lambda'\in [0, A]^k$, 
a standard volume argument establishes that the simple PoP in \Cref{alg:universal_prior} chooses such $(\lambda', w')$ pair with probability at least
\begin{align*}
	\Omega(\frac{1}{n^3})^{O(k)} = \exp(-O(k\log n)) = \exp\pth{-O\bpth{\frac{\log^2 n}{\log \log n}}}. 
\end{align*}
This verifies \Cref{eq:prior_mass} with the claimed $E(n^{-2},r)=O(\frac{\log^2 n}{\log\log n})$. 

\subsubsection{Proof of Lemma \ref{lemma:tv-chi-mixture}}\label{sec:tv-chi-mixture-proof}
For TV, we have 
\begin{flalign*}
	\TV(f_{G_1}, f_{G_2})
	&=\TV\bpth{\sum_{j=1}^L w_j\Poi(\lambda_j), \sum_{j=1}^L w'_j\Poi(\lambda'_j)}
	\\&\stepa{\le} \frac{1}{2}\sum_{x\ge 0}\sum_{j=1}^L |w_j - w_j'|\Poi(x;\lambda_j)
	+ w'_j|\Poi(x;\lambda_j) - \Poi(x;\lambda'_j)|
	\\&= \frac{1}{2}\sum_{j=1}^L |w_j - w_j'| + \sum_{j=1}^L w_j' \TV(\Poi(\lambda_j), \Poi(\lambda_j'))
	\\&\stepb{\le} \sum_{j=1}^L |w_j - w_j'| + \max_{j\le L}\{\TV(\Poi(\lambda_j), \Poi(\lambda_j'))\}
\end{flalign*}
where (a) is obtained by expanding the PMFs of the Poisson mixture at each $x$, and using the identity 
$|ab - a'b'|\le |b - b'|a + b'|a-a'|$, and (b) uses $w_j'\ge 0$ and $\sum_{j=1}^L w_j' = 1$. For $\chi^2$, we have
\begin{flalign*}
	1 + \chi^2(f_{G_1} \| f_{G_2})
	&= \sum_{x\ge 0}\frac{f_{G_1}(x)^2}{f_{G_2}(x)}
	\\&=\sum_{x\ge 0}\frac{(\sum_{j=1}^L w_j\Poi(x;\lambda_j))^2}{\sum_{j=1}^L w'_j\Poi(x;\lambda'_j)}
	\\&\stepc{\le}\sum_{x\ge 0}\sum_{j=1}^L\frac{(w_j\Poi(x;\lambda_j))^2}{w'_j\Poi(x;\lambda'_j)}
	\\&= \sum_{j=1}^L \frac{w^2_j}{w_j'}(1 + \chi^2(\Poi(\lambda_j) \| \Poi(\lambda_j')))
	\\&\le (1 + \chi^2(w \| w'))(1+\max_{j\le L} \chi^2(\Poi(\lambda_j) \| \Poi(\lambda_j'))), 
\end{flalign*}
where (c) follows from Cauchy--Schwarz. 

\subsection{Proof of \Cref{thm:length-generalization}}
The first statement directly follows from the following lemma. 

\begin{lmm}\label{lemma:length-generalization}
	For $i\in [n]$, the posterior mean $\bE_{\Pi}[\theta_i|X^n]$ can be expressed as $\bE_{\Pi}[\theta_i|X^n] = f_{\Pi,n}(X_i, \mu_n)$, where $\mu_n=\frac{1}{n}\sum_{i=1}^n \delta_{X_i}$ is the empirical distribution of $X^n$, and
	\begin{align*}
		f_{\Pi,n}(x,\mu)  = \frac{\int \Pi(\rmd G)\exp(n\int \log f_G(x') \mu(\rmd x'))\cdot \theta_G(x)}{\int \Pi(\rmd G)\exp(n\int \log f_G(x') \mu(\rmd x'))}. 
	\end{align*}
	In particular, for $\mu_m = \frac{1}{m}\sum_{i=1}^m \delta_{X_i}$ with a different length $m$, we have
	\begin{align*}
		f_{\Pi,n}(X_i, \mu_m) = \bE_{G\sim \Pi_{G|X^m}^\alpha }[\theta_G(X_i)] = \frac{\bE_{G\sim \Pi_{G|X_{\backslash i}}^\alpha }[\theta_G(X_i)f_G(X_i)^\alpha]}{\bE_{G\sim \Pi_{G|X_{\backslash i}}^\alpha }[f_G(X_i)^\alpha]}, 
	\end{align*}
	where $\alpha := \frac{n}{m}$, and $\Pi^\alpha$ denotes the $\alpha$-posterior:  \begin{align*}
		\Pi^\alpha(\rmd G|X^m) := \frac{\Pi(\rmd G) (\prod_{i=1}^{m} f_G(X_i))^{\alpha} }{\int \Pi(\rmd G') (\prod_{i=1}^{m} f_{G'}(X_i))^{\alpha}}. 
	\end{align*}
\end{lmm}

For the second statement, we need the following generalization of \prettyref{lemma:hellinger-to-regret-1}. 

\begin{lmm}\label{lemma:hellinger-to-regret-2}
	Let $G_0$ be a fixed prior supported on $[0,A]$, and $G$ be a random prior supported on $[0,A]$ almost surely. Fix any $\varepsilon\in (0,e^{-e})$, and $\alpha\in (0,1]$. Then for an absolute constant $C=C(A)>0$, 
	\begin{align*}
		\bE_{X\sim f_{G_0}}\bqth{ \bpth{ \frac{\bE[\theta_G(X)f_G(X)^\alpha]}{\bE[f_G(X)^\alpha]} - \theta_{G_0}(X)}^2 } \le C\bpth{\frac{\log (1/\varepsilon)}{\log\log (1/\varepsilon)} \bE\qth{H^2(f_G, f_{G_0})} + \varepsilon } . 
	\end{align*}
\end{lmm}

To prove \Cref{thm:length-generalization}, we first integrate the tails in \prettyref{lemma:posterior_contraction_alpha} to get
\begin{align*}
	\bE_{X_n}\bE_{G\sim \Pi_{G|X^{m-1}}^\alpha} \qth{H^2(f_G, f_{G_0})} &\le C\pth{\varepsilon_m^2 + r_{m,\alpha}^2 + \frac{1}{\alpha m}} + \int_{C(\varepsilon_m^2 + r_{m,\alpha}^2 + \frac{1}{\alpha m})}^2 \pth{m^{-1}+e^{-c\alpha mt}} \rmd t\\
	&\lesssim \varepsilon_m^2 + r_{m,\alpha}^2 + \frac{1}{\alpha m}. 
\end{align*}
Next we apply \prettyref{lemma:length-generalization} and \ref{lemma:hellinger-to-regret-2} to $m=\ntest$, the random prior $G\sim \Pi_{G|X^{\ntest-1}}^\alpha$ with $\alpha=\frac{n}{\ntest}$, an independent $X=X_n\sim f_{G_0}$, and a small parameter $\varepsilon = n_{\mathsf{test}}^{-1}$: 
\begin{align*}
	\sup_{G_0\in \calP([0,A])} \reg(f_{\Pi,n}^{\ntest}; G_0) \le \frac{C\log\ntest}{\log\log\ntest}\pth{\varepsilon_{\ntest}^2 + r_{\ntest,\alpha}^2 + \frac{1}{n}}. 
\end{align*}
Finally, \prettyref{lemma:metric_entropy} gives $\varepsilon_{\ntest}^2 = O(\frac{\log \ntest}{\ntest \log\log \ntest})$, and the inequality $E(n_{\mathrm{test}}^{-2}, r)\vee 1\le nr^2$ implies that $r^2 = \Omega(r_{\ntest,\alpha}^2 + n^{-1})$. This proves \Cref{thm:length-generalization}. 

\subsubsection{Proof of \prettyref{lemma:length-generalization}}
For the first statement, note that \prettyref{lemma:posterior_mean_training} yields $\widehat{\theta}_i(X^n) = \bE_{G\sim \Pi_{G|X^n}}[\theta_G(X_i)]$, with posterior
\begin{align*}
	\Pi(\rmd G|X^n) = \frac{\Pi(\rmd G)\prod_{j=1}^n f_G(X_j)}{\int \Pi(\rmd G')\prod_{j=1}^n f_{G'}(X_j)} = \frac{\Pi(\rmd G)\exp(n\int \log f_G(x) \mu_n(\rmd x))}{\int \Pi(\rmd G')\exp(n\int \log f_{G'}(x) \mu_n(\rmd x))}. 
\end{align*}
Therefore, we conclude that
\begin{align*}
	\widehat{\theta}_i(X^n) = \int \Pi(\rmd G|X^n) \theta_G(X_i) = \frac{\int \Pi(\rmd G)\exp(n\int \log f_G(x) \mu_n(\rmd x))\cdot \theta_G(X_i)}{\int \Pi(\rmd G)\exp(n\int \log f_G(x) \mu_n(\rmd x))} =: f_{\Pi,n}(X_i, \mu_n). 
\end{align*}
This completes the first claim. For the second claim, note that on a new sequence $X^m$,
\begin{align*}
	f_{\Pi,n}(X_i, \mu_m) &= \frac{\int \Pi(\rmd G)\exp(n\int \log f_G(x) \mu_m(\rmd x))\cdot \theta_G(X_i)}{\int \Pi(\rmd G)\exp(n\int \log f_G(x) \mu_m(\rmd x))} \\
	&= \frac{\int \Pi(\rmd G) (\prod_{j=1}^m f_G(X_j))^\alpha \cdot \theta_G(X_i)}{\int \Pi(\rmd G)(\prod_{j=1}^m f_G(X_j))^\alpha} =: \bE_{G\sim \Pi_{G|X^m}^\alpha }[\theta_G(X_i)], 
\end{align*}
where $\alpha=\frac{n}{m}$, and we use the notation $\Pi^\alpha(\rmd G|X^m)$ to denote the $\alpha$-posterior: 
\begin{align*}
	\Pi^\alpha(\rmd G|X^m) := \frac{\Pi(\rmd G) (\prod_{i=1}^{m} f_G(X_i))^{\alpha} }{\int \Pi(\rmd G') (\prod_{i=1}^{m} f_{G'}(X_i))^{\alpha}}. 
\end{align*}
Similar to the proof of \prettyref{lemma:posterior_mean_training}, we also have
\begin{align*}
	f_{\Pi,n}(X_i, \mu_m) = \bE_{G\sim \Pi_{G|X^m}^\alpha }[\theta_G(X_i)] = \frac{\bE_{G\sim \Pi_{G|X_{\backslash i}}^\alpha }[\theta_G(X_i)f_G(X_i)^\alpha]}{\bE_{G\sim \Pi_{G|X_{\backslash i}}^\alpha }[f_G(X_i)^\alpha]}. 
\end{align*}

\subsubsection{Proof of Lemma \ref{lemma:hellinger-to-regret-2}}
Let $\calE := \sth{x\in \naturals: \bE[f_G(x)^\alpha]\ge \frac{1}{2}f_{G_0}(x)^{\alpha}}$. By Markov's inequality, for $x\in \calE^c$, 
\begin{align*}
	\bP\pth{f_G(x)\le \frac{3}{4}f_{G_0}(x)} \ge 1 - \frac{\bE[f_G(x)^\alpha]}{\bpth{\frac{3}{4}f_{G_0}(x)}^{\alpha}} \ge 1 - \frac{1}{2}\pth{\frac{4}{3}}^{\alpha} \ge \frac{1}{3}. 
\end{align*}
Therefore, 
\begin{align*}
	\bE\qth{H^2(f_{G_0},f_G)} &= \sum_{x\in \naturals} \bE\pth{\sqrt{f_{G_0}(x)} - \sqrt{f_{G}(x)}}^2 \ge \sum_{x\in \calE^c} \bE\pth{\sqrt{f_{G_0}(x)} - \sqrt{f_{G}(x)}}^2 \\
	&\ge \sum_{x\in \calE^c} \frac{1}{3} \bpth{\sqrt{f_{G_0}(x)} - \sqrt{\frac{3}{4} f_{G_0}(x)}}^2 \ge c_0 \sum_{x\in \calE^c} f_{G_0}(x) = c_0 f_{G_0}(\calE^c),
\end{align*} 
for some universal constant $c_0>0$. Consequently, we have arrived at
\begin{align*}
	f_{G_0}(\calE^c) \le C_0 \cdot \bE\qth{H^2(f_{G_0},f_G)}. 
\end{align*}
Therefore, as $\theta_G(X), \theta_{G_0}(X)\in [0,A]$ almost surely, 
\begin{align}\label{eq:Ec}
	\bE_{X\sim f_{G_0}}\bqth{ \bpth{ \frac{\bE[\theta_G(X)f_G(X)^\alpha]}{\bE[f_G(X)^\alpha]} - \theta_{G_0}(X)}^2 \indc{X\in \calE^c} } \le A^2 f_{G_0}(\calE^c) \le C_1 \cdot \bE\qth{H^2(f_{G_0},f_G)}. 
\end{align}

Next, for $X\in \calE$, we have
\begin{align*}
	&\bE_{X\sim f_{G_0}}\bqth{ \bpth{ \frac{\bE[\theta_G(X)f_G(X)^\alpha]}{\bE[f_G(X)^\alpha]} - \theta_{G_0}(X)}^2 \indc{X\in \calE} } \\
	&= \bE_{X\sim f_{G_0}}\bqth{ \bpth{ \frac{\bE[(\theta_G(X)-\theta_{G_0}(X))f_G(X)^\alpha]}{\bE[f_G(X)^\alpha]}}^2 \indc{X\in \calE} } \\
	&\stepa{\le} \bE_{X\sim f_{G_0}}\bqth{ \frac{\bE[(\theta_G(X)-\theta_{G_0}(X))^2 f_G(X)^\alpha]}{\bE[f_G(X)^\alpha]} \indc{X\in \calE} } \\
	&\stepb{\le} 2\cdot \bE_{X\sim f_{G_0}}\bqth{ \frac{\bE[(\theta_G(X)-\theta_{G_0}(X))^2 f_G(X)^\alpha]}{f_{G_0}(X)^{\alpha}} } \\
	&= 2\cdot \bE\bqth{\sum_{x\in \naturals} (\theta_G(x)-\theta_{G_0}(x))^2 f_G(x)^\alpha f_{G_0}(x)^{1-\alpha} } \\
	&\stepc{\le} 2\alpha \cdot \bE\bqth{ \bE_{X\sim f_G} \bpth{ (\theta_G(X)-\theta_{G_0}(X))^2 } } + 2(1-\alpha)\cdot \bE\bqth{ \bE_{X\sim f_{G_0}} \bpth{ (\theta_G(X)-\theta_{G_0}(X))^2 } }  \\
	&\stepd{\le} C_2\bpth{\frac{\log (1/\varepsilon)}{\log\log (1/\varepsilon)} \bE\qth{H^2(f_G, f_{G_0})} + \varepsilon }. 
\end{align*}
Here (a) is due to Cauchy--Schwarz, (b) uses the definition of $\calE$, (c) follows from Young's inequality $x^\alpha y^{1-\alpha}\le \alpha x+(1-\alpha)y$ for $\alpha\in (0,1]$ and $x,y\ge 0$, and (d) applies \prettyref{lemma:hellinger-to-regret-1} twice to $(G, G_0)$ and $(G_0, G)$. Combining this inequality with \Cref{eq:Ec} completes the proof of \prettyref{lemma:hellinger-to-regret-2}. 

\section{Deferred Proofs in \Cref{sec:posterior_contraction}}

\subsection{Proof of Lemma \ref{lemma:hellinger-to-regret-1}}
Fix $\varepsilon\in (0,e^{-e})$. The following inequality was shown in \cite[Lemma 4]{jana2025optimal}: for every integer $K\ge 0$, 
\begin{align*}
	\bE_{X\sim f_{G_0}}\qth{ \pth{\theta_G(X) - \theta_{G_0}(X)}^2 } \le C\pth{ KH^2(f_G,f_{G_0}) + \varepsilon_K(G_0) }, 
\end{align*}
with $\varepsilon_K(G_0) := \sum_{y\ge K}f_{G_0}(y)$. Since $\mathrm{supp}(G_0)\subseteq [0,A]$, standard Poisson tails yield $\varepsilon_K(G_0) \le \varepsilon$ for $K=O_A(\frac{\log(1/\varepsilon)}{\log\log(1/\varepsilon)})$. Plugging this choice of $K$ gives the claimed bound. 

\subsection{Proof of Lemma \ref{lemma:metric_entropy}}
Fix $\varepsilon\in (0,e^{-e})$, any $\rho\ge \varepsilon$, and $G_0\in \calP([0,A])$. Since the Hellinger distance is a metric, it suffices to construct an \emph{improper} covering of $\sth{f_G: H(f_G,f_{G_0})\le 2\rho}$ using Hellinger balls of radius $\rho/2$. Let $L = \lceil C\frac{\log(1/\varepsilon)}{\log\log(1/\varepsilon)} \rceil$; by Poisson tail bounds, there is a large absolute constant $C = C(A)>0$ such that
\begin{align*}
	\sup_{G\in \calP([0,A])} \sum_{x>L} f_G(x) \le \frac{\varepsilon^2}{16} \le \frac{\rho^2}{16}. 
\end{align*}
Let $H_{\le L}(P,Q) := (\sum_{i=0}^L (\sqrt{P(i)}-\sqrt{Q(i)})^2)^{1/2}$ be the truncated Hellinger distance, we have
\begin{align*}
	H^2(f_{G_1},f_{G_2}) \le H_{\le L}^2(f_{G_1},f_{G_2}) + f_{G_1}(X>L) + f_{G_2}(X>L). 
\end{align*}
Therefore, for $G_1,G_2\in \calP([0,A])$, the inequality $H_{\le L}(f_{G_1},f_{G_2})\le \frac{\rho}{4}$ implies $H(f_{G_1},f_{G_2})\le \frac{\rho}{2}$. 
As a result, by truncating all Poisson mixtures onto a smaller support $\{0,1,\dots,L\}$, it suffices to restrict to the space $\calS$ of discrete sub-distributions 
\begin{align*}
	\calS = \bsth{ P= (p_0,\dots,p_L)\in \bR_+^{L+1}: \sum_{i=0}^L p_i\le 1 }, 
\end{align*}
and cover the set $\{P\in \calS: H(P,P_0)\le 2\rho\}$ using Hellinger balls of radius $\rho/4$. This covering problem is easily solved via a volume argument: parametrizing $q_i=\sqrt{p_i}$, the set $\calS$ becomes the intersection of the unit $\ell_2$ ball and the non-negative orthant for $Q=(q_0,\dots,q_L)\in \bR^{L+1}$, and the Hellinger distance becomes the $\ell_2$ metric: $H(P_1,P_2)=\|Q_1-Q_2\|_2$. Therefore, the covering problem is equivalent to covering an $\ell_2$ ball of radius $2\rho$ by smaller $\ell_2$ balls of radius $\rho/4$ in $\bR^{L+1}$; a simple volume argument shows that the log covering number is $\log N \le C(L+1)$. Therefore,
\begin{align*}
	\log \Nloc(\varepsilon, \calP, H) = O(L) = O\bpth{\frac{\log(1/\varepsilon)}{\log\log(1/\varepsilon)}}. 
\end{align*}

\subsection{Proof of Lemma \ref{lemma:posterior_mean_training}}
By repeated applications of Bayes' rule, the conditional density of $\theta_i$ given $X^n$ is
\begin{align*}
	\frac{\bE_{G\sim \Pi}\bqth{G(\theta_i)\Poi(X_i;\theta_i)\prod_{j\neq i} f_G(X_j)} }{\bE_{G\sim \Pi}\bqth{\prod_{j=1}^n f_G(X_j)}} = \bE_{G\sim \Pi_{G|X^n}}\bqth{ \frac{G(\theta_i)\Poi(X_i;\theta_i)}{f_G(X_i)} } = \bE_{G\sim \Pi_{G|X^n}}\bqth{ \bP_G\pth{\theta_i | X_i} }. 
\end{align*}
Here $\bP_G$ denotes the joint distribution $\theta_i\sim G$ and $X_i|\theta_i\sim \Poi(\theta_i)$. Therefore,
\begin{align*}
	\bE_\Pi[\theta_i | X^n] = \bE_{G\sim \Pi_{G|X^n}}\bqth{ \int \theta_i \bP_G\pth{\theta_i | X_i} \rmd \theta_i} = \bE_{G\sim \Pi_{G|X^n}}\bqth{ \theta_G(X_i) }, 
\end{align*}
by definition of $\theta_G$ in \Cref{eq:theta_G}. This establishes the first identity. For the second identity, we further use the Bayes' rule to write
\begin{align*}
	\bE_{G\sim \Pi_{G|X^n}}\bqth{ \theta_G(X_i) } = \frac{\bE_{G\sim \Pi_{G|X_{\backslash i}}}[\theta_G(X_i)f_G(X_i)]}{\bE_{G\sim \Pi_{G|X_{\backslash i}}}[f_G(X_i)]} \overset{\Cref{eq:theta_G}}{=} (X_i+1)\cdot \frac{\bE_{G\sim \Pi_{G|X_{\backslash i}}}[f_G(X_i+1)]}{\bE_{G\sim \Pi_{G|X_{\backslash i}}}[f_G(X_i)]}. 
\end{align*}
Since $G_i = \bE_{G\sim \Pi_{G|X_{\backslash i}}}[G]$, we have $f_{G_i}(x) = \bE_{G\sim \Pi_{G|X_{\backslash i}}}[f_G(x)]$ for every $x\in \naturals$. Therefore, we can continue the above expression and obtain
\begin{align*}
	\bE_{G\sim \Pi_{G|X^n}}\bqth{ \theta_G(X_i) } = (X_i+1)\frac{f_{G_i}(X_i+1)}{f_{G_i}(X_i)} \overset{\Cref{eq:theta_G}}{=} \theta_{G_i}(X_i). 
\end{align*} 

\subsection{Proof of \prettyref{lemma:posterior_contraction}}

In this section we prove \prettyref{lemma:posterior_contraction}. We essentially apply the same classical posterior contraction arguments in \cite{ghosal2000convergence}, with a few adaptations to obtain a high-probability statement. Let $G$ be given in \prettyref{definition:PoP} such that $\TV(f_G,f_{G_0})\le n^{-2}$. We claim that it suffices to consider the case $G_0 = G$. In fact, once we establish \prettyref{lemma:posterior_contraction} with $G$ in place of $G_0$, we can move to $G_0$ by noting that $\TV(f_{G_0}^{\otimes (n-1)}, f_G^{\otimes (n-1)})\le (n-1)\TV(f_{G_0}, f_G) \le n^{-1}$ and
\begin{align*}
	H^2(f_{G_0}, \Pi_{X_n|X^{n-1}}) &\lesssim H^2(f_{G}, \Pi_{X_n|X^{n-1}}) + H^2(f_G, f_{G_0}) \\
	&\le H^2(f_{G}, \Pi_{X_n|X^{n-1}}) + 2\TV(f_G, f_{G_0}) \\
	&\le H^2(f_{G}, \Pi_{X_n|X^{n-1}}) + O\pth{n^{-2}}. 
\end{align*}
Hence, this step only replaces $H^2$ by $O(H^2+n^{-2})$ and amplifies the error probability by an additive factor of $n^{-1}$; in the sequel we assume that $G = G_0$ and establish the exponential error probability. 

We apply the arguments in \cite[Theorem 8.1]{ghosal2000convergence}. Let $\Nloc(\varepsilon):=\Nloc(\varepsilon, \calP, H)$ be the local covering number of $\calP$, and $\varepsilon_n>0$ satisfy $\log\Nloc(\varepsilon_n)\le n\varepsilon_n^2$. By \cite[Theorem 7.1]{ghosal2000convergence} (or the classical results in \cite{birge1983approximation,le1986asymptotic}), for $\varepsilon\ge C\varepsilon_n$ with a large absolute constant $C\ge 2$, there exists a test $\phi_n = \phi_n(X^{n-1})\in \{0,1\}$ such that
\begin{align}
	\bE_{f_{G_0}^{\otimes (n-1)}}[\phi_n] &\le \Nloc(\varepsilon_n) \exp\pth{-cn\varepsilon^2} ,\label{eq:type_I_err} \\
	\sup_{f_G\in \calP: H(f_G,f_{G_0})\ge \varepsilon}\bE_{f_{G}^{\otimes (n-1)}}[1-\phi_n] &\le \exp\pth{-cn\varepsilon^2}, \label{eq:type_II_err}
\end{align}
where $c>0$ is a universal constant. By \eqref{eq:type_I_err}, 
\begin{align}\label{eq:component_I}
	f_{G_0}^{\otimes (n-1)}\pth{\sth{\phi_n = 1}} \le \Nloc(\varepsilon_n) \exp\pth{-cn\varepsilon^2}\le \exp\pth{-\frac{c}{2}n\varepsilon^2}, 
\end{align}
by $\log\Nloc(\varepsilon_n)\le n\varepsilon_n^2$ and $\varepsilon\ge C\varepsilon_n$ for large enough $C>0$. On the other hand, let $\Pi_{G|X^{n-1}}$ be the posterior distribution of $G$ given $X^{n-1}$, then
\begin{align*}
	\Pi_{G|X^{n-1}}\pth{H^2(f_{G_0}, f_G) > \varepsilon^2 } 
	&= \frac{\int_{G: H(f_{G_0}, f_G) > \varepsilon } \Pi(\rmd G)\prod_{i=1}^{n-1} f_G(X_i) }{\int \Pi(\rmd G) \prod_{i=1}^{n-1} f_G(X_i)} \\
	&= \frac{\int_{G: H(f_{G_0}, f_G) > \varepsilon } \Pi(\rmd G)\prod_{i=1}^{n-1} \frac{f_G(X_i)}{f_{G_0}(X_i)} }{\int \Pi(\rmd G) \prod_{i=1}^{n-1} \frac{f_G(X_i)}{f_{G_0}(X_i)}}
\end{align*}
by the Bayes' rule. We bound the numerator and denominator separately. 
\begin{enumerate}
	\item Numerator: Adding the event $\phi_n = 0$ and taking expectation with respect to $X^{n-1}\sim f_{G_0}^{\otimes (n-1)}$ yields
	\begin{align*}
		&\bE_{f_{G_0}^{\otimes (n-1)}}\qth{(1-\phi_n)\int_{G: H(f_{G_0}, f_G) > \varepsilon } \Pi(\rmd G)\prod_{i=1}^{n-1} \frac{f_G(X_i)}{f_{G_0}(X_i)}}  \\
		&= \int_{G: H(f_{G_0}, f_G) > \varepsilon } \Pi(\rmd G) \cdot \bE_{f_G^{\otimes (n-1)}}[1-\phi_n] \overset{\prettyref{eq:type_II_err}}{\le} e^{-cn\varepsilon^2}. 
	\end{align*}
	Therefore, by Markov's inequality, 
	\begin{align}\label{eq:numerator_upper_bound}
		f_{G_0}^{\otimes (n-1)}\pth{ \sth{\int_{G: H(f_{G_0}, f_G) > \varepsilon } \Pi(\rmd G)\prod_{i=1}^{n-1} \frac{f_G(X_i)}{f_{G_0}(X_i)} > e^{-\frac{c}{2}n\varepsilon^2}} \cap \sth{\phi_n=0} } \le \exp\pth{-\frac{c}{2}n\varepsilon^2}. 
	\end{align}
	\item Denominator: Let $r_n>0$ satisfy $E(n^{-2},r_n) \le nr_n^2$, and $U = \sth{G': \chi^2(f_{G_0}\| f_{G'}) \le r_n^2}$ be a neighborhood of $G_0$. Since we have assumed $G=G_0$ without loss of generality, assumption \eqref{eq:prior_mass} ensures $\Pi(U)\ge e^{-E_n}$, with $E_n:=E(n^{-2},r_n)$. By Jensen's inequality,
	\begin{align*}
		\int \Pi(\rmd G) \prod_{i=1}^{n-1} \frac{f_G(X_i)}{f_{G_0}(X_i)} \ge \int_U \Pi(\rmd G) \prod_{i=1}^{n-1} \frac{f_G(X_i)}{f_{G_0}(X_i)} \ge \Pi(U)\exp\pth{ \int \Pi_{|U}(\rmd G) \sum_{i=1}^{n-1} \log \frac{f_G(X_i)}{f_{G_0}(X_i)} }, 
	\end{align*}
	where $\Pi_{|U}$ denotes the restriction of $\Pi$ to $U$. By \prettyref{lemma:chi-squared} below (which itself is a refinement of \cite[Lemma 8.3]{ghosal2000convergence}) with $P=f_G, Q=f_{G_0}$, the random variable
	\begin{align*}
		Y_i = \int \Pi_{|U}(\rmd G) \log \frac{f_G(X_i)}{f_{G_0}(X_i)}
	\end{align*}
	satisfies $\bE[e^{|Y_i|} - 1 - |Y_i|]\le \int \Pi_{|U}(\rmd G)\chi^2(f_{G_0}\|f_G) \le r_n^2$ by convexity of $x\mapsto e^{|x|}-1-|x|$ and definition of $U$. Therefore, by the discussion below Bernstein's inequality in \cite[Lemma 2.2.11]{van1996weak}, we have
	\begin{align*}
		f_{G_0}^{\otimes (n-1)} \pth{\sum_{i=1}^{n-1} (Y_i-\bE[Y_i]) \ge -\frac{c}{8}n\varepsilon^2} \ge 1 - \exp\pth{-\frac{c'(cn\varepsilon^2/8)^2}{nr_n^2 + cn\varepsilon^2/8}} \ge 1-\exp\pth{-c''n\varepsilon^2}, 
	\end{align*}
	for universal constants $c,c'>0$, under an additional assumption that $\varepsilon\ge Cr_n$. Since
	\begin{align*}
		\bE[Y_i] = -\int \Pi_{|U}(\rmd G)\KL(f_{G_0}\|f_G) \ge -\int \Pi_{|U}(\rmd G)\chi^2(f_{G_0}\|f_G) \ge -r_n^2, 
	\end{align*}
	we conclude that
	\begin{align*}
		f_{G_0}^{\otimes (n-1)} \pth{\int \Pi(\rmd G) \prod_{i=1}^{n-1} \frac{f_G(X_i)}{f_{G_0}(X_i)} \ge \exp\pth{-E_n-\frac{c}{8}n\varepsilon^2-nr_n^2}} \ge 1-\exp\pth{-c''n\varepsilon^2}. 
	\end{align*}
	Finally, since $E_n\le nr_n^2$, as long as $\varepsilon\ge Cr_n$ with a large enough constant $C>0$, this implies
	\begin{align}\label{eq:denom_lower_bound}
		f_{G_0}^{\otimes (n-1)} \pth{\int \Pi(\rmd G) \prod_{i=1}^{n-1} \frac{f_G(X_i)}{f_{G_0}(X_i)} \ge \exp\pth{-\frac{c}{4}n\varepsilon^2}} \ge 1-\exp\pth{-c''n\varepsilon^2}. 
	\end{align}
\end{enumerate}
By \eqref{eq:component_I}, \eqref{eq:numerator_upper_bound}, and \eqref{eq:denom_lower_bound}, we conclude that for every $\varepsilon^2\ge C(\varepsilon_n^2+r_n^2+\frac{1}{n})$, with probability at least $1-e^{-c_1n\varepsilon^2}$ over the randomness of $X^{n-1}\sim f_{G_0}^{\otimes (n-1)}$, it holds that
\begin{align*}
	\Pi_{G|X^{n-1}}\pth{H^2(f_{G_0}, f_G) > \varepsilon^2} \le e^{-c_2n\varepsilon^2}, \quad \text{with }c_2 = \frac{c}{4}. 
\end{align*}
This argument applies with $\varepsilon^2$ replaced by $j\varepsilon^2$ for every integer $j\ge 1$. Taking a union bound over $1\le j\le \lceil 2/\varepsilon^2\rceil$, we obtain that, with probability at least $
1-\sum_{j=1}^{\lceil 2/\varepsilon^2\rceil}e^{-c_1 n j\varepsilon^2}$,
the following bound holds simultaneously for all such $j$:
\[
\Pi_{G|X^{n-1}}\pth{H^2(f_{G_0},f_G)>j\varepsilon^2}
\le e^{-c_2 n j\varepsilon^2}.
\]
On this event, since $H^2\le 2$, the layer-cake formula gives
\begin{align*}
	\bE_{G\sim\Pi_{G|X^{n-1}}}\qth{H^2(f_{G_0},f_G)}
	&\le \varepsilon^2\sum_{j=0}^{\lceil 2/\varepsilon^2\rceil}
	\Pi_{G|X^{n-1}}\pth{H^2(f_{G_0},f_G)>j\varepsilon^2} \\
	&\le \varepsilon^2\bpth{1+\sum_{j\ge 1}e^{-c_2 n j\varepsilon^2}}
	\le C\left(\varepsilon^2+\frac1{n}\right)
	\le C'\varepsilon^2,
\end{align*}
where the last inequality uses $\varepsilon^2\ge \frac{C}{n}$. The exceptional probability is bounded by
\[
\sum_{j\ge 1}e^{-c_1 n j\varepsilon^2}
\le C e^{-c_1 n\varepsilon^2}
\le e^{-c_1' n\varepsilon^2},
\]
after adjusting constants. Finally, as $\Pi_{X_n|X^{n-1}} = \bE_{\Pi_{G|X^{n-1}}}[f_G]$, on the above event, the convexity of the squared Hellinger distance gives
\begin{align*}
	H^2(f_{G_0}, \Pi_{X_n|X^{n-1}}) \le \bE_{\Pi_{G|X^{n-1}}}\qth{H^2(f_{G_0}, f_G)} \le C'\varepsilon^2. 
\end{align*}
This is the claimed result. 

\begin{lmm}\label{lemma:chi-squared}
	For probability distributions $P$ and $Q$,
	\begin{align*}
		\bE_Q\bqth{ \exp\bpth{\Big|\log \frac{P}{Q}\Big|} - 1 - \Big|\log \frac{P}{Q}\Big| } \le \chi^2(Q\|P). 
	\end{align*}
\end{lmm}
\begin{proof}
	Using $\log x\ge 1-\frac{1}{x}$, when $P\ge Q$, this integral is
	\begin{align*}
		\int_{P\ge Q} \bpth{ P - Q - Q\log \frac{P}{Q}} \le \int_{P\ge Q} \bpth{P-Q-Q\bpth{1-\frac{Q}{P}}} = \int_{P\ge Q} \frac{(Q-P)^2}{P}. 
	\end{align*}
	When $P<Q$, this integral becomes
	\begin{align*}
		\int_{P<Q} \bpth{\frac{Q^2}{P} - Q - Q\log\frac{Q}{P}} \le \int_{P<Q} \bpth{\frac{Q^2}{P} - Q - Q\bpth{1-\frac{P}{Q}}} = \int_{P<Q} \frac{(Q-P)^2}{P}. 
	\end{align*}
	Summing up gives the target upper bound $\chi^2(Q\|P)$.
\end{proof}

\subsection{Proof of \prettyref{lemma:posterior_contraction_alpha}}
The proof precisely mimics the arguments of \prettyref{lemma:posterior_contraction} and uses the same test construction $\phi_n$ satisfying \Cref{eq:type_I_err} and \Cref{eq:type_II_err}. The main difference starts from
\begin{align*}
	\Pi_{G|X^{n-1}}^\alpha \pth{H^2(f_{G_0}, f_G) > \varepsilon^2 } = \frac{\int_{G: H(f_{G_0}, f_G) > \varepsilon } \Pi(\rmd G)\bpth{\prod_{i=1}^{n-1} \frac{f_G(X_i)}{f_{G_0}(X_i)}}^\alpha }{\int \Pi(\rmd G) \bpth{\prod_{i=1}^{n-1} \frac{f_G(X_i)}{f_{G_0}(X_i)}}^\alpha }. 
\end{align*}
Again, we bound the numerator and denominator separately. 
\begin{enumerate}
	\item Numerator: Since $0\le \alpha\le 1$, by H\"older's inequality,
	\begin{align*}
		&\bE_{f_{G_0}^{\otimes (n-1)}}\qth{(1-\phi_n)\int_{G: H(f_{G_0}, f_G) > \varepsilon } \Pi(\rmd G)\bpth{\prod_{i=1}^{n-1} \frac{f_G(X_i)}{f_{G_0}(X_i)}}^\alpha}  \\
		&= \int_{G: H(f_{G_0}, f_G) > \varepsilon } \Pi(\rmd G) \cdot \int_{\phi_n=0} f_G^{\otimes (n-1)}(\rmd x^{n-1})^{\alpha}  f_{G_0}^{\otimes (n-1)}(\rmd x^{n-1})^{1-\alpha}  \\
		&\le \int_{G: H(f_{G_0}, f_G) > \varepsilon } \Pi(\rmd G) \cdot  f_G^{\otimes (n-1)}(\sth{\phi_n=0})^{\alpha}  f_{G_0}^{\otimes (n-1)}(\sth{\phi_n=0})^{1-\alpha} \overset{\prettyref{eq:type_II_err}}{\le} e^{-c\alpha n\varepsilon^2}. 
	\end{align*}
	Therefore, by Markov's inequality, 
	\begin{align}\label{eq:numerator_upper_bound_alpha}
		& f_{G_0}^{\otimes (n-1)}\pth{ \sth{\int_{G: H(f_{G_0}, f_G) > \varepsilon } \Pi(\rmd G)\bpth{\prod_{i=1}^{n-1} \frac{f_G(X_i)}{f_{G_0}(X_i)}}^{\alpha} > e^{-\frac{c}{2}\alpha n\varepsilon^2}} \cap \sth{\phi_n=0} } \nonumber\\
		&\le \exp\pth{-\frac{c}{2}\alpha n\varepsilon^2}. 
	\end{align}
	\item Denominator: Again, let $U = \sth{G': \chi^2(f_{G_0}\| f_{G'}) \le r_{n,\alpha}^2}$ be a neighborhood of $G_0$, and $\Pi_{|U}$ be the restriction of $\Pi$ to $U$. By Jensen's inequality,
	\begin{align*}
		\int \Pi(\rmd G) \bpth{\prod_{i=1}^{n-1} \frac{f_G(X_i)}{f_{G_0}(X_i)}}^\alpha &\ge \int_U \Pi(\rmd G) \bpth{\prod_{i=1}^{n-1} \frac{f_G(X_i)}{f_{G_0}(X_i)}}^\alpha \\
		&\ge \Pi(U)\exp\pth{ \alpha \int \Pi_{|U}(\rmd G) \sum_{i=1}^{n-1}  \log \frac{f_G(X_i)}{f_{G_0}(X_i)} }. 
	\end{align*}
	Using $\Pi(U)\ge e^{-E(n^{-2},r_{n,\alpha})}\ge e^{-\alpha nr_{n,\alpha}^2}$ and the same Bernstein concentration for the exponent, as long as $\varepsilon\ge Cr_{n,\alpha}$, we have
	\begin{align}\label{eq:denom_lower_bound_alpha}
		f_{G_0}^{\otimes (n-1)} \pth{\int \Pi(\rmd G) \bpth{\prod_{i=1}^{n-1} \frac{f_G(X_i)}{f_{G_0}(X_i)}}^\alpha \ge \exp\pth{-\frac{c}{4}\alpha n\varepsilon^2}} \ge 1-\exp\pth{-c''n\varepsilon^2}. 
	\end{align}
\end{enumerate}
By \Cref{eq:type_I_err}, \Cref{eq:numerator_upper_bound_alpha}, and \Cref{eq:denom_lower_bound_alpha}, the same arguments in the proof of \prettyref{lemma:posterior_contraction} lead to
\begin{align*}
	& f_{G_0}^{\otimes (n-1)}\pth{\sth{\Pi_{G|X^{n-1}}^\alpha\pth{H^2(f_{G_0}, f_G) > \varepsilon^2 } > e^{-c\alpha n\varepsilon^2/4}}} \\
	&\le f_{G_0}^{\otimes (n-1)}\pth{\sth{\Pi_{G|X^{n-1}}^\alpha\pth{H^2(f_{G_0}, f_G) > \varepsilon^2 } > e^{-c\alpha n\varepsilon^2/4}} \cap \sth{\phi_n=0} } + f_{G_0}^{\otimes (n-1)}\sth{\phi_n=1} \\
	&\le e^{-c_1\alpha n\varepsilon^2},
\end{align*}
as long as $\varepsilon^2 \ge C(\varepsilon_n^2+r_{n,\alpha}^2+\frac{1}{\alpha n})$. The same layer-cake argument in the proof of \prettyref{lemma:posterior_contraction} gives the target upper bound of $\bE_{G\sim \Pi_{G|X^{n-1}}^\alpha}[H^2(f_{G_0},f_G)]$. 

\section{Deferred proofs in \Cref{sec:discussion}}\label{app:deferred-proofs-disc}

\subsection{Proof of Lemma \ref{lemma:finite_M}}
We first derive the expression of $\widehat{\theta}^n$ in \Cref{eq:ERM-PI}. For each $X^n \in \naturals^n$, let $\mathcal{S}(X^n) = \{m: \exists \pi\in S^n, \pi\circ X^{n, (m)} = X^n\}$ be the collection of data batches $X^{n,(m)}$ which are permutations of $X^n$. In addition, for all $m\in \mathcal{S}(X^n)$, we use $\pi_m$ to record the $m$-th permutation, i.e., $\pi_m\circ X^{n, (m)}=X^n$. Note that $\calS(X^n)$ only depends on the \emph{type} (or the sorted version) of $X^n$, denoted by $T(X^n)$; in addition, for a permutation-equivariant estimator $\widehat{\theta}^n$, we only need to specify its output for each input type. Therefore, the ERM objective in \Cref{eq:ERM-PI} takes a separable form on types, and for each type $T$, $\widehat{\theta}(X^n)$ (with input type $T$) is the minimizer of
\begin{align*}
	\sum_{m\in \calS(T)} \| \theta^{n,(m)} - \widehat{\theta}(X^{n,(m)}) \|_2^2 &= \sum_{m\in \calS(T)} \| \theta^{n,(m)} - \pi_m^{-1}\circ \widehat{\theta}(X^n) \|_2^2 \\
	&= \sum_{m\in \calS(T)} \| \pi_m \circ \theta^{n,(m)} - \widehat{\theta}(X^n) \|_2^2. 
\end{align*} 
This gives that
\begin{align*}
	\widehat{\theta}^n(X^n) = \frac{1}{|\mathcal{S}(X^n)|}\sum_{m\in \mathcal{S}(X^n)} \pi_m\circ \theta^{n, (m)}
\end{align*}
whenever $\mathcal{S}(X^n)\neq\emptyset$;
when $\mathcal{S}(X^n)=\emptyset$, we set $\widehat{\theta}^n(X^n)$ to an arbitrary vector in $[0,A]^n$. 

To continue, we claim the following that connects the approximation error and regret bound uniform over all $G_0\in\mathcal{P}([0, A])$, which we will defer the proof to \Cref{app:approximate_erm_proof}. 

\begin{lmm}\label{lemma:approximate_ERM}
	Let $r_n$ be defined in \Cref{thm:general}. Let $\delta > 0$ and a set $\mathcal{G}\subseteq\naturals^n$ such that for all priors $G_0\in \mathcal{P}([0, A])$, 
	we have $f_{G_0}^{\otimes n}(\mathcal{G}^c)\le \delta$. 
	Suppose an estimator $\widehat{\theta}^n\in [0,A]^n$ is an approximate population risk minimizer on $\mathcal{G}$, satisfying
	\begin{align}\label{eq:approximate_ERM}
		\bE_{\Pi}\qth{ \frac{1}{n}\|\widehat{\theta}^n(X^n) - \bE_{\Pi}[\theta^n | X^n]\|_2^2 \indc{X^n\in \mathcal{G}}} \le \varepsilon^2. 
	\end{align}
	Then for an absolute constant $C=C(A)$, 
	\begin{align*}
		\sup_{G_0} \reg(\widehat{\theta}^n; G_0) \le \frac{C\log n}{n\log\log n}\bpth{\frac{\log n}{\log\log n} + nr_n^2} + C\pth{\varepsilon e^{nr_n^2} + \delta}. 
	\end{align*} 
\end{lmm}

We take $\mathcal{G}=[0, L]^n$ with $L = C(A)\frac{\log n}{\log\log n}$, so that $\delta = O(\frac{1}{n})$. 
We show that $\varepsilon \le \frac{1}{n}\exp(-nr_n^2)$ for the estimator $\widehat{\theta}^n$ in \eqref{eq:ERM-PI} and $M = \exp(O_A(\frac{\log^2 n}{\log \log n} + nr_n^2))$. Let ${\bf X}=(X^{n,(m)})_{m\in [M]}$ denote the collection of all $X$ sequences in training data. We note that whenever $|\calS(x^n)|\ge 1$, 
\begin{align*}
	\bE[\widehat{\theta}^n(x^n) | {\bf X}] = \frac{1}{|\mathcal{S}(x^n)|}\sum_{m\in \mathcal{S}(x^n)} \bE[\pi_m\circ \theta^{n, (m)} | {\bf X}] = \bE_\Pi[\theta^n | X^n = x^n], 
\end{align*}
where $(\theta^n, X^n)$ is a single generic batch generated from $\Pi$. By independence among batches, we conclude from the above unbiasedness that
\begin{align*}
	\bE\bqth{ \| \widehat{\theta}^n(x^n) - \bE_\Pi[\theta^n | X^n = x^n] \|_2^2 \mid {\bf X}} &= \frac{1}{|\calS(x^n)|}\bE\bqth{ \| \theta^n - \bE_\Pi[\theta^n|X^n = x^n] \|_2^2 \mid X^n = x^n } \\
	&\le \frac{nA^2}{|\calS(x^n)|}. 
\end{align*}
For $|\calS(x^n)|=0$, we simply replace the denominator by $1$. Finally, taking expectation with respect to ${\bf X}$ and $x^n$ gives
\begin{align*}
	\bE \qth{\frac{1}{n}\|\widehat{\theta}^n(X^n) - \bE_{\Pi}[\theta^n | X^n]\|_2^2 \indc{X^n\in [0, L]^n}} \le \bE\qth{\frac{A^2}{|\calS(X^n)| \vee 1} \indc{X^n\in [0, L]^n}}. 
\end{align*}

To upper bound the final quantity, we sum over all possible types $T$ obtained from $[0, L]^n$. Let $p(T)$ be the probability that a fresh draw $X^n\sim \Pi_{X^n}$ has type $T$, then $|\calS(X^n)|\sim \mathsf{B}(M, p(T))$. Consequently, 
\begin{flalign*}
	\bE\qth{\frac{A^2}{|\calS(X^n)| \vee 1} \Big| T(X^n) = T} 
	&= A^2\cdot (1 - p(T))^M + A^2 \sum_{k=1}^M \frac{1}{k} \binom{M}{k}p(T)^k(1-p(T))^{M-k}
	\\&\stepa{\lesssim} (1 - p(T))^M + \min\{1, \frac{1}{Mp(T)}\}
	\\&\stepb{\lesssim} \min\{1, \frac{1}{Mp(T)}\}
\end{flalign*}
for some absolute constant $C>0$. Here (a) is due to \cite[Lemma 16]{polyanskiy2021sharp}, 
and (b) is due to the inequality $(1-x)^M\le \min\{1, \frac{1}{Mx}\}$ for all $x\in [0, 1]$. 
Let $N$ be the number of types for $X^n\in [0, L]^n$, with 
\begin{align*}
	N = \binom{L + n}{L} = O((n+L)^L) = O\left(\exp(\frac{\log^2 n}{\log \log n})\right).
\end{align*}
Denoting $T_L^{(n)}$ as the set of all types in $[0, L]^n$,  we may continue the previous display to get 
\[
\varepsilon^2 \lesssim \mathbb{E}\qth{\min\{1, \frac{1}{Mp(T)}\}\indc{T\in T_L^{(n)}}}
\le \sum_{T\in T_L^{(n)}} \frac{p(T)}{Mp(T)}
=\frac{N}{M}. 
\]
Therefore, by taking $M = n^2N\exp(2nr_n^2) \le \exp\left(O\left(\frac{\log^2 n}{\log \log n} + nr_n^2\right)\right)$, we obtain $\varepsilon\le \frac{1}{n}e^{-nr_n^2}$, as desired. 

\subsubsection{Proof of Lemma \ref{lemma:approximate_ERM}}\label{app:approximate_erm_proof}
By \Cref{thm:general} and the triangle inequality $(a+b)^2\le 2(a^2+b^2)$, we only need to show that for every $G_0$ supported on $[0,A]$, 
\begin{align*}
	\bE_{G_0}\qth{ \frac{1}{n} \|\widehat{\theta}^n(X^n) - \bE_{\Pi}[\theta^n | X^n]\|_2^2 } \le C\pth{\varepsilon e^{nr_n^2} + \delta}. 
\end{align*}
We first note that 
\begin{align*}
	\bE_{G_0}\qth{ \frac{1}{n} \|\widehat{\theta}^n(X^n) - \bE_{\Pi}[\theta^n | X^n]\|_2^2 \indc{X^n\in\mathcal{G}^c}} \le A^2\mathbb{P}_{G_0}[X^n\in \mathcal{G}^c]\le A^2\delta 
\end{align*}
hence it suffices to show that 
\begin{align*}
	\bE_{G_0}\qth{ \frac{1}{n} \|\widehat{\theta}^n(X^n) - \bE_{\Pi}[\theta^n | X^n]\|_2^2 \indc{X^n\in\mathcal{G}}} \le Ce^{nr_n^2}\varepsilon
\end{align*}
To this end, first note that the data distribution $f_{G_0}^{\otimes n}$ can be changed into $f_G^{\otimes n}$ for some $G$, in view of \prettyref{definition:PoP}, by incurring a cost at most $\TV(f_G^{\otimes n},f_{G_0}^{\otimes n}) \le n\TV(f_G,f_{G_0})=O(\frac{1}{n})$. Therefore, WLOG we may assume that $G=G_0$ and write
\begin{align*}
	\chi^2(f_{G_0}^{\otimes n} \| \Pi_{X^n}) + 1 &= \sum_{x^n\in \naturals^n} \frac{f_{G_0}^{\otimes n}(x^n)^2}{\Pi_{X^n}(x^n)} = \sum_{x^n\in \naturals^n} \frac{f_{G_0}^{\otimes n}(x^n)^2}{\bE_{G\sim \Pi}[f_G^{\otimes n}(x^n)]} \\
	&\stepa{\le} \frac{1}{\Pi(U)} \sum_{x^n\in \naturals^n} \frac{f_{G_0}^{\otimes n}(x^n)^2}{\bE_{G\sim \Pi_{|U}}[f_G^{\otimes n}(x^n)]} \\
	&\stepb{\le} \frac{1}{\Pi(U)} \bE_{G\sim \Pi_{|U}}\qth{\chi^2(f_{G_0}^{\otimes n}\|f_G^{\otimes n})+1} \\
	&= \frac{1}{\Pi(U)} \bE_{G\sim \Pi_{|U}}\qth{\pth{\chi^2(f_{G_0}\|f_G)+1}^n} \\
	&\stepc{\le} e^{E(n^{-2},r_n)+nr_n^2} \stepd{\le} e^{2nr_n^2},  
\end{align*}
where (a) defines $U := \sth{G: \chi^2(f_{G_0}\|f_G)\le r_n^2}$ and $\Pi_{|U}$ as the restriction of $\Pi$ to $U$, (b) follows from convexity of $x\mapsto \frac{1}{x}$, (c) uses the definition of $U$ and \prettyref{definition:PoP}, and (d) uses the definition of $r_n$ that $E(n^{-2},r_n)\le nr_n^2$. Next, we invoke the following form of the Cauchy-Schwarz inequality on arbitrary distributions $P$ and $Q$, an event $E$, and arbitrary nonnegative function $g$: 
\begin{flalign*}
	(\mathbb{E}_P[g\indc{E}])^2 \le \mathbb{E}_Q[g^2\indc{E}](\chi^2(P\|Q)+1). 
\end{flalign*}
Using $P=f_{G_0}^{\otimes n}$, $Q=\Pi_{X^n}$, $g = \frac{1}{n} \|\widehat{\theta}^n(X^n) - \bE_{\Pi}[\theta^n | X^n]\|_2^2$, 
and $E$ denotes the event $X^n\in \mathcal{G}$,
we conclude that 
\begin{align*}
	\bE_{G_0}\qth{ g \indc{X^n\in \mathcal{G}}} &\le \sqrt{e^{2nr_n^2}\cdot \bE_{\Pi}\bqth{g^2\indc{X^n\in \mathcal{G}}}} \\
	&\le \sqrt{e^{2nr_n^2}\cdot A^2 \bE_{\Pi}\bqth{g\indc{X^n\in \mathcal{G}}}} \le Ce^{nr_n^2}\varepsilon.
\end{align*}
This is the desired claim. 

\subsection{Proof of \prettyref{thm:simple_prior_subexpo}}
To generalize our approach to subexponential priors, we consider the prior truncation techniques (as generalized from \cite[Lemma 12]{jana2023empirical}). In other words, we show that for any reasonable estimator $\widehat{\theta}$ and a prior $G\in\subexpo(s)$, 
we can find $G'$ supported on $[0, c(s)\log n]$ such that the extra regret incurred by this prior truncation is at most $O(\frac{1}{n})$. 

\begin{lmm}\label{lmm:prior-truncation}
	Let $G\in\subexpo(s)$, and consider any estimator $\widehat{\theta}^n \in [0,M]^n$.  Denote $G'$ as the prior truncated at $c\log n$ where $c:=c(s)$ is such that $G(\theta > c\log n) < \frac{1}{n^{10}}$, and $G'(\theta\in\cdot) = G(\theta\in \cdot | \theta\le c\log n)$. Then 
	\[
	\reg(\widehat{\theta}^n; G)\le \reg(\widehat{\theta}^n; G') + O_s\left(\frac{1+M^2}{n^4}\right). 
	\]
\end{lmm}
The proof is deferred to \prettyref{app:lmm_trunc_proof}. 

For the hierarchical Bayes estimator $\theta_\Pi^n$, since $\Pi\in \calP(\calP([0,c_0\log n]))$, we have $M=O(\log n)$ in \prettyref{lmm:prior-truncation}. Therefore, the overhead in \prettyref{lmm:prior-truncation} is negligible, and we may assume that $G_0$ is supported on $[0,c_0\log n]$. 

Next we solve for $\varepsilon_n$ and $r_n$ in the posterior contraction lemma (\prettyref{lemma:posterior_contraction}). First, for $\mathcal{P} = \{f_G: \text{supp}(G)\subseteq [0, c_0\log n]\}$, the same truncation argument (with the truncation threshold at $\Theta(\log n)$ rather than $\Theta(\frac{\log n}{\log\log n})$) in \prettyref{lemma:metric_entropy} gives $\log \Nloc(\varepsilon, \calP, H)=O(\log n)$ for all $\varepsilon^2\ge \frac{1}{n}$. Therefore, we can take $\varepsilon_n^2 = O(\frac{\log n}{n})$ for the inequality $\log \Nloc(\varepsilon_n, \calP, H)\le n\varepsilon_n^2$. Next, to upper bound $E(n^{-2},r)$, \prettyref{lemma:moment_matching} implies that every $f_G\in \calP$ can be approximated by a finite Poisson mixture with $L=O(\log n)$ atoms. Therefore, by the same arguments as in the proof of \prettyref{lemma:simple_prior}, we get $E(n^{-2},r) = O(\log^2 n)$ for $r^2\ge \frac{1}{n}$, so that we can take $r_n^2 = O(\frac{\log^2 n}{n})$ from the inequality $E(n^{-2},r_n)\le nr_n^2$. Therefore, by \prettyref{lemma:posterior_contraction} and integrating the tails, we obtain
\begin{align*}
	\bE_{X^{n-1}\sim f_{G_0}^{\otimes (n-1)}}[H^2(f_{G_0}, f_{G_n})] = O\pth{\frac{\log^2 n}{n}}, 
\end{align*}
with $G_n = \bE_{G\sim \Pi_{G|X^{n-1}}}[G]$ defined in \prettyref{lemma:posterior_mean_training}. 

We now conclude by connecting Hellinger distance regret bound by using the following tools \cite[Lemma 4]{jana2025optimal}: for $G,G_0\in \calP([0,A])$ and any $K\in \naturals$, 
\[
\reg(\theta_G; G_0)\le C(AKH^2(f_G, f_{G_0}) + \varepsilon_K(G_0)),
\]
with $\varepsilon_K(G_0) := \sum_{y\ge K} f_{G_0}(y)$. 
Standard Poisson tails yield that $\varepsilon_K(G_0)\le \frac{1}{n}$ for some $K=O(\log n)$. Choosing $G=G_n$, taking expectation over $G_n$, and noting that $A=c_0\log n$, the regret bound is now $O(\frac{\log^4 n}{n})$, as desired. 

\subsubsection{Proof of \prettyref{lmm:prior-truncation}}\label{app:lmm_trunc_proof}
Denote the event $E = \{\max_{i=1}^n \theta_i \le c\log n\}$; 
we have $\mathbb{P}[E^c] \le n^{-9}$. 
Let $\mmse(G) := \min_{\widehat{\theta}}\mathbb{E}_{\theta\sim G}[(\widehat{\theta}(X) - \theta)^2]$, i.e. the MSE achieved by the Bayes estimator. Then by \cite[Eqn. (131)]{polyanskiy2021sharp}, 
\[
\reg(\widehat{\theta}^n; G) \le \reg(\widehat{\theta}^n; G')+\mmse(G')-\mmse(G)+\mathbb{E}_{G}\bqth{\frac{1}{n}\|\widehat{\theta} - \theta\|_2^2\indc{E^c}}
\]
By \cite[Lemma 2]{wu2011functional}, 
$\mmse(G')-\mmse(G)\le \frac{\varepsilon}{1-\varepsilon}\mmse(G)=O_s(\varepsilon)$, with $\varepsilon=G(\theta>c\log n)\le n^{-10}$. It remains to bound $\mathbb{E}_{\pi}[(\widehat{\theta} - \theta)^2\indc{E^c}]$: by Cauchy-Schwarz, 
\[
\mathbb{E}_{G}\bqth{\frac{1}{n}\|\widehat{\theta} - \theta\|_2^2\indc{E^c}}
\le \sqrt{\mathbb{P}[E^c]\mathbb{E}_{G}\bqth{\frac{1}{n^2}\|\widehat{\theta} - \theta\|_2^4}}
\le n^{-4}(M^2+O_s(1))
=O_s\pth{\frac{1+M^2}{n^4}}, 
\]
where we have used that $\mathbb{E}_G[\theta^4]=O_s(1)$ for all $G\in \subexpo(s)$. This completes the proof. 

\subsection{Proof of \Cref{thm:simple_prior_gaussian}}
We first establish the rate of posterior contraction in the Gaussian case, by working out $\varepsilon_n$ and $r_n$ in \prettyref{lemma:posterior_contraction}. To upper bound the local entropy $\log \Nloc(\varepsilon, \calP, H)$ for the class of Gaussian mixtures $\calP = \sth{f_G(x)=\bE_{\theta\sim G}[\varphi(x-\theta)]: G\in \calP([-A,A])}$, we will overbound it by the global entropy $\log N(\varepsilon, \calP, H)$. We quote the following relationship between the TV distance of Gaussian mixtures and moment matching \cite[Lemma 9]{wu2020optimal}: 
\[
\TV(f_{G_0}, f_{G_1})\le \frac{1}{2}\left[\sum_{m=0}^{\infty} \frac{|\bE_{U\sim G_1} [U^m] - \bE_{V\sim G_2} [V^m]|^2}{m!}\right]^{1/2}. 
\]
Since $H^2\le 2\TV$, by simple algebra and Carath\'eodory's theorem, every $f_{G_0}\in \calP$ is $O(\frac{1}{n^2})$-close to a finite Gaussian mixture with $L=O_A(\frac{\log n}{\log\log n})$ atoms. Therefore, by quantizing the atom locations and weights of an $L$-component Gaussian mixture, we get
\begin{align*}
	\log N(n^{-1/2}, \calP, H) = O\pth{\frac{\log^2 n}{\log\log n}}. 
\end{align*}
We can therefore choose $\varepsilon_n^2 = O(\frac{\log^2 n}{n\log\log n})$ in \prettyref{lemma:posterior_contraction}. 

The above approximation by an $L$-component Gaussian mixture also gives an upper bound of $E(n^{-2},r)$. Based on the finite Gaussian mixture, the same argument in the proof of \prettyref{lemma:simple_prior} yields
\begin{align*}
	E(n^{-2},r) = O(L\log n) = O\pth{\frac{\log^2 n}{\log\log n}}, \quad r^2\ge \frac{1}{n}. 
\end{align*}
Therefore, we can take $r_n^2 = O(\frac{\log^2 n}{n\log\log n})$ in \prettyref{lemma:posterior_contraction}, which gives
\begin{align*}
	f_{G_0}^{\otimes (n-1)}\bpth{H^2(f_{G_0}, f_{G_n}) \ge C\varepsilon^2} \le \frac{1}{n} + e^{-cn\varepsilon^2}, \quad \text{for } \varepsilon^2 \gtrsim \frac{\log^2 n}{n\log\log n},  
\end{align*}
where $G_n = \bE_{G\sim \Pi_{G|X^{n-1}}}[G]$ is defined in \prettyref{lemma:posterior_mean_training}.

Finally we quote a state-of-the-art regret-Hellinger inequality in the recent work \cite[Theorem 1]{chen2026sharp}:
\begin{align*}
	\reg(\theta_{G_n}; G_0) \le CH^2(f_{G_n},f_{G_0})\frac{\log \frac{1}{H^2(f_{G_n},f_{G_0})}}{\log\log \frac{1}{H^2(f_{G_n},f_{G_0})}}. 
\end{align*}
By the Hellinger guarantee above and tail integration, we obtain a regret of $O(\frac{\log^3 n}{n(\log\log n)^2})$ for $\theta_\Pi^n$, as desired.

\subsection{Proof of \Cref{thm:simple_prior_polynomial}}
The recent work \cite[Lemma 16]{kang2026function} tells that, for any $G, G_0$ supported on $[0, A]$ and any $K > p$, we have 
\[
\bE_{G_0}\qth{ \pth{g_G - g_{G_0}}^2 } 
\le C(A^{2p} + A^pK^p)H^2(f_{G_0}, f_G)+A^{2p}f_{G_0}(X > K - p).
\]
Taking $K = C(A)\frac{\log n}{\log \log n}$ such that 
$f_{G_0}(X > K - p) \le \frac{1}{n^2}$, the target regret bound is then a direct consequence of the same Hellinger guarantee in \Cref{thm:general} and \prettyref{lemma:simple_prior}: 
\begin{align*}
	\bE_{X^{n-1}\sim f_{G_0}^{\otimes (n-1)}}[H^2(f_{G_0}, f_{G_n})] = O\pth{\frac{\log^2 n}{n\log\log n}}, 
\end{align*}
with $G_n = \bE_{G\sim \Pi_{G|X^{n-1}}}[G]$, where  \prettyref{lemma:posterior_mean_training} implies that $g_{\Pi,n} = g_{G_n}$. 

\section{Additional experimental details}

\subsection{Details of Experiments in \Cref{sec:performance}}
In this section, we mention the details of how the test priors are generated: they are the neural and the multinomial prior-on-priors, 
which are different from the PoP we used in \Cref{alg:universal_prior} that we use to train our transformer. Code release is available at \url{https://github.com/Anzoteh96/eb-transformers}.

\paragraph{Neural-generated prior-on-priors.} This is described in \cite[Appendix A.2]{teh2025solving}, 
reproduced here for clarity. We sample $\theta_{\mathsf{base}}\in [0, A]$ via the following: first, let $\mathcal{M}$ be the set of priors determined by some two-layer perceptron with a non-linear activation. This is defined as follows: 
\begin{equation*}
	\mathcal{M} = \{\pi: \pi = \varphi^{W_1, W_2, \sigma}_{\sharp} \mathsf{Unif}[0, A]\}
\end{equation*}
where $\varphi^{W_1, W_2, \sigma}(x) = \mathsf{Sigmoid}(10W_2\sigma(W_1x))$, $W_1, W_2$ are linear operators, and $\sigma$ is an activation function chosen randomly from 
\[\{\mathrm{GELU, ReLU, SELU, CELU, SiLU, Tanh, TanhShrink}\}. \]
The test Poisson means $\theta_{\mathsf{base}}$ are then produced by sampling from a mixture of 4 priors in $\mathcal{M}$. 

\paragraph{Multinomial prior-on-prior.} 
Here, the prior-on-prior $\Pi_{\mathsf{test}}$ consists of priors in the form $\sum_{j=1}^{10A} w_j\delta_{\frac{j}{10}}$
with $(w_1, \cdots, w_{10A})\sim \text{Dir}(1, 1, \dots, 1)$. In other words, the test prior is a discrete distribution uniformly chosen from the simplex over a fixed grid. 

\subsection{Details of Experiments in \Cref{sec:alpha-posterior}}
Here, we describe how the experiments in \Cref{sec:alpha-posterior} are set up. 
For each $m=2, 5, 10$ we consider $N = 2,000$ runs of sampling $m$ priors $G_1, \cdots, G_m$, take the simple PoP as $\Pi_m=\frac{1}{m}\sum_{i=1}^m \delta_{G_i}$, 
and choose the ones (among the $N$ runs) such that the hierarchical Bayes of $\Pi_m$ has the highest average regret across $G_1, \cdots, G_m$. 

Next, we train a transformer that trains exclusively on $G_1, \dots, G_m$. For all $m=2, 5, 10$, we take $M = 200,000$ steps, 
and for each step $i$ we sample 200 batches, each in the form of $(\theta^n, X^n)$ pairs sampled from $G_j$ where $j = (i-1) \% m + 1$. The transformer architecture and training details are identical to \cite{teh2025solving}.

\subsection{Additional Results in \Cref{sec:alpha-posterior}}
In this section we demonstrate some of the numerical results deferred from the main parts. 
Precisely, we tabulate the results of $\alpha$-posterior on $m = 5$ and $m = 10$, as tabulated in \Cref{fig:alpha-posterior-5priors} and 
\Cref{fig:alpha-posterior-10priors}. 
We see that the hypothesis that the transformers are doing $\alpha$-posterior for $\alpha\simeq\frac{n}{\ntest}$ continues to hold, 
despite larger error bars. 

\UnifiedFigure{hbtp}{fig:alpha-posterior-5priors}{Mean squared distance between transformer output and the hierarchical Bayes estimator using various $\alpha$-posteriors, trained on $m = 5$ priors}
	{\subfloat[$m = 5, n = 50$]{
			\includegraphics[scale=0.4]{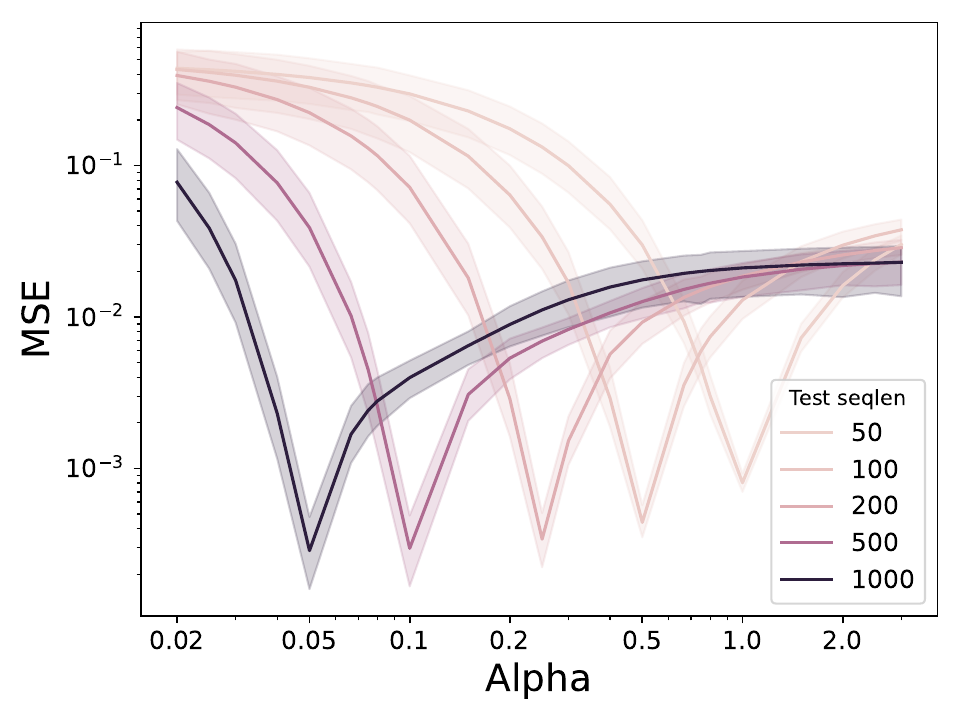}
		}
	  \subfloat[$m = 5, n = 100$]{
	  	\includegraphics[scale=0.4]{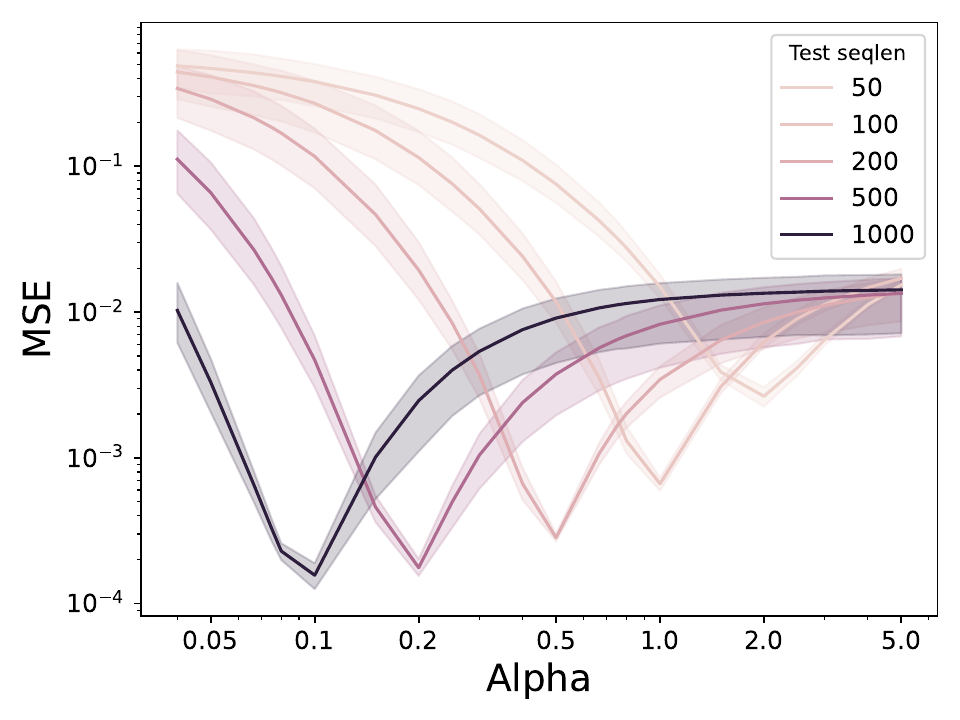}
	  }
      
      \subfloat[$m = 5, n = 200$]{
	  	\includegraphics[scale=0.4]{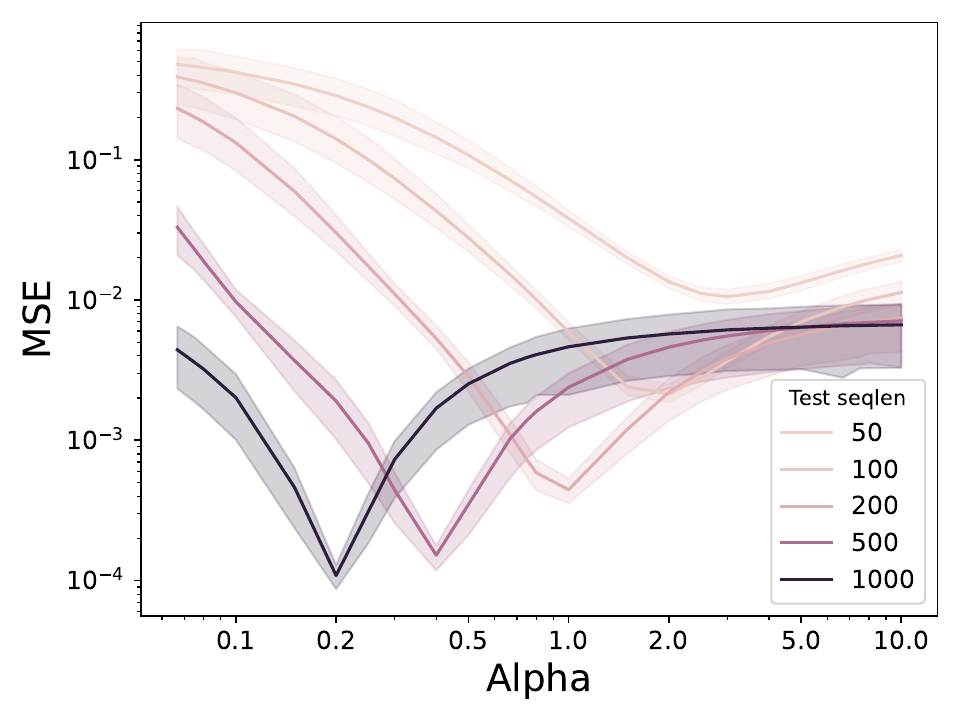}
	  }
      \subfloat[$m = 5, n = 500$]{
      	\includegraphics[scale=0.4]{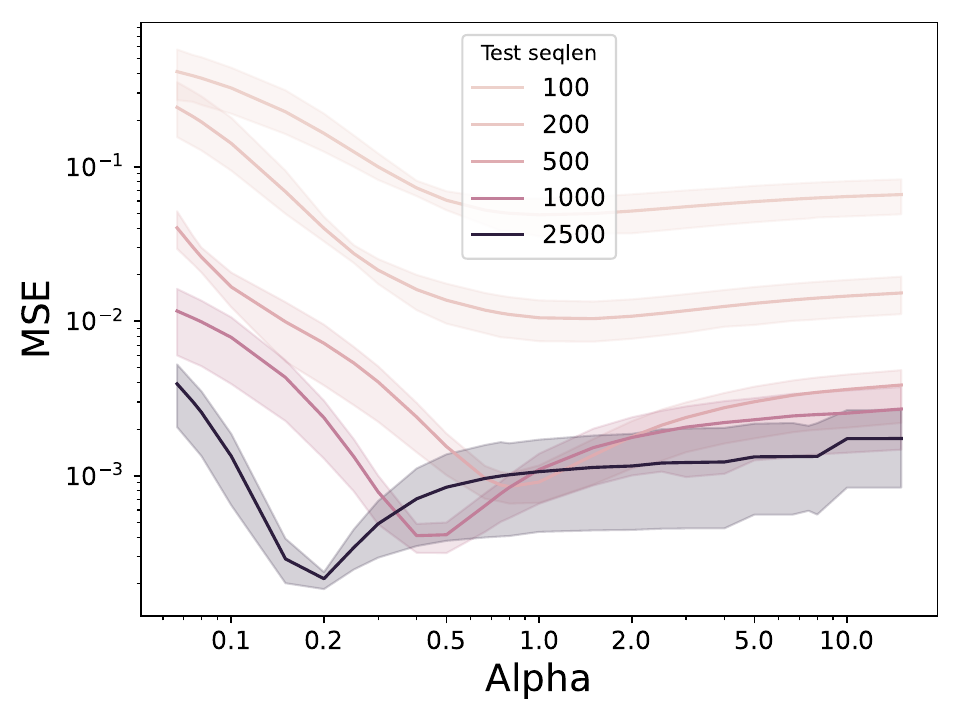}
      }
    }

\UnifiedFigure{t}{fig:alpha-posterior-10priors}{Mean squared distance between transformer output and the hierarchical Bayes estimator using various $\alpha$-posteriors, trained on $m = 10$ priors}
	{\subfloat[$m = 10, n = 50$]{
			\includegraphics[scale=0.4]{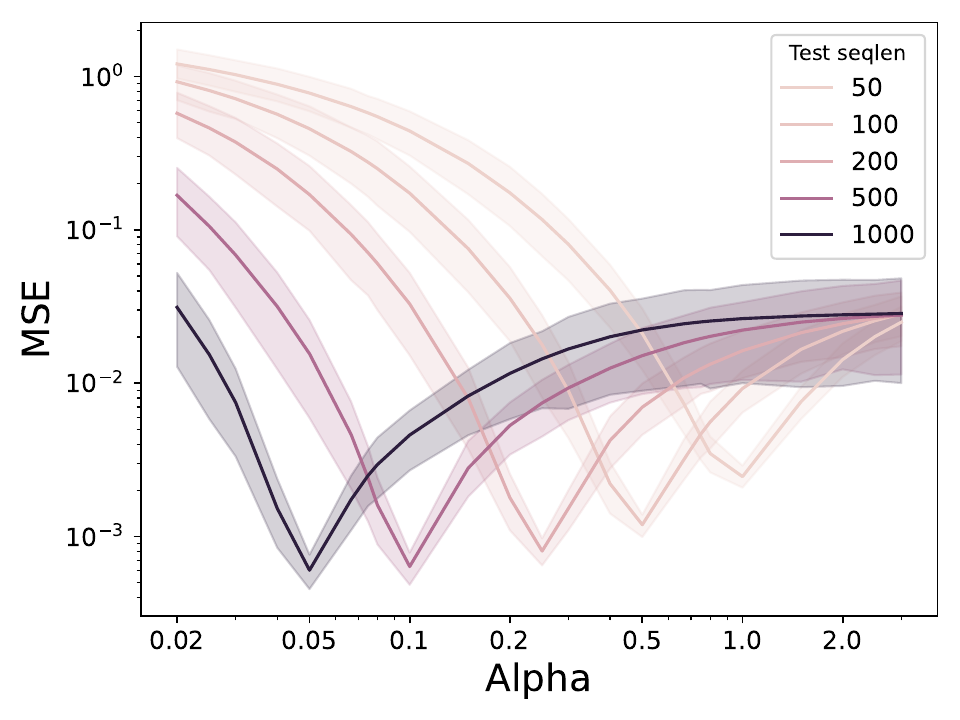}
		}
	   \subfloat[$m = 10, n = 100$]{
	  	\includegraphics[scale=0.4]{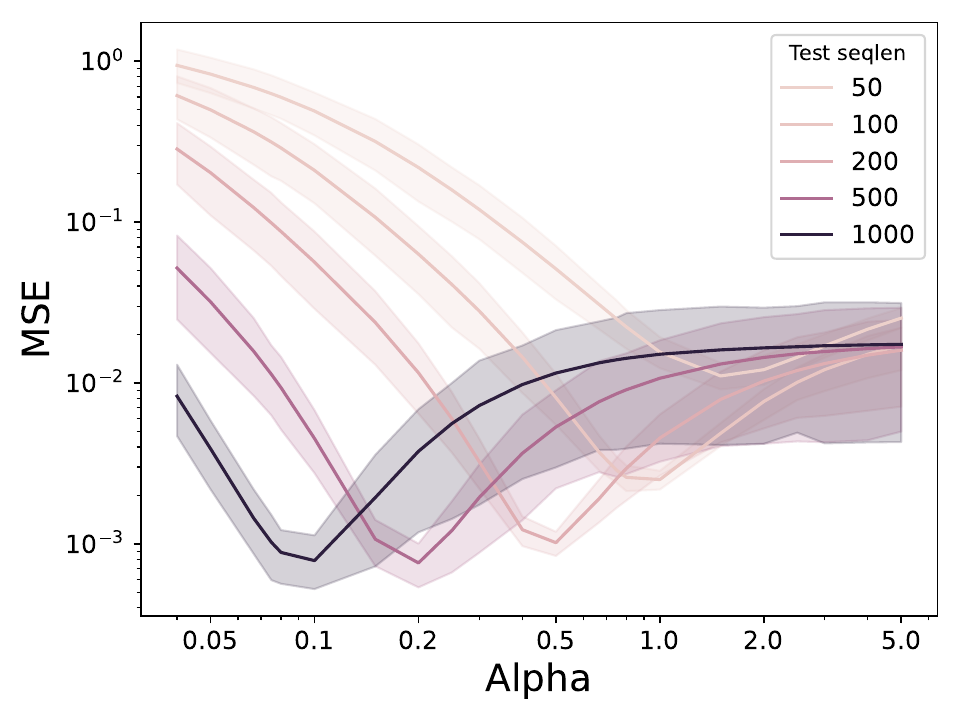}
	   }
       
      \subfloat[$m = 10, n = 200$]{
	  	\includegraphics[scale=0.4]{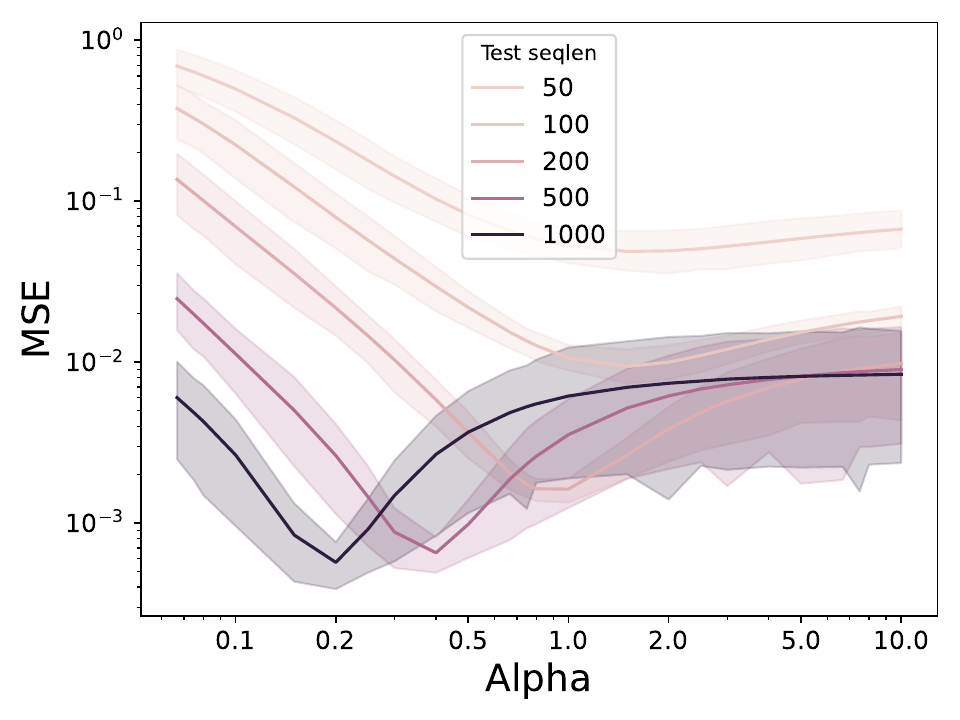}
	   }
      \subfloat[$m = 10, n = 500$]{
      	\includegraphics[scale=0.4]{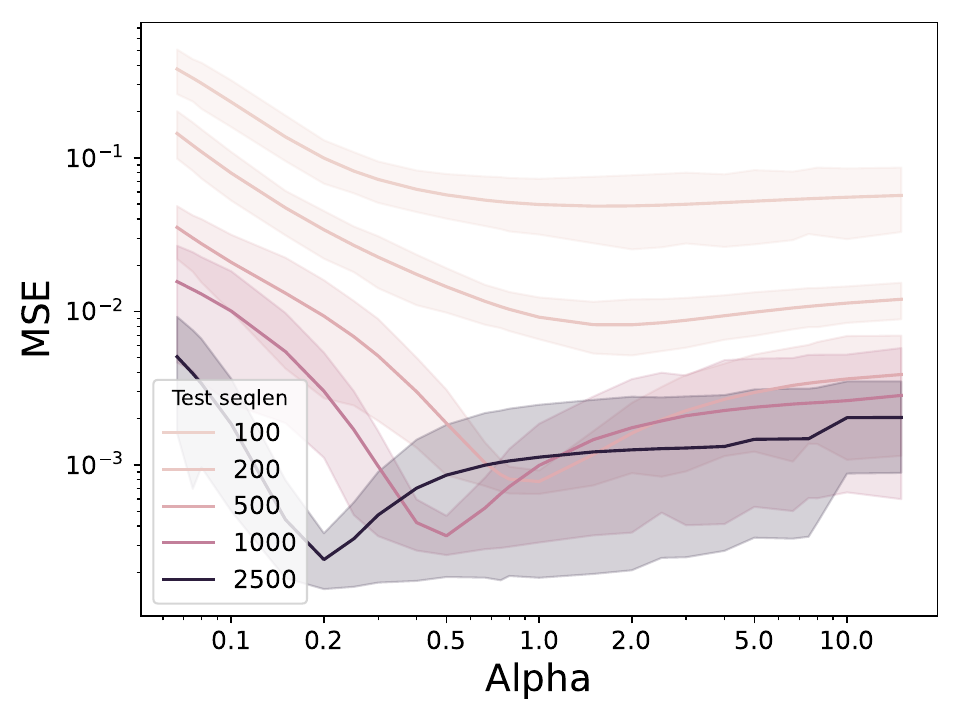}
      }
    }

\end{document}